\documentclass[twoside,11pt]{article}
\usepackage[hyperref, preprint]{jmlr2e}
\usepackage[utf8]{inputenc}
\usepackage[linesnumbered,ruled,vlined]{algorithm2e}
\usepackage{microtype}
\usepackage{graphicx}
\usepackage{subfigure}
\usepackage{amsmath}
\usepackage{amssymb}
\usepackage{booktabs} 
\usepackage{multirow}
\usepackage{mdframed}
\usepackage{tikz}
\hypersetup{
    colorlinks = true,
    linkcolor = blue,      
    citecolor = blue,  
    linkbordercolor = {white} 
}

\usetikzlibrary{arrows.meta, positioning}

\newcommand{\argmin}{{\text{argmin}}}
\newtheorem{assumption}{Assumption}
\let\proof\relax

\newenvironment{proof}{\par\noindent{\bf Proof\ }}{\hfill\BlackBox\\[2mm]}

\usepackage{pifont}
\newcommand{\cmark}{\ding{51}}%
\newcommand{\xmark}{\ding{55}}%
\usepackage{lastpage}
\jmlrheading{25}{2024}{1-\pageref{LastPage}}{9/22; Revised
6/24}{9/24}{22-1024}{Xuetong Wu, Mingming Gong, Jonathan H. Manton, Uwe Aickelin, Jingge Zhu}

\ShortHeadings{On Causality in Domain Adaptation and Semi-Supervised Learning}{Wu, Gong, Manton, Aickelin, Zhu}

\firstpageno{1}

\begin{document}

\title{On Causality in Domain Adaptation and Semi-Supervised Learning: an Information-Theoretic Analysis for Parametric Models}

\author{\name Xuetong Wu \email xuetongw1@student.unimelb.edu.au\\
\addr Department of Electrical and Electronic Engineering\\
\AND
\name Mingming Gong \email mingming.gong@unimelb.edu.au\\
\addr School of Mathematics and Statistics\\
\AND 
\name Jonathan H. Manton \email jmanton@unimelb.edu.au\\
\addr Department of Electrical and Electronic Engineering\\
\AND
\name Uwe Aickelin \email uwe.aickelin@unimelb.edu.au\\
\addr Department of Computing and Information Systems\\
\AND
\name Jingge Zhu \email jingge.zhu@unimelb.edu.au \\
       \addr Department of Electrical and Electronic Engineering\\
       University of Melbourne\\
       Parkville, 3010, Australia
       }
\editor{Ilya Shpitser}

\maketitle

\begin{abstract} 
\noindent Recent advancements in unsupervised domain adaptation (UDA) and semi-supervised learning (SSL), particularly incorporating causality, have led to significant methodological improvements in these learning problems. However, a formal theory that explains the role of causality in the generalization performance of UDA/SSL is still lacking. In this paper, we consider the UDA/SSL scenarios where we access $m$ labelled source data and $n$ unlabelled target data as training instances under different causal settings with a parametric probabilistic model. We study the learning performance (e.g., excess risk) of prediction in the target domain from an information-theoretic perspective. Specifically, we distinguish two scenarios: the learning problem is called causal learning if the feature is the cause and the label is the effect, and is called anti-causal learning otherwise.  We show that in causal learning, the excess risk depends on the size of the source sample at a rate of $O(\frac{1}{m})$ only if the labelling distribution between the source and target domains remains unchanged. In anti-causal learning, we show that the unlabelled data dominate the performance at a rate of typically $O(\frac{1}{n})$. These results bring out the relationship between the data sample size and the hardness of the learning problem with different causal mechanisms.
\end{abstract}


\section{Introduction} \label{Introduction}
A common obstacle in many real-world learning problems is that the training and testing data may originate from different distributions. Such a paradigm is known as the ``domain adaptation" problem. Specifically, we consider the unsupervised domain adaptation (UDA) scenarios in which we have two datasets drawn from different distributions, namely the ``source" and ``target" distributions, respectively. The source dataset includes both features and labels, whereas the target dataset contains only features and no labels. The goal is to train a model that performs well on the \textbf{target} distribution. This assumption is particularly interesting because it reflects real-world scenarios where the target labels are often unavailable.

\citet{scholkopf2012causal} began the pioneering work of developing a framework that links causal mechanisms with UDA, where the objective is to predict the label $Y$ using feature $X$. They delve into two fundamental causal settings: the ``causal learning" setting, where $X$ is the cause of $Y$, and the ``anti-causal learning" setting, where $Y$ is the cause of $X$. An interesting empirical observation made in the paper is that semi-supervised learning (SSL) - a machine learning paradigm where the model is trained on a mix of labelled and unlabelled data - improves learning performance in the anti-causal direction but does not provide a similar boost in the causal direction. This finding suggests that, given known causal structures, we may be able to enhance the generalization capabilities of machine learning algorithms strategically. Even though numerous causality-driven machine learning algorithms have demonstrated their effectiveness empirically \citep{scholkopf2012causal,zhang2013domain,gong2016domain}, the analytical part remains less investigated. Specifically, understanding how causality impacts learning performance and how the unlabelled target data and labelled source data contribute to the prediction under specific causal settings is yet to be deepened. This paper attempts to demystify how causal directions influence generalization ability and how the labelled source and unlabelled target data contribute to the prediction in the UDA/SSL settings under generative parametric models. Specifically, we examine the excess risk under various distribution shift conditions under the UDA setup, including the case of no distribution shifts as seen in SSL.

Our main results reveal that in the causal learning scenario, the unlabelled target data do not contribute to the prediction, and the source data only aids in reducing the excess risk when the conditional probability distribution $P(Y|X)$ remains consistent between source and target domains. Conversely, in anti-causal learning, unlabelled data are always useful. However, the usefulness of the source data, in terms of the convergence rate for excess risk, is contingent on the distribution shift conditions. In situations where the causal relationship between the feature and the label is unknown, improving generalization capability in domain adaptation requires careful consideration when making predictions from either a causal or anti-causal direction. This understanding enables us to design more efficient learning algorithms that are equipped to handle the challenges presented by complex real-world learning problems.

\section{Related Work}

{\bf Causal Inference and Machine Learning. }
Two important frameworks in causal inference are the potential outcome (counterfactual) framework and the structural causal model (SCM) \citep{holland1986statistics,hernan2010causal,imbens2015causal,pearl2018book}\footnote{It is sometimes also called structural equation model (SEM).}, which allows reasoning about a system not only under observation but also under intervention, and they have become an influential tool in several machine learning problems. For example, \citet{scholkopf2012causal} study the causal and anti-causal learning for domain adaptation with an additive noise SCM. \citet{bottou2013counterfactual} carry out the counterfactual analysis for the advertisement placement problem, allowing more flexibility in decision-making and thus improving the system performance. More recently, \citet{scholkopf2019causality} put forward significant issues such as i.i.d. assumptions and generalization ability of current machine learning algorithms and summarized the intrinsic connections between machine learning and the causality. \citet{moraffah2020causal} reviewed several causal interpretable models and suggested that the causal interpretable model under these causal and anti-causal frameworks is a way to explain the black-box machine learning algorithms. \citet{makhlouf2020survey} argue that causality-based machine learning algorithms are necessary to address the problem of fairness appropriately. 

However, although the causal models are favourable for specific learning regimes, only a few works generally consider generalization ability. To name a few, \citet{kilbertus2018generalization} argue that the generalization capabilities for anti-causal learning problems are associated with the hypothesis space searching and validation, but no theoretical analysis is presented. \citet{kuang2018stable} and \citet{cui2022stable} develop a stable learning algorithm that is robust across different underlying distributions and derives the generalization error bound with the ``causal" features, which are stable across different environments. \citet{arjovsky2019invariant} propose the invariant risk minimization to generalize well across different domains.  \citet{chen2021domain} develop a theoretical framework via the linear structural causal models, allowing comparisons of the learning performance for existing domain adaptation methods.

\noindent {\bf Domain Adaptation} Most techniques to conquer domain adaptation problems are purely statistics-based without referring to causal concepts. For example, the instance-based methods identify source samples that bear similarities to target samples based on the probability density ratio on the marginal distribution of features \citep{cortes2008sample, gretton2009covariate}. The feature-based methods will seek a new latent space where the discrepancy of the empirical distribution embeddings between the source and target domains are small under some metric  \citep{pan2010transfer,zhang2017joint}. The popular deep learning-based methods will involve deep generative networks to align distributions between source and target domains \citep{tzeng2017adversarial,shen2018wasserstein}. However, recent works have shown that introducing causal concepts leads to more robust and efficient algorithms for domain adaptation. The main idea is to identify and extract the transferable components that are invariant across different domains under certain causal models \citep{gong2016domain,magliacane2017domain,rojas2018invariant,mahajan2021domain}. 
Nevertheless, they mainly focus on the empirical verification of the effect of source samples instead of a theoretical analysis of their algorithms. To rigorously investigate the generalization ability and usefulness of the source and target data, \citet{wu2021online} give an attempt to interpret the transfer learning in terms of parametric probabilistic models. \citet{kpotufe2018marginal} study the covariate shift problem and derive the minimax rate with the notion of ``transfer component".  \citet{cai2021transfer} investigate the concept drift problem and establish the optimal minimax convergence rate with weighted $k$-nearest neighbour classifier. \citet{maity2020minimax} consider the target shift condition and derive the optimal minimax rate in non-parametric classification. 

\noindent {\bf Semi-Supervised Learning}
Semi-supervised learning aims to learn the predictor with scarce labelled and abundant unlabelled data. The crucial questions are when the unlabelled data are useful and how to avoid their negative impact. On the practical side, \citet{scholkopf2012causal} find that the unlabelled data will be useful for prediction when these data are the effect of their corresponding (unknown) labels. \citet{li2014towards} propose a robust SVM-based algorithm to prevent the unlabelled data from hurting the performance. Under generalized linear models, \citet{yuval2020semi} analyze the effectiveness of the unlabelled data via risk minimization.  On the theoretical side, \citet{castelli1996relative} and \citet{zhang2000value} pose the parametric assumptions on data distributions and claim the value of the unlabelled data depends on the Fisher information matrices of the distribution parameters. A similar argument is made in \citet{zhu2020semi} that if the unlabelled data contain all information of the required parameters, they will be equally useful as the labelled data. \citet{seeger2000input} and \citet{liang2007use} suggest that for certain data-generating processes, the unlabelled data is not useful from a Bayesian perspective. We refer to \citet{mey2019improvability} for other plentiful theoretical results on semi-supervised learning. Our methods provide a pathway to probabilistically analyze the semi-supervised learning problem and definitude the conditions when the unlabelled data are useful from a causal point of view.

\section{Preliminaries}
In this paper, we use the convention that capital letters denote the random variables and small letters their realizations. We define $a \vee b=\max (a, b)$ and $a \wedge b=\min (a, b)$. The notation $f(n) \asymp g(n)$ means that there exists some positive integer $n_0$ such that for all $n > n_0$, $c_1g(n) \leq f(n) \leq c_2g(n)$ always holds for some positive $c_1$ and $c_2$. We also use $f(n) = O(g(n))$ by meaning that there exists some integer $n_0$ such that for all $n > n_0$, $ f(n) \leq c_3g(n)$ always holds for some positive value $c_3$. We denote the KL divergence between two distributions $P$ and $Q$ by $\textup{KL}(P\|Q) = \mathbb{E}_{P}\left[\log\frac{dP}{dQ}\right]$. We use $P(X) \ll Q(X)$ to denote that the probability distribution $P(X)$ is absolutely continuous w.r.t. $Q(X)$. If not otherwise specified, the notation $\mathbb{E}_{\theta}[\cdot]$ denotes the expectation taken over all data examples involved that are drawn from ${P}_{\theta}$.

\subsection{Information Theory Basics}
Before proceeding, we will define several common information theory quantities such as entropy, mutual information, and Kullback-Leibler divergence (KL divergence), and state several well-known results on these measures that will be referenced in the literature. For more information on the basics, the readers can refer to \cite{cover2006elements}. The Shannon entropy of a discrete random variable $X$ is defined as:
\begin{align}
 H(X)=-\sum_{x \in \mathcal{X}} \mathbb{P}(X=x) \log \mathbb{P}(X=x).
\end{align}
For continuous random variable $X$ with the probability density function $p(x)$, the differential entropy is defined as:
\begin{align}
 h(X)=-\int p(x) \log p(x) dx.
\end{align}
Note that for discrete r.v., the Shannon entropy is always nonnegative and bounded by $\log |\mathcal{X}|$ while the differential entropy is considered as a measure of relative information and can be negative. Next, we define the Kullback-Leibler divergence: for two probability measures $P$ and $Q$, if $P$ is absolutely continuous with respect to $Q$, the Kullback-Leibler divergence between $P$ and $Q$ is:
\begin{align*}
   D(P \| Q)=\int \log \left(\frac{d P}{d Q}\right) d P,
\end{align*}
where $\frac{d P}{d Q}$ is the Radon-Nikodym derivative of $P$ with respect to $Q$. The KL divergence roughly estimates how different the two distributions $P$ and $Q$ are. For any probability distributions $P$ and $Q$ over the space $\Omega$ such that $P$ is absolutely continuous with respect to $Q$, we have the non-negativity property such that $D(P\|Q) \geq 0$ and the quantity is usually non-symmetric, e.g., $D(P\|Q) \neq D(Q\|P)$ if $P\neq Q$. We can then define the mutual information between the random variables $X$ and $Y$ as:
\begin{align}
I(X; Y)=D(P(X, Y) \| P(X) P(Y)),
\end{align}
which is the Kullback-Leibler divergence between the joint distribution of $X$ and $Y$ and the product of the marginal distributions. From the definition, it is clear that $I(X; Y)=I(Y; X)$, and the first property of the KL divergence implies that $I(X; Y)$ is nonnegative and $I(X; Y)=0$ when $X$ and $Y$ are independent. Furthermore, we also define conditional mutual information as
\begin{align*}
    I(X,Y|Z) = \mathbb{E}_{Z}\left[D(P(X,Y|Z)\| P(X|Z)P(Y|Z))\right],
\end{align*}
where it represents the amount of information gained about $X$ by observing $Y$ given a third variable $Z$.

\subsection{Prediction with Mixture Strategy}
Considering the effectiveness and complexity of UDA and SSL problems, we use the parametric distribution models as a critical component of our approach. The reason for this choice is that the distribution shifts can be characterized concisely by the parameter changes. This approach allows for a rigorous statistical framework in which the complexities of the learning problem can be analyzed. 

The mixture strategy is an important concept in the field of statistical inference that was leveraged from \cite{clarke1994jeffreys,clarke1990information,merhav1998universal} with the application of universal prediction, which involves the construction of a mixture distribution over the model parameters for prediction when the true distribution (parameters) is unknown. Here, ``universal'' means that the predictor does not depend on the unknown underlying distribution and performs essentially as well as if the distribution was known in advance. Furthermore, given these complexities and the distributional shifts of data sources, a mixture strategy becomes a natural choice for tackling these challenges in different domain adaptation settings as it allows us to integrate source and target distribution information, enabling a comprehensive understanding of the learning performance. 

The mixture strategy has been extensively studied in the literature, with several important works exploring its properties and applications in various fields. For example, \citet{feder1992universal, merhav1998universal, cover1996universal} mainly focused on situations where data is drawn independently and identically from a single parametric distribution, which is similar to traditional online learning problems. However, the bounds obtained through the conditional mutual information cannot provide more quantitative insights for analyzing the regret. To this end, the previous works such as \cite{clarke1999asymptotic, clarke1990information, zhu2020semi} provided an asymptotic analysis for the conditional mutual information under the conventional online learning or semi-supervised learning problems, where the regret approximation is associated with the sample size and the prior distribution over the distribution parameters. 

Mathematically, let $\theta$ be the parameter of interest that is involved in the model distribution, and let $p(\theta)$ be the prior distribution over $\theta$. Assume we have the training dataset $\mathcal{D}$ with each $Z_i \in \mathcal{D}$ i.i.d. drawn from a distribution $p_\theta^*(Z)$. If we consider the predictor $\omega$ to be a probability distribution over the data sample $Z$, the logarithmic loss is then defined as
\begin{equation}
    \ell(\omega,Z) = - \log \omega(Z).
\end{equation}
We can define the expected loss on test data $Z'$ as
\begin{equation}
    L := - \mathbb{E}_{\theta^*}\left[ \log Q(Z'|\mathcal{D}) \right].
\end{equation}
where the mixture strategy involves constructing a mixture distribution over $Z'$ for the testing data given the training data as
\begin{align}
   Q(Z'|\mathcal{D})= \frac{\int p(\mathcal{D}, Z' | \theta) p(\theta) d \theta}{\int p(\mathcal{D} | \theta) p(\theta) d \theta} = \int p_{\theta}(Z') Q(\theta|\mathcal{D}) d\theta,
\end{align}
where $Q( \theta | \mathcal{D})$ is the conditional distribution of the parameter $\theta$ given the dataset $\mathcal{D}$ induced by $Q(\mathcal{D}) = \int p(\mathcal{D} | \theta) p(\theta) d \theta$ and the joint distribution $p(\mathcal{D}, \theta) = p(\mathcal{D}|\theta)p(\theta)$, and $p(\theta)$ is a prior distribution over $\theta$. From a Bayesian perspective, we assign a probability distribution $p(\theta)$ over the parameter space to represent our prior knowledge, and we update the posterior with the training data to approximate the underlying distributions. With the mixture strategy, the excess risk w.r.t. the best estimator could be rewritten as:
\begin{align}
    R & := - \mathbb{E}_{\theta^*}\left[ \log Q(Z'|\mathcal{D}) \right] - \mathbb{E}_{\theta^*} \left[ \log p_{\theta^*}(Z') \right] \\
    &= \mathbb{E}_{\theta^*}\left[  \log \frac{P_{\theta^*}(Z')}{Q(Z'|\mathcal{D})}\right].  \\
    &= I(Z';\theta^* | \mathcal{D}).
\end{align}
The above characterization implies that under logarithmic loss, with a specific prior $p(\theta)$, the excess risk induced by the mixture strategy is captured by the conditional mutual information between the sample $Z'$ and distribution parameter that is evaluated at $\theta^*$ given the training data, which naturally gives an interpretation on the amount of information that the test data point $Z'$ carries about the true parameter $\theta^*$, given the whole training set $\mathcal{D}$. Such an information-theoretic framework has been established and studied in SSL and online learning problems (see \cite{merhav1998universal, zhan2015online, zhu2020semi} for references). One advantage of this framework is that information-theoretic tools are powerful in studying asymptotic behaviours as well as deriving learning performance bounds for various statistical problems. This characterization also ensures minimax optimality, which means that irrespective of the underlying parameters, the resultant learning rate is guaranteed to be optimal, even in the worst-case scenario. Additionally, information-theoretic quantities such as mutual information and KL divergence (relative entropy) give natural interpretations for the learning bounds. Furthermore, when it comes to distribution parameter estimation, the mixture model is particularly beneficial when the data is believed to be generated from a certain underlying process, as it can provide a probabilistic representation of the diverse sub-populations, and this is particularly valuable where only assuming a single distribution could lead to skewed or inaccurate results (such as the plug-in method). On the other hand, while estimating a single distribution offers simplicity, the model is sensitive to outliers and may fall short when the data complexity is high or the sample size is small. Taking advantage of the robustness of the mixture strategy, this paper expands on the findings of \citet{merhav1998universal} and \citet{zhu2020semi}, which were initially applied to conventional learning scenarios where the source and target originate from the same distribution. In the following, we will examine both UDA and SSL learning bounds across various distribution shift conditions by leveraging a mixture strategy grounded in causal and anti-causal settings. 

\section{Problem Formulation}
We consider the typical \emph{unsupervised domain adaptation} problem for classification. Given the labelled source data $ D^m_s = (X^{(1)}_s,Y^{(1)}_s,\cdots,X^{(m)}_s,Y^{(m)}_s)$ and the unlabelled target data $D^{\textup{U},n}_t = (X^{(1)}_t,\cdots,X^{(n)}_t)$, we assume each source sample is i.i.d. drawn from a probability distribution $P_S(X,Y)$ and takes value in $\mathcal{Z} = \mathcal{X}\times\mathcal{Y}$ and each target sample is i.i.d. drawn from the marginal distribution of $P_T(X,Y)$ and takes value in $\mathcal{X}$. In general, $P_S(X,Y)$ is different from $P_T(X,Y)$, and both $\mathcal{X}$ and $\mathcal{Y}$ can be discrete or continuous. For simplicity, we consider the case where both $X$ and $Y$ are discrete in this paper. We point out that the analysis in the paper continues to hold for a continuous $Y$ in the causal learning case and for a continuous $X$ in the anti-causal learning case. We will predict the label $Y'_t$ for the previously unseen sample $X'_t$ in target domain, utilising the training sample $D^m_s$ and $D^{\textup{U},n}_t$ with the learning algorithm $\mathcal{A} : \mathcal{Z}^{m} \times \mathcal{X}^{n} \times \mathcal{X} \rightarrow \mathcal{B}$, whose output $b$ is the distribution-independent \emph{predictor} for the outcome $Y'_t$ in the predictor space $\mathcal{B}$. We define the loss function $\ell: \mathcal{B} \times \mathcal{Y} \rightarrow \mathbb{R}$ that evaluates the prediction performance. The learning task is to minimise the corresponding \emph{excess risk} for its label $Y'_t$ defined as
\begin{align}
    \mathcal{R}(b) := \mathbb{E}_{D^m_s,D^{\textup{U},n}_t, X'_t, Y'_t} \left[ \ell\left(b, Y'_t\right) - \ell(b^*, Y'_t) \right], \label{eq:excess-risk}
\end{align}  
where the expectation is taken with respect to all the source and target data, and $b^*$ is the optimal predictor that can depend on the true distribution of the data. Particularly, we will also examine the excess risk under the condition $P_S(X, Y) = P_T(X, Y)$, commonly known as \emph{semi-supervised learning}.

\subsection{Causal Settings} 
In this section, we introduce the concept of causality within a supervised learning context involving feature variable $X$ and label variable $Y$. Here we take an approach by establishing the learning model based on the parametric data distributions. We focus on scenarios where there are no other con-founders but only variables $X$ and $Y$. Assume $X$ is drawn from a finite set $\mathcal{X} = \{x_1, x_2, \ldots, x_k\}$ with $k$ elements and the corresponding label $Y$ is drawn from a finite set $\mathcal{Y} = \{y_1, y_2, \ldots, y_{k'}\}$ with $k'$ elements. We then construct the parametric models under causal settings by specifying the joint distribution of $X$ and $Y$ as follows:
\begin{definition}[Causal Settings]\label{def:causal_settings}
We define two distinct learning settings based on the direction of causality for $X \in \mathcal{X}$ and $Y \in \mathcal{Y}$ using the following generation process:
\begin{itemize}
    \item \textbf{Causal learning (Figure \ref{fig:causal_x_to_y})}  We say that ``$X$ causes $Y$" (denoted as $X\rightarrow Y$) if the pair $(X,Y)$ is generated as follows: $X$ is firstly generated according to the distribution $P_{\theta_X}$. Given $X=x$, $Y$ is generated from the distribution $P_{\theta_{Y_x}}$. This implies that the joint distribution of $(X, Y)$ is given by
      \begin{equation}
        P(x,y) = P_{\theta_X}(x)P_{\theta_{Y_{x}}}(y). \label{eq:x_causes_y}
    \end{equation}
    We call a learning problem ``causal learning" if the underlying causal mechanism satisfies $X \rightarrow Y$.
    \item \textbf{Anti-causal learning (Figure \ref{fig:causal_y_to_x})}  We say that ``$Y$ causes $X$" (denoted as $Y\rightarrow X$) if the pair $(X,Y)$ is generated as follows: $Y$ is firstly generated according to the distribution $P_{\theta_Y}$. Given $Y=y$, $X$ is generated from the distribution $P_{\theta_{X_y}}$. This implies that the joint distribution of $(X, Y)$ is given by
      \begin{equation}
        P(x,y) = P_{\theta_Y}(y)P_{\theta_{X_{y}}}(x). \label{eq:y_causes_x}
    \end{equation}
    We call a learning problem ``anti-causal learning" if the underlying causal mechanism satisfies $Y \rightarrow X$.
\end{itemize}    
\end{definition}
These learning scenarios are conceptualized through parametric data generation mechanisms and sketched in Figure~\ref{fig:causal-anti-dag}. When considering the causal setting $X \rightarrow Y$, we assume that $X$ is drawn from the distribution $P_{\theta_X}$ and when we see a realization $x_i$ of the random variable $X$, the distribution of the outcome variable $Y$ is then characterized by a distinct parameter $\theta_{Y_{x_i}}$, and the observed outcome $y$ is assumed to be drawn from the distribution $P_{\theta_{Y_{x_i}}}$. The double subscript notation is intentionally used to emphasize that the parameters $\theta_{Y_{x_i}}$ describe the distribution of $Y$, which is directly associated with the specific values of $x_i$. This framework inherently incorporates the concept of the ``soft" intervention that alters the conditional probability distributions of the variables being intervened upon\citep{eberhardt2007interventions,pearl2009causality,pearl1998graphs,imbens2015causal}, which is a fundamental concept in the study of causality. By firstly setting $X = x_i$, we effectively intervene in the system, which allows for the direct examination of its impact on $Y$ for different interventions. Hence, the model not only captures the association between $X$ and $Y$ but also provides a structured way to explore causal effects through interventions. For the anti-causal setting $Y \rightarrow X$, the procedure is analogous: the distribution of $Y$ is defined by a parameter $\theta^*_Y$, and upon intervening to set $Y$ to $y_i$, the distribution of $X$ is specified by the parameter $\theta_{X_{y_i}}$, from which we observe $x$ through the distribution $P_{\theta_{X_{y_i}}}(X)$. In Definition~\ref{def:causal_settings}, we assume that both $X$ and $Y$ are discrete variables for simplicity. However, it is important to note that our results also apply to cases with discrete causes and continuous effects. 
\begin{figure}[htb]
    \centering
\subfigure[Causal Learning: $X \rightarrow Y$ \label{fig:causal_x_to_y}]{\includegraphics[width =2.55in]{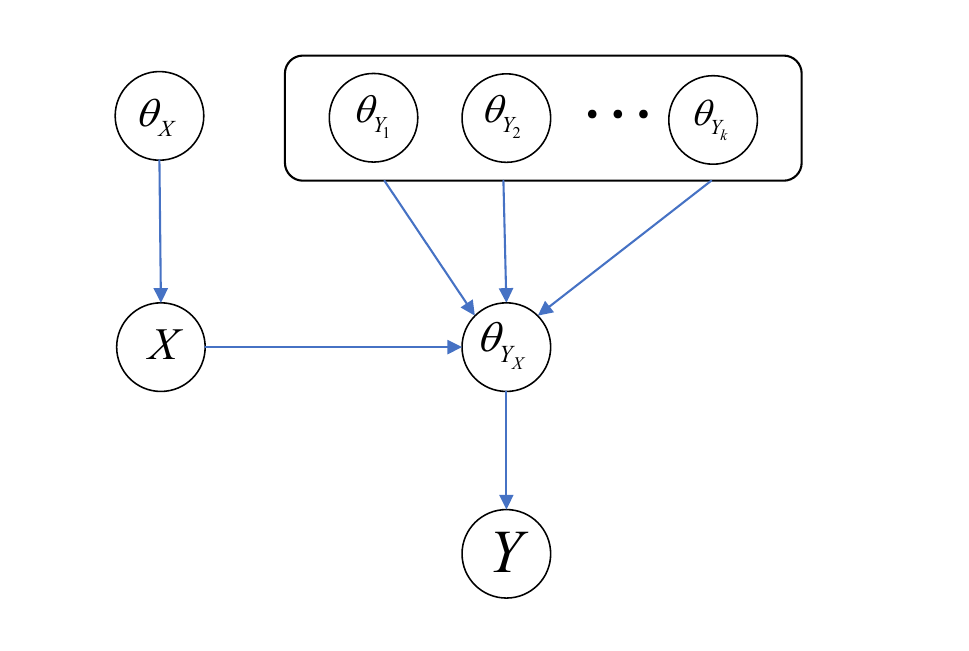}}
\subfigure[Anti-causal Learning: $Y \rightarrow X$ \label{fig:causal_y_to_x}]{\includegraphics[width = 2.33in]{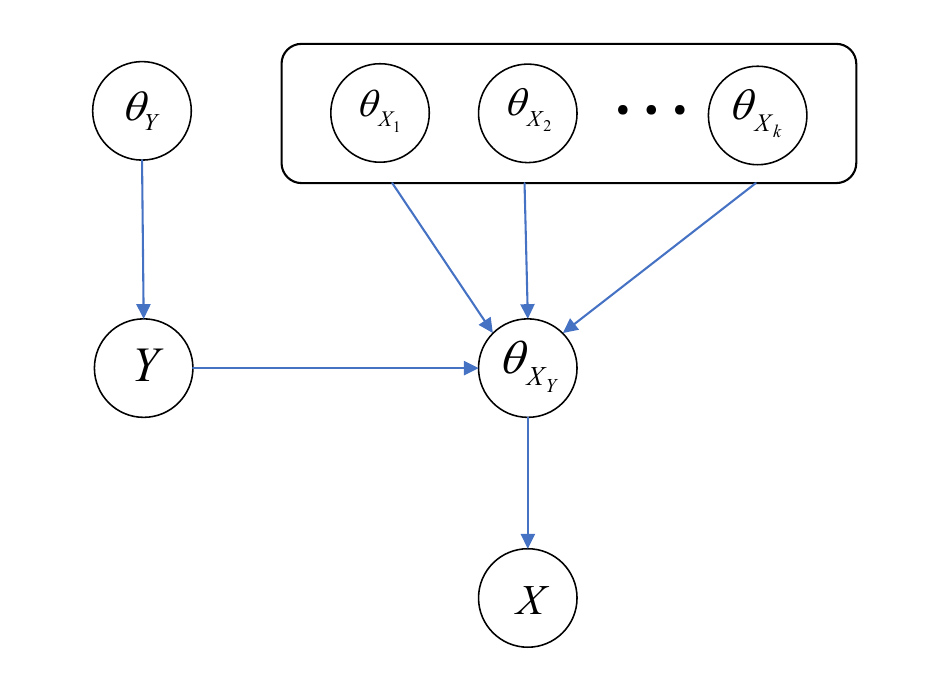}}
\caption{Causal settings for $X \rightarrow Y$ in (a) and $Y \rightarrow X$ in (b). We refer to the scenario in (a) as the ``causal learning" setting because the direction of causation aligns with the direction of prediction, whereas the scenario in (b) is termed the ``anti-causal learning" setting since the direction of causation is opposite to the direction of prediction.}\label{fig:causal-anti-dag}
\end{figure}

We draw the diagram in Figure~\ref{fig:causal-anti-dag} to visualize the parametric models under these two different mechanisms. The models in Figures~\ref{fig:causal_x_to_y} and \ref{fig:causal_y_to_x} are called ``causal learning" and ``anti-causal learning" respectively \citep{scholkopf2012causal}, to mirror the causation direction in alignment with the prediction direction. In causal learning, the prediction direction coincides with the causation direction, whereas in anti-causal learning, the causation direction opposes the prediction direction.


\begin{remark}
As we will show in the sequel, the causal structure of the data-generating process can be leveraged to enhance the prediction performance, which cannot be achieved by using the knowledge of the observational distribution of $(X,Y)$ alone. Roughly speaking, under certain regularity conditions, we could learn the parameters ($\theta_X, \theta_{Y_x}$, etc.) directly from the unlabelled data, thus improving the prediction performance. As for the labelled source data, they can be partially profitable if the target domain shares some distribution parameters with the source domain. We will support these intuitions with our theoretical analysis in Section~\ref{Sec:main}.
\end{remark}

\begin{remark}
We make the following remarks regarding the definitions of the above settings.
\begin{itemize}
\item For simplicity, we will use the notation $Y_x$ to denote a random variable if it is generated according to the distribution $P_{\theta_{Y_x}}(y)$ in the causal learning setting. Similarly, $X_y$ denotes a random variable generated according to the distribution $P_{\theta_{X_y}}(x)$ in the anti-causal learning setting. More generally, we define a random variable $Y_X$ if it is drawn from a random distribution $P_{\theta_{Y_X}}(y)$ induced by the random variable $X$. This notation also suggests an equivalent way of expressing the causality. Namely, we have $Y = \sum_{i=1}^{k} \mathbf{1}_{X = x_i} Y_{x_i}$ for the causal setting where $X$ is generated according to $P_{\theta_X}$, and $X = \sum_{i=1}^{k} \mathbf{1}_{Y = y_i} X_{y_i}$ for the anti-causal setting where $Y$ is generated according to $P_{\theta_y}$. This notation is consistent with the notations used in \citep{hernan2010causal, imbens2015causal, cabreros2019causal}, where the concept of \textit{potential outcome} is used.
\item Figure~\ref{fig:causal-anti-dag} suggests that the random variables $X$ and $Y_{x_1}, Y_{x_2},\ldots, Y_{x_k}$ are mutually \emph{independent} in the causal settings. Similarly, $Y, X_{y_1},\ldots, X_{y_k}$ are also mutually independent in the anti-causal learning setting.
\item The causal setting outlined can also be specialized to parametric structural causal models as outlined by \cite{hernan2010causal,pearl2018book}, which takes the form of the relationship $X' \rightarrow Y'$ by
\begin{align*}
     X' := N_X, \quad Y' := f(N_Y, X').
\end{align*}
Here, $f$ is a function that defines the parametric distributions of $Y'$, with $N_Y$ and $N_X$ being independent random variables. This setup allows us to parameterize the distribution of $X$ with $N_X$ by identifying $P_{\theta_{X}}$ with $P$ where $P$ is the distribution of $N_X$. By setting $Y_{x_k} = f(x_k, N_Y)$, we could then model the distribution of the outcome by $P_{\theta_{Y_{x_k}}}(Y)$ where the parameters depend on the function $f$, $N_Y$ and $x_k$. Then we could express $Y$ as a sum over potential outcomes of $X$, represented as $Y = \sum_{i=1}^{k} \mathbf{1}_{X = x_i} f(x_i, N_Y)$, which simplifies to $Y = f(N_Y, X)$.
\end{itemize}
\end{remark}

For the following discussion and main results, we assume that the causal relationship between $X$ and $Y$ is always unique, e.g., the causal direction is acyclic. Initially, we also assume the relationship is known for the theoretical analysis. In later parts of this discussion, we will also examine the case in which the causal direction of the underlying causal direction is unknown. 

\subsection{Parametric Models}
When studying domain adaptation, we have two sets of random variables $(X_s, Y_s)$ and $(X_t, Y_t)$, where the former denotes the feature and label in the source domain and the latter for the target domain.  We will consider two causal settings.  The first one is given by $X_s \rightarrow Y_s$ and $X_t \rightarrow Y_t$ with the definition of causation given in Figure~\ref{fig:causal_x_to_y}, namely the adaptation with the \textit{causal learning} setting. We assume $X_s, X_t$ take value in $\{x_1,x_2,\cdots, x_k\}$ and $Y_{s}, Y_{t}$ take values in $\mathcal{Y}$, which could be either a continuous or discrete space. We will focus on parametric models in this work, and more precisely, the source distribution (similarly to target distribution) $P_{X_s}$ is parameterized by a parameter $\theta^{s*}_X$ and the distributions of the outcome random variables $P_{Y_{x_i}}$ are also parameterized by the parameters $\theta^{s*}_{Y_{x_i}}$ for all $i=1,\ldots, k$. Then the joint distribution of the data pair $(X_s,Y_s)$ and $(X_t,Y_t)$ can be formulated as,
\begin{align}
    P_{\theta^*_s}(x_s,y_s) = P_{\theta^{s*}_X}(x_s)P_{\theta^{s*}_{Y_{x_s}}}(y_s),  \label{eq:cause_s}\\
    P_{\theta^*_t}(x_t,y_t) = P_{\theta^{t*}_X}(x_t)P_{\theta^{t*}_{Y_{x_t}}}(y_t), \label{eq:cause_t}
\end{align}
where we use $\theta^*_s$ and $\theta^*_t$ to encapsulate all the parameters:
\begin{align}
    \theta^*_s = (\theta^{s*}_X, \theta^{s*}_{Y_{x_1}}, \cdots,  \theta^{s*}_{Y_{x_k}} ) \in \Lambda \label{eq:para-s-true}, \\
    \theta^*_t = (\theta^{t*}_X, \theta^{t*}_{Y_{x_1}}, \cdots,  \theta^{t*}_{Y_{x_k}}) \in \Lambda \label{eq:para-t-true}.
\end{align}
For simplicity, we assume that every element in both $\theta^*_s$ and $\theta^*_t$ is a scalar in $\mathbb{R}$ and $\Lambda \subseteq \mathbb{R}^{k+1}$ is a closed set endowed with Lebesgue measure.  In the sequel, we write $P_S(X) = P_T(X)$ (similarly for $P_S(Y|X)$) with the understanding that their underlying parameters are elementwise equal (e.g., $\theta^{s*}_X = \theta^{t*}_X$) and vice versa. 

The second learning model we consider in this work is given by $Y_s \rightarrow X_s$ and $Y_t \rightarrow X_t$, where $X_{s,y_i}$ and $X_{t,y_i}$ denote the random outcomes given the treatment $y_i$ in source and target domains, namely the adaptation with the \textit{anti-causal learning} setting. The parameterization, in this case, is analogous to causal learning by regarding $Y$ as a cause and $X$ as an effect. Instead, we now assume $Y_s, Y_t$ take value in $\{y_1,y_2,\cdots, y_{k'}\}$ and $X_{s}, X_{t}$ take values in a continuous or discrete space $\mathcal{X}$ for the anti-causal learning. Similarly to the causal learning, we assume $Y_s$ and $Y_t$ are parameterized by $\theta^{s*}_Y$ and $\theta^{t*}_Y$, and $X_{s,y_i}$ and $X_{t,y_i}$ are parameterized by $\theta^{s*}_{X_{y_i}}$ and $\theta^{t*}_{X_{y_i}}$ for all $i = 1,\cdots, k'$, and we use the same notation $\theta^*_s$ and $\theta^*_t$ to encapsulate all the parameters and every parameter in both $\theta^*_s$ and $\theta^*_t$ is a scalar in $\mathbb{R}$ and $\Lambda \subseteq \mathbb{R}^{k'+1}$ is a closed set endowed with Lebesgue measure. 

Under causal learning ($X \rightarrow Y$), it can be seen that the unlabelled target data are generated only with $\theta^{t*}_{X}$ and thus do not contain knowledge about $\theta^{t*}_{Y_{x_i}}$ as they are statistically independent. Intuitively speaking, the parameters associated with $P(Y|X)$ in the target domain cannot be accurately estimated exclusively from the unlabelled data. However, under anti-causal learning ($Y \rightarrow X$), the unlabelled target data are associated with all parameters $\theta^{t*}_{Y}$ and $\theta^{t*}_{X_{y_i}}$ that induce the labelling distribution $P(Y|X)$ in the target domain. In addition, we make the following assumption for the data distributions in both causal settings.

\begin{assumption}[Parametric IID data] \label{asp:para-dist}
We assume the labelled source and unlabelled target samples are generated independently and identically under both causal learning and anti-causal learning. More precisely, the joint distribution of the data sequence pairs $P_{\theta^*_s,\theta^*_t}({D_t^{U,n}, D_s^m})$ can be written as 
\begin{align*} 
P_{\theta^*_s,\theta^*_t}(D^{\textup{U},n}_t, D_s^m) = \prod_{i=1}^n P_{\theta_{t}^*}(X_t^{(i)}) \prod_{j=1}^m P_{\theta_s^*}(X_s^{(j)}, Y_s^{(j)}),
\end{align*}
where $P_{\theta_{t}^*}(X_t^{(i)})$ is the marginal of $P_{\theta_t^*}(X_t^{(i)},Y_t^{(i)})$.  We also assume $\theta^*_t$ and $\theta^*_s$ are points in the interior of $\Lambda$. Furthermore, in both models, the parametric families for the cause and effect are assumed to be known in advance.
\end{assumption}
Based on the models defined above, the excess risk in Equation~(\ref{eq:excess-risk}) can be written as
\begin{align}
    \mathcal{R}(b) &:= \mathbb{E}_{P_{\theta^*_s}(D^m_s)P_{\theta^*_t}(D^{\textup{U},n}_t, X'_t, Y'_t)} \left[ \ell\left(b, Y'_t\right) - \ell(b^*, Y'_t) \right] \nonumber \\
    &= \mathbb{E}_{\theta_s,\theta_t}\left[ \ell\left(b, Y'_t\right) - \ell(b^*, Y'_t) \right]  \label{eq:excess-risk-para}
\end{align}
For simplicity, we use the notation $\mathbb{E}_{\theta_s,\theta_t}[\cdot]$ (similarly, $\mathbb{E}_{\theta_t}[\cdot]$ and $\mathbb{E}_{\theta_s}[\cdot]$) to denote the expectation taken over all source and target samples drawn from ${P}_{\theta_s}$ and ${P}_{\theta_t}$.  

\section{Main Results} \label{Sec:main}
In this section, we will examine the excess risk for causal and anti-causal learning under various conditions of distribution shift, e.g., covariate shift \citep{gretton2009covariate}, target shift \citep{zhang2013domain}, concept drift \citep{cai2021transfer}, etc. 

Before diving into the details, we informally outline our main results in Table~\ref{tab:result} under log-loss. Recall that in both causal and anti-causal learning, the goal is to learn the conditional distribution $P_T(Y|X)$ such that the label $Y$ can be predicted from the feature $X$ in the target domain. In causal learning, this corresponds to learning the outcome random variables $Y_{t,x_i}$. However, the unlabelled target data $X_t$ (``cause" in this case) do not contain information about $Y_{t, x_i}$ as they are independent under causal generating processes. Therefore, the unlabelled target data are not useful in the causal learning case, as indicated in the table. The usefulness of the source data depends on the causal settings.  When the labelling distribution is invariant across two domains (e.g., $P_S(Y|X) = P_T(Y|X)$), the source data help reduce excess risk by providing information about $Y_{t,x_i}$,  which is identical to $Y_{s,x_i}$. The learning rate is then shown to be $O(\frac{k}{m})$, where $k$ is the number of parameters and $m$ is the size of the source sample. On the other hand, if $P_S(Y|X) \neq P_T(Y|X)$, the source data generally do not provide information about $Y_{t,x_i}$ and the excess risk does not converge to zero even with sufficient source and target data.

\begin{table*}[h!]
    \centering
    \begin{small}
    \begin{tabular}{|c|c|c|c|c|}
    \hline 
     Causal Setting  & Conditions & UT  & LS  & Rate \\
    \hline
    \multirow{4}{*}{$X \rightarrow Y$} & $P_S(X) \neq P_{T}(X)$, $P_S(Y|X) \neq P_{T}(Y|X)$   & \xmark   &  \xmark  &  - \\
    \cline{2-5}
    & $P_S(X) \neq P_{T}(X)$, $P_S(Y|X) = P_{T}(Y|X)$  & \xmark & \cmark   &  $O(\frac{k}{m})$ \\
    \cline{2-5}
    & $P_S(X) = P_{T}(X)$, $P_S(Y|X) \neq P_{T}(Y|X)$ & \xmark & \xmark   &   - \\
    \cline{2-5}
    & $P_S(X) = P_{T}(X)$, $P_S(Y|X) = P_{T}(Y|X)$  & \xmark  &  \cmark  &  $O(\frac{k}{m})$ \\
    \hline 
    \hline 
    \multirow{4}{*}{$Y \rightarrow X$} &  $P_S(Y) \neq P_{T}(Y)$, $P_S(X|Y) \neq P_{T}(X|Y)$   & \cmark  &  \xmark  & $O(\frac{1+k'}{n})$ \\
    \cline{2-5}
    & $P_S(Y) \neq P_{T}(Y)$, $P_S(X|Y) = P_{T}(X|Y)$   & \cmark  &  \cmark & $O(\frac{1}{n} + \frac{k'}{n+m})$ \\
    \cline{2-5}
    & $P_S(Y) = P_{T}(Y)$, $P_S(X|Y) \neq P_{T}(X|Y)$   & \cmark  & \cmark &  $O(\frac{k'}{n} + \frac{1}{n+m})$\\
    \cline{2-5}
    & $P_S(Y) = P_{T}(Y)$, $P_S(X|Y) = P_{T}(X|Y)$  &  \cmark  & \cmark  & $O(\frac{k'+1}{m+n})$\\
    \hline 
    \end{tabular}
    \end{small}
    \caption{(Informal) results on the effectiveness of source and unlabelled target data under causal and anti-causal learning problems.``\cmark" and ``\xmark" marks indicate whether the data are useful or not for the prediction under specific conditions and causal settings. ``UT" and ``LS" are abbreviated for ``Unlabelled Target" and ``labelled Source", respectively. The rate illustrates the convergence for the excess risk under log-loss in terms of the target sample size $n$ and source sample size $m$. The ``-" sign in the rate column means the risk will not converge to zero even if we have sufficient source and target data. }
    \label{tab:result}
\end{table*}

In anti-causal learning scenario ($Y \rightarrow X$, $P_S(X,Y) \neq P_T(X,Y)$), however, learning $P_T(Y|X)$ requires to estimate all the parameters of $Y_t$ and $X_{t,y_i}$. Unlike causal learning, where $P_T(Y|X)$ is fully represented by the random outcome variables $Y_{t,x_i}$, in this case, we need to infer $P_T(Y|X)$ from the joint distribution $P_T(X, Y)$. We will show that the unlabelled target data is always useful in anti-causal learning under certain conditions. The source data can also contribute to learning, depending on the assumptions we have made about the distribution shift. For example, if $P_S(Y) \neq P_T(Y)$ and $P_S(X|Y) \neq P_T(X|Y)$ with the independence assumption, there is no reason for the source data to be useful for prediction in the target domain. Therefore, the rate, in this case, is $O(\frac{k'+1}{n})$, which solely depends on the number of unlabelled target data. Intuitively, this is the cost of learning $k'+1$ parameters with $n$ unlabelled target samples. Under the target shift condition ($P_S(Y) \neq P_T(Y)$ and $P_S(X|Y) = P_T(X|Y)$), the source data helps in learning the outcome variables $X_{y_i}, i=1,\ldots, k'$, which is evinced in the rate $O(\frac{1}{n}+\frac{k'}{m+n})$ that constitutes the learning of $Y_t$ (with associated parameter $\theta^{t*}_{Y}$) with a rate $O(\frac{1}{n})$ and $X_{t,y_i}, i=1,\ldots,k'$ (with associated parameters $\theta^{t*}_{X_{y_i}}$) with a rate $O(\frac{k'}{n+m})$. Similarly, for the conditional shift ($P_S(Y) = P_T(Y)$ and $P_S(X|Y) \neq P_T(X|Y)$), the rate becomes $O(\frac{k'}{n}+\frac{1}{m+n})$  where sufficient source data boosts the learning of $Y_t$ (associated with parameter $\theta^{t*}_{Y}$) with a rate $O(\frac{1}{m+n})$, but are not helpful for learning outcomes variables $X_{t, y_i}$.
    
As a special case of domain adaptation, we also consider SSL where $P_S(X,Y) = P_T(X,Y)$. Using the same arguments in causal and anti-causal settings, we obtain a better rate of $O(\frac{k'+1}{m+n})$ in anti-causal learning, where the unlabelled target data take effect on prediction, compared to $O(\frac{k}{m})$ in causal learning, where the unlabelled target data are not helpful. For readers interested in empirical verification of our results, we substantiate the analysis with a toy example, which can be found in Section~\ref{sec:example}. More generally, our analysis also holds for the case when the cause is discrete and the effect can be either discrete or continuous. This is practically useful since the datasets in many real classification problems are usually anti-causal with a finite label space $\mathcal{Y}$ where the feature space is usually continuous \citep{scholkopf2012causal,zhang2013domain,gong2016domain}. To summarize, different causation directions incentivize different learning complexity for generalization, which is reflected in the number of model parameters and the effectiveness of the data. It comes naturally when we could model both the source and target data from either $X \rightarrow Y$ or $Y \rightarrow X$ in some non-identifiable circumstances, we need to take the distribution shift conditions and sample sizes into account to achieve better learning performance. We will first show our main proof techniques in Section~\ref{sec:proof} and examples are followed in Section~\ref{sec:example}.

Many theoretical results on generalization in domain adaptation depend on distributional conditions and algorithms. Notably, based on the covariate shift condition, \citet{kpotufe2018marginal} propose the ``transfer component" that evaluates the support overlap between the source and target domains and derives the minimax rate for the generalization error. However, such a notion cannot be generally applied to other distribution shift conditions. Similarly, \citet{cai2021transfer} determine the optimal minimax rate of convergence with the weighted $k$-nearest neighbour classifier using the notion of ``relative signal exponent" based on the concept drift condition. Under the target shift condition, \citet{maity2020minimax} and \citet{gong2016domain} derive the learning guarantees for the distribution reweighting strategies, which are algorithm-dependent. While our analysis is restricted to parametric models, it applies to all possible distribution shift conditions. This applicability facilitates a unified framework for assessing learning performance from a causal viewpoint. It also offers an intuitive understanding of the values derived from source and target data. In particular, our result of the covariate shift condition offers the same insight when $P_T(X)$ is absolutely continuous w.r.t. $P_S(X)$ in \citet{kpotufe2018marginal}, where the labelled source has the same value as the labelled target. The target shift result agrees with \citet{maity2020minimax} in the sense that the unlabelled target is equally useful as the labelled target data, achieving a rate of $O(\frac{1}{n})$. Under the concept drift condition, we argue that the excess risk does not converge, which is consistent with Theorem 3.1 in \citet{cai2021transfer} for a large relative signal exponent and no labelled target data. Moreover, we prove in Lemma~\ref{lemma:wsc} that the excess risk is minimax optimal under log-loss.

\subsection{Information-theoretic Characterization} \label{sec:proof}
In this section, we will outline our primary proof techniques for the findings presented in Table~\ref{tab:result}. Our proofs primarily build upon the work of \cite{merhav1998universal} and \cite{zhu2020semi}, which originally focused on the sequential learning problem or semi-supervised learning problem. However, we extend their results by applying the mixture strategy to the UDA and SSL problems with the information-theoretic framework. To begin with, we first consider the \emph{log-loss} (also known as the logarithmic loss), which is formally defined as follows.
\begin{definition}[Log-loss] Let the predictor $b$ be a probability distribution over the target label $Y'_t$. The log-loss is then defined as,
\begin{equation}
    \ell(b, Y'_t) = - \log b(Y'_t).
\end{equation}
\label{def:logloss}
\end{definition}
\noindent Given the testing feature $X'_t$, training data $D^m_s$ and $D^{U,n}_t$, 
we may view the predictor $b$ as the conditional distribution $Q(Y'_t|D^{\textup{U},n}_t, D^m_s,X'_t)$ over the unseen target label given the testing feature $X'_t$ and the training data $D^m_s, D^{\textup{U},n}_t$. It could be proved that the true predictor $b^*$ is given by the underlying target distribution as $b^*(Y'_t) = P_{\theta^*_t}(Y'_t|X'_t)$. Then the excess risk can be expressed as,
\begin{align}
    \mathcal{R}(b)&= \mathbb{E}_{\theta^*_t,\theta^*_s}\left[  \log \frac{P_{\theta^*_t}(Y'_t|X'_t)}{Q(Y'_t|D^{\textup{U},n}_t, D^m_s,X'_t)}\right] \label{eq:excessrisk}
\end{align}
Concerning the choice of the predictor $Q(Y'_t|D^{\textup{U},n}_t, D^m_s,X'_t)$,  we first define $\Theta_s$ and $\Theta_t$ as random vectors over $\Lambda$, which can be interpreted as a random guess of $\theta^*_s$ and $\theta^*_t$. Note that $\Theta_s$ and $\Theta_t$ may share some common parameters, e.g., $\Theta_{s,i}=\Theta_{t,i}$ for $i$th entry. Then by \emph{mixture strategy} \citep{merhav1998universal,xie2000asymptotic}, we assign a probability distribution $\omega$ over $\Theta_s$ and $\Theta_t$ w.r.t. the Lebesgue measure to represent our prior knowledge and update the posterior with the incoming data to approximate the underlying distributions.  That is, 
\begin{align}
    Q(Y'_t|D^{\textup{U},n}_t, D^m_s,X'_t) &= \frac{\int P_{\theta_t}(D^{\textup{U},n}_t, X'_t, Y'_t) P_{\theta_s}(D^m_s) \omega(\theta_t,\theta_s) d\theta_t d\theta_s }{\int P_{\theta_t}(X'_t)P_{\theta_t}(D^{\textup{U},n}_t) P_{\theta_s}(D^m_s) \omega(\theta_t,\theta_s) d\theta_t d\theta_s  } \nonumber \\
    &= \int  P_{\theta_t}(Y'_t|X'_t) P(\theta_t,\theta_s|X'_t, D^m_s, D^{\textup{U},n}_s) d\theta_sd\theta_t. \label{eq:mixture}
\end{align}
We can interpret~(\ref{eq:mixture}) as estimating $Y'$ in a two-step procedure.  With a prior distribution $\omega$, the first step is to learn the parameters $\theta_s$, $\theta_t$ with the joint posterior $P(\theta_s,\theta_t|D^m_s,D^{\textup{U},n}_t, X'_t)$. In the second step, the learned $\theta_t$ is applied for prediction in terms of the parametric distribution $P_{\theta_t}(Y'_t|X'_t)$. One way to comprehend the mixture strategy is that we encode our prior knowledge over target and source domain distributions in terms of the prior distribution $\omega(\Theta_s, \Theta_t)$, and different distribution shift conditions correspond to different priors. 
By the mixture strategy, we give the excess risk under log-loss.
\begin{theorem}[Excess Risk with Log-loss]\label{thm:excessrisk-log}
Under log-loss, let the predictor $Q$ be the distribution in~(\ref{eq:mixture}) with the prior distribution $\omega(\Theta_s,\Theta_t)$. Then the excess risk can be expressed as
\begin{equation}
    \mathcal{R}(b)  =  I(Y'_t; \theta^*_t, \theta^*_s|D^{m}_s, D^{\textup{U},n}_t, X'_t),
\end{equation}
where the R.H.S. denotes the conditional mutual information $I(Y'_t; \Theta_t, \Theta_s|D^{m}_s, D^{\textup{U},n}_t, X'_t)$ evaluated at $\Theta_t = \theta^*_t$ and $\Theta_s = \theta^*_s$.
\end{theorem}
All proofs in this paper can be found in the Appendix. A similar learning strategy can be used for more general loss functions. Given a general loss function $\ell$, we define the predictor $b$ as
\begin{equation}
    b = \argmin_{b} \mathbb{E}_{Q}\left[ \ell(b,Y'_t)|X'_t, D^{U,n}_t, D^m_s \right], \label{eq:general_b_bound}
\end{equation}
with the choice of the mixture strategy 
\begin{equation*}
    Q(X'_t, Y'_t, D^n_t,D^{m}_s) = \int P_{\theta_t,\theta_s}(X'_t,Y'_t, D^{U,n}_t,D^{m}_s)\omega(\theta_t, \theta_s)d \theta_t d\theta_s
\end{equation*}
for some prior $\omega$. The optimal predictor is then given by
\begin{equation}
b^* = \argmin_{b} \mathbb{E}_{\theta^*_t}\left[ \ell(b,Y'_t)|D^{U,n}_t, X'_t. \right]\label{eq:general_bkstar_bound}
\end{equation}
We have the following theorem for $\beta$-exponential concave loss functions as follows.
\begin{theorem}[Excess Risk with Exponential Concave Loss]\label{thm:excessrisk-expccv}
Assume the loss function is $\beta$-exponentially concave of $b$ for any $y$. Then the excess risk induced by $b$ and $b^*$ in Equation~(\ref{eq:general_b_bound}) and~(\ref{eq:general_bkstar_bound}) can be bounded as
\begin{equation}
    \mathcal{R}(b)  \leq \frac{1}{\beta}I(Y'_t; \theta^*_t, \theta^*_s|D^{m}_s, D^{\textup{U},n}_t, X'_t). \label{ineq:expccv}
\end{equation}
\end{theorem}
The log-loss can be regarded as a special case with $\beta = 1$. One can refer to Lemma 1 (also the proof) in \citet{zhu2020semi} for more details and comments, which we will not repeat in our context. Likewise, if the loss function is bounded, we arrive at the following theorem.
\begin{theorem}[Excess Risk with Bounded Loss]\label{thm:excessrisk-general}
Assume the loss function satisfies $|\ell(b,y) - \ell(b^*,y)| \leq M$ for any observation $y$ and any two predictors $b,b^*$. Then the excess risk can be bounded as
\begin{equation}
    \mathcal{R}(b)  \leq M\sqrt{2 I(Y'_t; \theta^*_t, \theta^*_s|D^{m}_s, D^{\textup{U},n}_t, X'_t)}. \label{ineq:bound}
\end{equation}
\end{theorem}
From the above theorems, we can see the analogy that the expected regrets induced by the mixture strategy are both characterized by CMI evaluated at $\theta^*_t$ and $\theta^*_s$. Note that these results apply to both causal and anti-causal learning problems. Nevertheless, the characterization of learning performance in its present form is less informative because it does not show the effect of sample sizes and causal directions. To this end, we make some regularity assumptions on the parametric conditions \citep{clarke1990information, merhav1998universal, zhu2020semi} and define the proper prior distribution to obtain an asymptotic approximation.

\begin{assumption}[Parametric Distribution Conditions] \label{asp:para-trans} 
With the aforementioned parameterization, let $\mathbf{\theta}^* = (\theta^*_s, \theta^*_{t})$ denote the underlying parameters for labelled source and unlabelled target data.  We assume:
\begin{itemize}
\item \textbf{Condition 1:} The source and target distributions $P_{\theta_s}(X_s,Y_s)$ and $P_{\theta_{t}}(X_t)$ is twice continuously differentiable at $\theta_s^*$ and $\theta^*_{t}$ for almost every $(X_s,Y_s)$ and $X_t$. \label{cond:1} 

\item \textbf{Condition 2:} Define the Fisher information matrix
\begin{align*}
    I_s &= -\mathbb{E}_{\theta^*_s}[ \nabla^2 \log P(X_s,Y_s|\theta^*_s)], \\
    I_t &= -\mathbb{E}_{\theta^*_t}[ \nabla^2 \log P(X_t|\theta^*_{t})], \\
    I_{0} &= -\mathbb{E}_{\theta^*}[ \nabla^2 \log P(X_t, X_s, Y_s|\theta^*)].
\end{align*}
We assume $I_s$ and $I_t$ are positive definite and it holds that $I_{0}$ is also positive definite. \label{cond:2}

\item \textbf{Condition 3} \citep{clarke1990information}: Assume that the convergence of a sequence of parameter values is equivalent to the weak convergence of the distributions they index. Particularly:
\begin{align*}
    \theta_s \rightarrow \theta^*_{s} &\Leftrightarrow P_{\theta_s}(X,Y) \rightarrow P_{\theta^*_{s}}(X,Y), \\
    \theta_{t} \rightarrow \theta^*_{t} &\Leftrightarrow P_{\theta_{t}}(X) \rightarrow P_{\theta_{t}^*}(X),
\end{align*}
for source and target domains, respectively. \label{cond:3}
\item \textbf{Condition 4}: Assume that for all $\theta_s$ in some neighbourhood of $\theta^*_s$ and $\theta_t$ in some neighbourhood of $\theta^*_t$, the normalized R\'enyi divergences of order $1+\lambda$, the following holds
\begin{align}
    &\log \int P_{\theta^*_s}(x,y)^{1+\lambda} P_{\theta_s}(x,y)^{-\lambda} dxdy < \infty, \\
    &\log \int P_{\theta^*_{t}}(x)^{1+\lambda} P_{\theta_{t}}(x)^{-\lambda} dx < \infty
\end{align}
for sufficiently small $\lambda > 0$.
\label{cond:4}
\item \textbf{Condition 5}: Assume that for all $\theta_s$ in some neighbourhood of $\theta^*_s$ and $\theta_{t}$ in some neighbourhood of $\theta^*_{t}$, the moment generating function is bounded as
\begin{align}
    \mathbb{E}_{\theta^*_s}\left[e^{\lambda \frac{\partial^{2}}{\partial \theta_{j} \partial \theta_{k}}\log p(X_s, Y_s \mid \theta_s)} \right] < \infty, \\
    \mathbb{E}_{\theta^*_{t}}\left[e^{\lambda \frac{\partial^{2}}{\partial \theta_{j} \partial \theta_{k}}\log p(X_t \mid \theta_{t})} \right] < \infty,
\end{align}
for all $j, k=1, \ldots, d$ with some small $\lambda>0$, where $d$ is determined based on the causal settings and conditional shifting conditions.
\label{cond:5}
\item \textbf{Condition 6:} Let $l_{s}:=\left[(\nabla \log p\left(X, Y \mid \theta^*_{s}\right)), \mathbf{0}_{d''} \right]^T$, $l_{t} :=[\mathbf{0}_{d''}, \nabla \log p\left(X \mid \theta^*_{t}\right)]^T$, where $\mathbf{0}_{d''}$ denotes the zero vector with length $d''$, and $d''$ denotes the number of distribution parameters for both source and target domains. We also define $l'_s, l'_t$ as an independent copy of $l_s$ and $l_t$, respectively. We assume the moment-generating functions
\begin{align*}
&\mathbb{E}\left[e^{\lambda l_s^{T} I_0 l_s}\right], \mathbb{E}\left[e^{\lambda l_s^{T}I_0 l'_s}\right], \mathbb{E}\left[e^{\lambda l_t^{T}I_0 l_t}\right], \\
& \mathbb{E}\left[e^{\lambda l^{T}_t I_0 l'_t}\right],  \mathbb{E}\left[e^{\lambda l^{T}_t I_0 l_s}\right]
\end{align*}
exist for some small enough $\lambda>0$.
\label{cond:6}
\end{itemize} 
\end{assumption}

\begin{assumption}[Proper Prior] \label{asp:proper-prior}
We assume that the prior distribution $\omega(\Theta_s, \Theta_t)$ is continuous and positive over its whole support.
\end{assumption}
\begin{remark}
We impose the first three conditions on parametric distributions with the proper prior distribution to ensure that the posterior distribution of $\Theta_t$ and $\Theta_s$ asymptotically concentrates on neighbourhoods of $\theta^*_{t}$ and $\theta^*_s$ under both causal settings given sufficient source and target data. In particular, the positive definite Fisher information matrix and parameter uniqueness assumption imply that $\theta^*_{t}$ and $\theta^*_s$ are identifiable within $\Lambda$. We also impose some technical conditions to ensure that the posterior of the parameters converges to their true values at an appropriate rate. Additionally, for the anti-causal setting $Y \rightarrow X$, we exclude the case when outcome variable $X$ has the same distribution for all $y_i$ with Condition 2, that is,  $P_{\theta_{X_{y_i}}}(X)$ is identical for all $y_i \in \mathcal Y$. Because in this case, $X$ and $Y$ are effectively independent, and the fisher information $I_t$ is no longer positive definite as the distribution of $X$ no longer depends on the parameter $\theta^*_{Y}$.
\end{remark}

\begin{remark}
The last three technical conditions are adopted and modified from \citet{zhu2020semi} to ensure that the posterior of the parameters converges to their true values at an appropriate rate for both source and target domains. We will mainly use these conditions for asymptotic estimation of KL divergence, e.g., see proof of Lemma~\ref{lemma:asymptotics}.
\end{remark}

\begin{remark}
Though asymptotically, the prior distribution does not affect the learning rate, its choice is crucial in practice, particularly with limited data. Priors should be selected based on parameter understanding, model complexity, and existing knowledge. For simple parametric models such as generalized linear models, we can adopt conjugate priors \citep{diaconis1979conjugate, chen2003conjugate} for updating parameters easily. For more complex models, we may require non-conjugate priors where the data are used to estimate the parameters of the prior distribution \citep{efron2012large, carlin2008bayesian}. This is particularly useful when we have little prior knowledge about the distribution. In practice, the sensitivity analysis could also be conducted to assess the robustness of the posterior distribution to the choice of prior. This helps ensure that the posterior is not unduly influenced by the choice of prior. 
\end{remark}

\subsection{Excess Risk in Causal Learning}

In this section, we will characterize the excess risk asymptotically under causal learning. We first consider the learning scenario when $P_S(Y|X) = P_T(Y|X)$, which corresponds to SSL if $P_T(X) = P_S(X)$ and covariate shift regime otherwise. The random vector $\Theta_s$ and $\Theta_t$ can be explicitly written as
\begin{align}
    \Theta_s = (\Theta^s_X, \Theta^s_{Y_{x_1}}, \cdots, \Theta^s_{Y_{x_k}}) = (\Theta^s_X, \Theta^s_{Y_X}), \label{eq:para-causal-s} \\
    \Theta_t = (\Theta^t_X, \Theta^t_{Y_{x_1}}, \cdots, \Theta^t_{Y_{x_k}}) = (\Theta^t_X, \Theta^t_{Y_X}), \label{eq:para-causal-t}
\end{align}
where $\Theta_{Y_X} = (\Theta_{Y_{x_1}}, \cdots, \Theta_{Y_{x_k}})$ for succinctness. 
We assume $\Theta^t_X$ and $\Theta^s_X$ are independent of $\Theta^s_{Y_{X}}$ and $\Theta^t_{Y_{X}}$, but we will keep $\Theta^s_{Y_{X}}$ and $\Theta^t_{Y_{X}}$ identical according to the assumption $P_S(Y|X) = P_T(Y|X)$, written as $\Theta^{st}_{Y_X}$. With the proper prior distribution, we simplify the mixture distribution $Q$ as follows by omitting the unlabelled target data as follows:
\begin{align*}
 Q(Y'_t|D^{\textup{U},n}_t, D^m_s,X'_t) 
     = \int P(Y'_t|\theta^{st}_{Y_X}, X'_t) P(\theta^{st}_{Y_X}|D^m_s)
    d\theta^{st}_{Y_X},  
\end{align*}
where the knowledge transfer depends on the conditional posterior $P(\theta^{st}_{Y_X}|D^m_s)$. Since $P_S(Y|X) = P_T(Y|X)$,  without any labels from the target domain, we can only learn the parameters of the random outcomes $Y_X$ from the source data. On the other hand, if the assumption $P_S(Y|X) = P_T(Y|X)$ does not hold, namely, the concept drift if $P_S(X) = P_T(X)$ and general shift condition otherwise, the mixture strategy in (\ref{eq:mixture}) becomes 
\begin{align}
    Q(Y'_t|D^{\textup{U},n}_t, D^m_s,X'_t) & = \int P_{\theta^t_{Y_{X'_t}}}(Y'_t) \omega(\theta^t_{Y_{X'_t}}) d\theta^t_{Y_{X'_t}}  \nonumber 
\end{align}
due to the mutual independence properties of the distribution parameters.  In this case, neither the unlabelled target data nor the source data are useful for the estimation, the prediction is only piloted by the prior distribution $\omega(\theta^t_{Y_X})$ as the initial estimate for $\theta^{t*}_{Y_X}$. As a result, the excess risk, in this case, does not go to zero even if we have enough source and target data. 
To formally state the idea, we give the asymptotic estimation in the following main theorem.
\begin{theorem}[Excess Risk with Causal Learning] \label{thm:causal} 
In addition to Assumption~\ref{asp:para-dist},\ref{asp:para-trans} and~\ref{asp:proper-prior}, we also assume that $X$ causes $Y$ in both source and target domains. Let $\Theta_s$ and $\Theta_t$ be parameterized in (\ref{eq:para-causal-s}) and (\ref{eq:para-causal-t}).  As $m \rightarrow \infty$, the mixture strategy under log-loss yields:
\begin{itemize}
    \item (General shift and Concept drift) For any $P_{\theta^{t*}_{X}}(X) \ll P_{\theta^{s*}_{X}}(X)$, if $P_S(Y|X) \neq P_T(Y|X)$:
    \begin{align}
        \mathcal{R}(b) = \mathbb{E}_{\theta^{t*}_X}[\textup{KL}(P_{\theta^{t*}_{Y_{X'_t}}}(Y'_t)\|Q(Y'_t|X'_t)],
    \end{align}
    where $Q(Y'_t|X'_t) = \int P_{\theta^t_{Y_{X'_t}}}(Y'_t) \omega(\theta^t_{Y_{X'_t}}) d\theta^t_{Y_{X'_t}}$ for a certain prior $\omega$ over $\Theta^t_{Y_{X'_t}}$. 
    \item (covariate shift and SSL) For any $P_{\theta^{t*}_{X}}(X) \ll P_{\theta^{s*}_{X}}(X)$, if $P_S(Y|X) = P_T(Y|X)$:
    \begin{equation}
        \mathcal{R}(b)  \asymp \frac{k}{m}.
    \end{equation}    
\end{itemize}
\end{theorem}
From the above theorem, it is clear that the target data are not useful without labels and $n$ does not occur in the rate. This is understandable because such data do not contain information about $P(Y|X)$ due to the independence assumptions between $X$ and $Y_{x_i}$. If the conditional distribution remains unchanged between source and target domains, the excess risk converges with the rate of $O(\frac{k}{m})$. 

\subsection{Excess Risk in Anti-Causal Learning}
We now turn to the opposite causal direction where $Y \rightarrow X$. Similarly, we define the random variable $\Theta_s$ and $\Theta_t$ with the same form as ~(\ref{eq:para-causal-s}) and~(\ref{eq:para-causal-t}) by
\begin{align}
    \Theta_s = (\Theta^s_Y, \Theta^s_{X_{y_1}}, \cdots, \Theta^s_{X_{y_{k'}}}) = (\Theta^s_Y, \Theta^s_{X_Y}), \label{eq:para-anti-s}\\
    \Theta_t = (\Theta^t_Y, \Theta^t_{X_{y_1}}, \cdots, \Theta^t_{X_{y_{k'}}}) = (\Theta^t_Y, \Theta^t_{X_Y}). \label{eq:para-anti-t}
\end{align}
At this stage, we do not particularize any conditions on the parameters. From the Bayes rule, we rewrite the mixture distribution $Q$ in terms of the above parameterization as
\begin{align*}
    Q(Y'_t|D^{\textup{U},n}_t, D^m_s,X'_t) &= \frac{\int P_{\theta_t}(D^{\textup{U},n}_t, X'_t, Y'_t) P_{\theta_s}(D^m_s) \omega(\theta_t,\theta_s) d\theta_t d\theta_s }{\int P_{\theta_t}(X'_t)P_{\theta_t}(D^{\textup{U},n}_t) P_{\theta_s}(D^m_s) \omega(\theta_t,\theta_s) d\theta_t d\theta_s  } \nonumber \\
    &= \frac{\int P(Y'_t|\theta_t, X'_t) P(\theta_t|D^{\textup{U},n}_t, X'_t,\theta_s)d\theta_t P(\theta_s|D^m_s) d\theta_s}{\int P(\theta_t|D^{\textup{U},n}_t, X'_t,\theta_s)d\theta_t P(\theta_s|D^m_s) d\theta_s} \nonumber \\
    &= \int P(Y'_t|\theta_t, X'_t) P(\theta_t|D^{\textup{U},n}_t, X'_t,\theta_s)d\theta_t P(\theta_s|D^m_s)d\theta_s . \label{eq:mixture-anti}
\end{align*}
To interpret, the mixture strategy first provides an estimate of $\theta_s$ from the source data, then knowledge is transferred from $\theta_s$ to $\theta_t$ with the prior distribution $\omega(\theta_t|\theta_s)$, which induces the posterior $P(\theta_t|X'_t,D^{U,n}_t, \theta_s)$ along with the features $X'_t,D^{U,n}_t$ in the target domain, since the unlabelled data may contain all the information of $\theta^{*}_t$ under the anti-causal parameterization. Eventually, the prediction of $Y'_t$ will be based on the estimated $\theta_t$ and $X'_t$. 

With condition 3 under Assumption~\ref{asp:para-trans}, we require that the true parameters $\theta^*_t$ are identifiable given sufficient unlabelled target data, where its distribution is a mixture distribution, i.e., $\sum_{y\in \mathcal{Y}} P_{\theta_Y}(y)P_{\theta_{X_{y}}}(X)$. In general, this is a strong condition where the mixture distributions, such as the Bernoulli mixture, do not satisfy the assumption \citep{gyllenberg1994non} and the parameters within their support are \textbf{not} identifiable. But for certain types of families, the parameters are identifiable up to \emph{label swapping}, such as Gaussian \citep{teicher1963identifiability}, exponential families \citep{barndorff1965identifiability}, and many other finite continuous mixture distributions \citep{mclachlan2019finite}. Under label swapping, the posterior of the parameters approaches one of all permutations \citep{marin2005bayesian} and our result holds only up to the permutation where we simply set $\theta^*$ to be the parameters for that permutation. To solve the label swapping problem, the methods proposed include the specification of parameterization constraints \citep{marin2005bayesian, mclachlan2019finite}, a relabelling algorithm \citep{stephens2000dealing}, and constraint clustering \citep{grun2009dealing}. Once the label swapping is addressed, the mixed distributions are identifiable \citep{titterington1985statistical, mclachlan2019finite} and our results hold for estimating the corresponding $\theta^*$ as well. For illustration, we give a simple example of a categorical mixture distribution identifiable by adding structural constraints to the parameterization in Section~\ref{sec:example}.
We will now consider different distribution shift scenarios under anti-causal learning and derive the corresponding asymptotic estimation for the excess risk.
\begin{theorem}[Excess Risk with Anti-causal Learning] \label{thm:anti-case}
In addition to Assumptions~\ref{asp:para-dist},~\ref{asp:para-trans} and~\ref{asp:proper-prior}, we also assume $Y \rightarrow X$ in both source and target domains. Let $\Theta_s$ and $\Theta_t$ be parameterized in (\ref{eq:para-anti-s}) and (\ref{eq:para-anti-t}). As $m \asymp n^p$ for some $p > 0$ and $n  \rightarrow \infty$, the mixture strategy under log-loss yields:
\begin{itemize}
    \item (General shift) If $P_S(Y) \neq P_{T}(Y)$, $P_S(X|Y) \neq P_{T}(X|Y)$,
    \begin{equation}
     \mathcal{R}(b)  \asymp \frac{1+k'}{n}.
    \label{eq:case1}
    \end{equation}
    \item (Conditional shift) If $P_S(Y) = P_{T}(Y)$, $P_S(X|Y) \neq P_{T}(X|Y)$,
    \begin{equation}
    \mathcal{R}(b)  \asymp \frac{k'}{n} + \frac{1}{n \vee n^p}.
    \label{eq:case2}
    \end{equation}
        \item (Target shift) If $P_S(Y) \neq P_{T}(Y)$, $P_S(X|Y) = P_{T}(X|Y)$,
    \begin{equation}
    \mathcal{R}(b)  \asymp \frac{1}{n} + \frac{k'}{n \vee n^p}. \label{eq:case3}
    \end{equation}
    \item (SSL) If $P_S(Y) = P_{T}(Y)$, $P_S(X|Y) = P_{T}(X|Y)$,
    \begin{equation}
    \mathcal{R}(b)  \asymp \frac{k'+1}{n \vee n^p}. \label{eq:case4}
    \end{equation}
\end{itemize}
\end{theorem}
In contrast to causal learning, in the general shift case, we can achieve good generalization ability only with the unlabelled target data, while the source data do not help at all. This result confirms the value of unlabelled data, which is consistent with the intuition from Figure~\ref{fig:causal-anti-dag}. In the conditional shift and target shift cases, we can further show that the source data can only help improve the excess risk from $O(\frac{k'+1}{n})$ to $O(\frac{k'+1-j}{n} + \frac{j}{n \vee n^p})$ depending on how many $j$ common parameters $\theta^*_s$ and $\theta^*_t$ share. Intuitively, $O(\frac{k'+1-j}{n})$ can be viewed as the learning cost for $k'+1-j$ domain-specific parameters and $O(\frac{j}{n \vee n^p})$ as the learning cost for domain-sharing parameters. Therefore, the source data are incapable of changing the overall rate since the unlabelled target data always dominates the rate. In SSL, the rate $O(\frac{k'+1}{n \vee n^p})$ indicates that unlabelled target data are as useful as the labelled source data and that sufficient source data (e.g., $p > 1$) can indeed change the convergence rate. The results show that the \emph{learning complexity} under different causal directions will vary. This crucial distinction discloses how the causal relationships affect the model complexity and its generalization ability. 

Our results in Theorem~\ref{thm:excessrisk-log}, \ref{thm:causal}, \ref{thm:anti-case} establish the convergence rate for the mixture strategy. Here we show that this strategy is in fact optimal for log-loss. 
\begin{lemma}[Worst-Case Excess Risk] \label{lemma:wsc} For log-loss, 
\begin{equation*}
 \min_{b}\max_{\theta^*_s,\theta^*_t}\mathcal{R}(b) = \max_{\omega(\theta_s, \theta_t)} I(Y'_t; \Theta_t, \Theta_s|D^{m}_s, D^{\textup{U},n}_t, X'_t),
\end{equation*}
where $(\Theta_t,\Theta_s)$ is endowed with some prior distribution $\omega$.
\end{lemma}
This lemma exactly characterizes the excess risk for log-loss in the worst case. It shows that the worst-case regret is captured by the same CMI term as in Theorem \ref{thm:excessrisk-log}, although maximized w.r.t. the prior distribution over the source and target parameters. However, it can be shown that the maximization does not change the convergence rate of the mutual information term \citep{clarke1994jeffreys,merhav1998universal}. In other words, the convergence rate in Theorem \ref{thm:causal}, \ref{thm:anti-case} is indeed optimal and cannot be improved using a different learning algorithm. Even though we only consider the log-loss in the previous analysis, the results can be extended straightforwardly in the case of other general loss functions, such as exponentially concave or bounded losses, where the excess risk is captured by the same CMI term in Theorem~\ref{thm:excessrisk-log} (see Theorem~\ref{thm:excessrisk-general} for bounded losses as an example).

\section{Experiments} \label{sec:example}
In this section, we begin by confirming our main results with a toy example, for which we elaborate on the case when the data can be modeled both as causal learning and anti-causal learning. Subsequently, we extend the idea to tackle real-world challenges like the classification of handwritten digits. For these scenarios, we parametrize the data distribution using the Gaussian mixture model as an approximation, and the insights drawn from our experimental results reflect a similarity to those deduced from our theoretical analysis, confirming the effectiveness of the source and target data in more complicated learning problems.
 
\subsection{A toy example}
We will numerically confirm our main results using a toy example. We consider a simple example where $\mathcal{Y} = \{0,1\}$ and $\mathcal{X} = \{1,2,3,4\}$. In causal learning, we model the data distributions as 
\begin{align*}
    &X  \sim \textup{Cat}(\theta_{x_1}, \theta_{x_2}, \theta_{x_3}, \theta_{x_4}) \\
    &Y_{x_i}  \sim \textup{Ber}(\theta_{Y_{x_i}}) \textup{ for } i = 1,2,3,4.
\end{align*}
We set $\theta^{t*}_X = (0.25,0.25,0.25,0.25)$ and $\theta^{t*}_{Y_X} = (0.3,0.4,0.5,0.6)$ for synthetic experiments, and we will vary $\theta^{s*}_{X} = (0.6,0.1,0.1,0.2)$ and $\theta^{s*}_{Y_X} = (0.5,0.5,0.3,0.5)$ for the covariate shift and concept drift conditions, respectively. The parameters are estimated using the maximum likelihood algorithm and used in the prediction. We run experiments 3000 repeatedly and the results are shown in Figure~\ref{fig:causal-result}. For the general shift case in (a), we fix $m=2000$ and vary $n$ from 500 to 16000 and it can be seen that with the unlabelled target sample increasing, the risk will remain around $0.34$ and hence does not converge in this case. We sketch the regret for covariate shift and semi-supervised learning in figures (b) and (d), here we fix $n = 2000$ and vary $m$ from 500 to 16000. It can be seen in that $\mathcal{R}(b)$ in blue converges to zero with $m$ increasing in these two cases, then we also plot the $\mathcal{R}(b)^{-1}$ in red to show the rate. The reciprocal of the excess risk is linear in the source sample size, which coincides with our theoretical analysis. It is worth pointing out that the slopes are different in these two cases because the quantity will depend on the Fisher information matrix of $P_{\theta^{s*}_{Y_X}}(Y)$ and the distribution of the covariate $X$ varies across two domains. For concept drift learning in (c), we fix $n = 2000$ and vary $m$ from 500 to 16000. Similar to the general shift case, the excess risk is maintained around 0.34 as well, which is independent of the source sample size $m$. 
\begin{figure*}[h]
\centering
\subfigure[Vary $n$\label{fig:causal_unlabel}]{\includegraphics[width =2.10in]{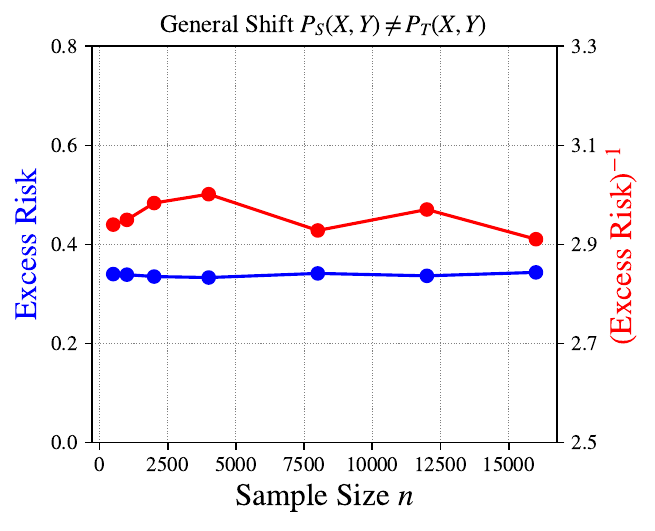}}
\subfigure[Fix $n$, vary $m$]{\includegraphics[width = 2.15in]{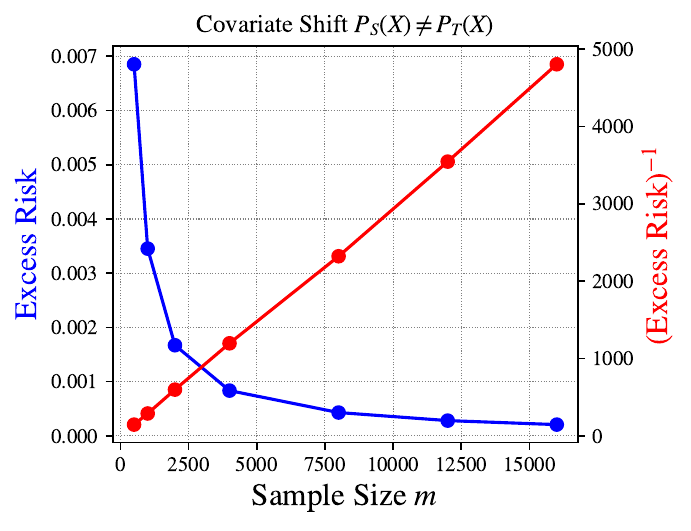}}
\subfigure[Fix $n$, vary $m$]{\includegraphics[width = 2.10in]{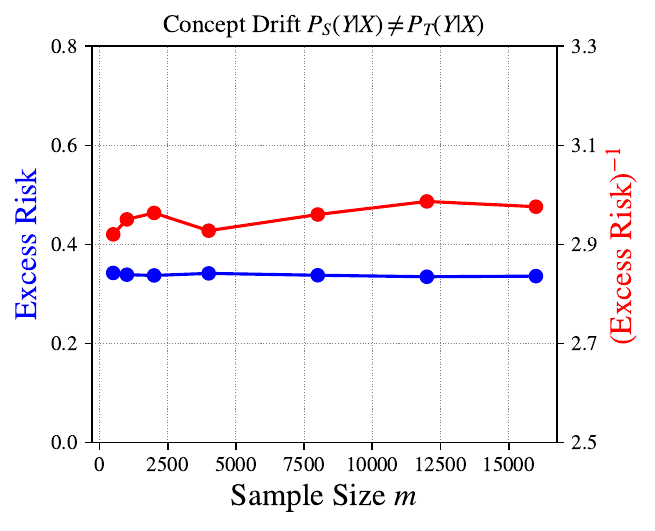}}
\subfigure[Fix $n$, vary $m$]{\includegraphics[width = 2.15in]{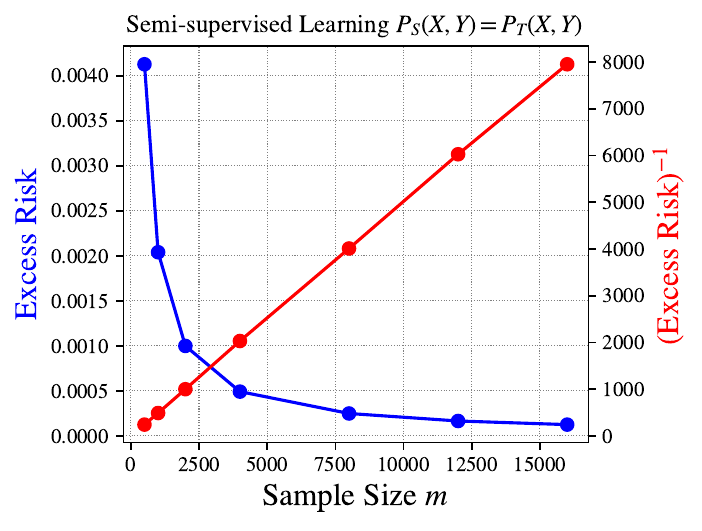}} 
\caption{Excess risk comparisons under causal learning. (a) and (c) represents the results of $\mathcal{R}(b)$ for general shift case and concept drift learning, where we vary $n$ from 500 to 16000 in (a), and fix $n = 2000$ but vary $m$ from 500 to 16000 in (c). We sketch the results $\mathcal{R}(b)$ for covariate shift and semi-supervised learning in (b) and (d), here we fix $n = 2000$ and vary $m$ from 500 to 16000.  We also plot $\mathcal{R}(b)^{-1}$ to show the rate w.r.t. $m$.  We plot all excess risks in blue and their reciprocals in red. All results are derived by 3000 experimental repeats.} \label{fig:causal-result}
\end{figure*}

In anti-causal learning, we will model the distributions of the outcome random variables as 
\begin{align*}
    Y &\sim \textup{Ber}(\theta_Y), \\
    X_{0} &\sim \textup{Cat}(\theta_{{0}}, \theta_{{0}}+0.55, \theta_{{0}} +0.2, 0.25 - 3\theta_{{0}}), \\
    X_{1} &\sim \textup{Cat}(\theta_{{1}}, \theta_{{1}}+0.25, 0.4 - 3\theta_{{1}} , \theta_{{1}}+0.35). 
\end{align*}
For experiments, we set $\theta^{t*}_{Y} = 0.5$ and $\theta^{t*}_{X_{0}} = \theta^{t*}_{X_{1}} = 0.05$ as an example, and we will vary $\theta^{s*}_{Y} = 0.7$ and $\theta^{s*}_{X_{0}} =  \theta^{t*}_{X_{1}} = 0.01$ for the target shift and conditional shift conditions, respectively. Using the maximum likelihood algorithm, we sketch the results in Figure~\ref{fig:anti-causal-result}. For the general shift case in (a), the excess risk converges as $n$ becomes larger, and more explicitly $\mathcal{R}(b)^{-1}$ is linear in $n$, which confirms our theoretical result. For target shift and conditional shift in (b) and (c), it can be seen that $\mathcal{R}(b)$ converges to a non-zero value $\lambda$ with $m$ increasing in these two cases, then we also plot the $(\mathcal{R}(b) - \lambda)^{-1}$ to show the rate w.r.t. the sample size $m+n$. These two curves indicate that the source data can only help reduce the excess risk up to a constant. For semi-supervised learning in (d), as expected, the excess risk will converge to zero as $m$ increases. It is also observed that the slope of the reciprocal is higher compared to the general shift condition, implying the source data contain more information than the unlabelled target data and lead to higher scaling factor $c$ (e.g., $O(\frac{c}{m})$) in the rate. We empirically depict the rate of learning performance under different causal mechanisms and domain shift conditions, from which the usefulness of the source and target data is manifested.

\begin{figure*}[h]
\centering
\subfigure[Vary $n$\label{fig:anti_causal_unlabel}]{\includegraphics[width = 2.10in]{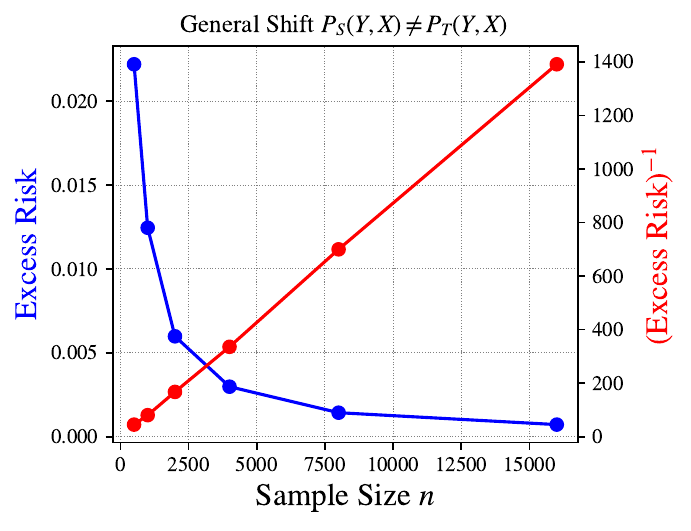}}
\subfigure[Fix $n$, vary $m$]{\includegraphics[width = 2.20in]{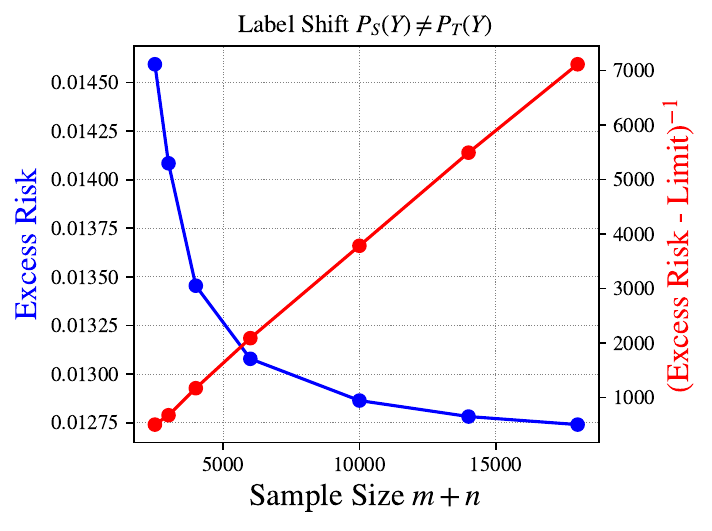}}
\subfigure[Fix $n$, vary $m$]{\includegraphics[width = 2.20in]{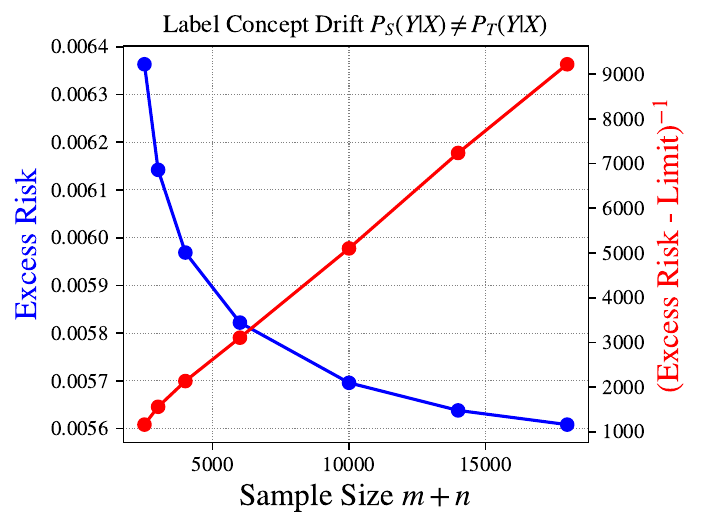}}
\subfigure[Fix $n$, vary $m$]{\includegraphics[width = 2.20in]{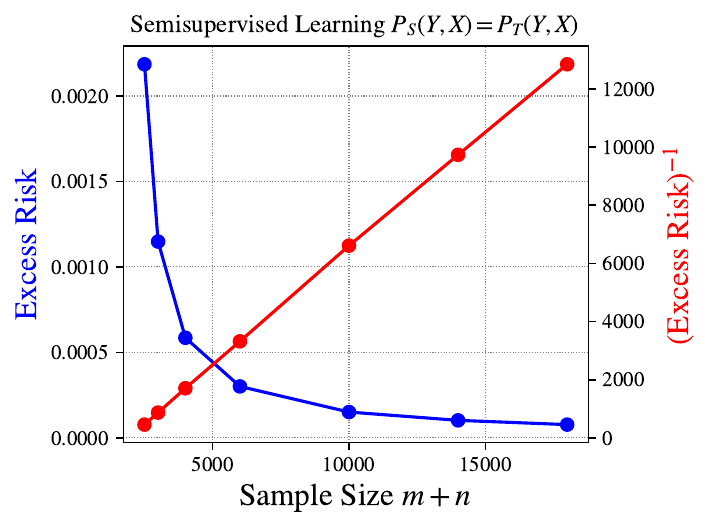}}
\caption{Excess risk comparisons under anti-causal learning. (a) represents the results of $\mathcal{R}(b)$ and $\mathcal{R}(b)^{-1}$ for general shift case, and we vary $n$ from 500 to 16000. We sketch the results $\mathcal{R}(b)$ for label shift, label concept drift and semi-supervised learning in (b), (c) and (d). Here we fix $n = 2000$ and vary $m$ from 500 to 16000. It can be seen in that $\mathcal{R}(b)$ converges to a non-zero value $\lambda$ with $m$ increasing in (b) and (c), then we also plot $(\mathcal{R}(b) - \lambda)^{-1}$ to show the rate w.r.t. $m+n$. We plot all the excess risks in blue and their reciprocals in red.  All results are derived by 3000 experimental repeats.} \label{fig:anti-causal-result}
\end{figure*}

\subsection{Experiments with Real Datasets}
In this section, we shift our focus to real-world datasets (e.g., the MNIST dataset) for anti-causal learning to further reinforce our idea in practical scenarios. Although the core of our analysis lies in the assumption that the data distribution is parametric, this is often not the case when dealing with real-world data. As such, we need to find a parametric model to approximate the true underlying distribution with finite samples. In the following, we use Gaussian mixture models (GMM) to approximate the data, where we assume each class label $y_i$ corresponds to a specific cluster of features and these features are modeled by a Gaussian distribution denoted as $P_{X_{y_i}}(x)$, with parameters including a mean vector $\mu_{i}$ and a covariance matrix $\Sigma_i$. Our implementation of this model is based on the expectation-maximization (EM) algorithm \citep{dempster1977maximum} by efficiently estimating the initial GMM parameters from the labelled data, and the parameters will be updated with the additional unlabelled data or data with a distributional shift. This framework has been applied to semi-supervised learning and unlabelled domain adaptation problems where the details are outlined in Algorithm~\ref{alg:gmms}. While there would exist a potential mismatch between the parametric model and the true underlying distribution and some estimation errors, the empirical results nevertheless demonstrate that anti-causal learning can enhance prediction performance when we efficiently use unlabelled target data and source data. 

\begin{algorithm}[htb]
\caption{Anti-Causal Learning with GMMs} \label{alg:gmms}
\DontPrintSemicolon 
\KwData{A small set of labelled target training dataset $\mathcal{D} = \{(x_i, y_i)\}_{i=1}^N$ with $N$ samples, where $x_i$ are features and $y_i$ are labels, unlabelled target training dataset $\mathcal{D}_U = \{x_i\}_{i=1}^M$ with $M$ samples, labelled source training dataset $\mathcal{D}' = \{(x_i, y_i)\}_{i=1}^L$ with $L$ samples and test dataset $\mathcal{D}_{T} = \{(x_i, y_i)\}_{i=1}^T$ with $T$ samples}
\KwResult{Improved prediction performance using GMM on $\mathcal{D}_{T}$.}
Initialize $K$, the number of Gaussian components, corresponding to the number of class labels.\;
Initialize parameters $\Theta = \{\mu_k, \Sigma_k\}_{k=1}^K$ for each Gaussian component.\;
\textbf{Step 1: Feature Engineering}\;
Conduct feature engineering using methods such as PCA or other feature selection with $\mathcal{D}$, and $\mathcal{D}_U$ or $D'$ depending on the SSL/UDA tasks\;
\textbf{Step 2: Parameter Estimation}\;
\For{each class label $k = 1$ to $K$}{
    Estimate $\mu_k$ and $\Sigma_k$ using EM algorithm on $\mathcal{D}$ with corresponding instances with label $k$.\;
}
\textbf{Step 3: SSL/UDA with GMM}\;
\While{not converged}{
    For SSL: Use unlabelled data $\mathcal{D}_U$ to update $\Theta$ by the EM algorithm\;
    For DA: Use labelled source data $\mathcal{D}'$ to update $\Theta$ by the EM algorithm\;
}
\textbf{Step 4: Prediction}\;
\For{each new instance $x$ in $\mathcal{D}_T$}{
    Predict label $y$ by selecting the Gaussian component $k$ that maximizes $P_{X_{k}}(x)$ with parameters $(\mu_k, \Sigma_k)$.\;
}
\end{algorithm}

\subsubsection*{Semi-supervised Learning}
The MNIST dataset\footnote{http://yann.lecun.com/exdb/mnist/}  (\cite{lecun1998gradient}) serves as a well-recognized standard for benchmarking, comprising 70,000 grayscale, handwritten digit images (ranging from 0 to 9), each of pixel size $28 \times 28$. It is a frequent choice for testing various machine learning algorithms, particularly in image classification scenarios. Our analysis will primarily focus on exploring the usefulness of unlabelled data under the anti-causal learning setting using the Gaussian mixture model, specifically with the MNIST dataset. To achieve this, we select two digits at random (for instance, 2 and 5) and construct a dataset comprising 100 labelled samples, while varying the unlabelled sample size from 0 to 1,000. By introducing a small set of labelled target data, we can accurately determine the correct labels, addressing the potential label-swapping issue that may arise with the unlabelled data only. Our goal is to demonstrate that incorporating unlabelled data can still improve the performance of the model effectively. The initial step in our approach involves data preprocessing, which includes applying principal component analysis (PCA) to both the labelled and unlabelled datasets to reduce the input feature dimensionality from 784 down to a manageable number - 20 in our experiment. This reduction aids in addressing the curse of dimensionality, enhancing the computational speed and potentially boosting the Gaussian mixture model's performance. Following this, we establish an initial Gaussian mixture model using the labelled data only. Then we follow the procedures in Algorithm~\ref{alg:gmms} to update the parameters of the initial GMM. We will finally compare the performance of the updated GMM with its initial model using a test set from the same digit pair with the size of 3,000. 

\begin{figure}[h]
    \centering
    \includegraphics[width = 6cm]{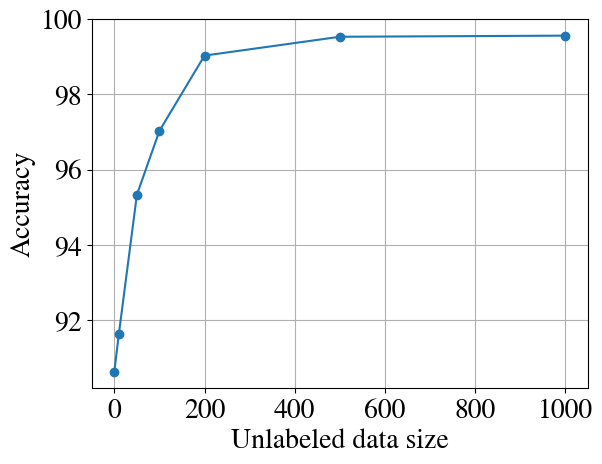}
    \caption{Accuracy v.s. unlabelled sample size for digit pair (2, 5) }
    \label{fig:mnist1}
\end{figure}

\begin{figure*}[h]
\centering
\subfigure[Initial GMM]{\includegraphics[width = 2.30in]{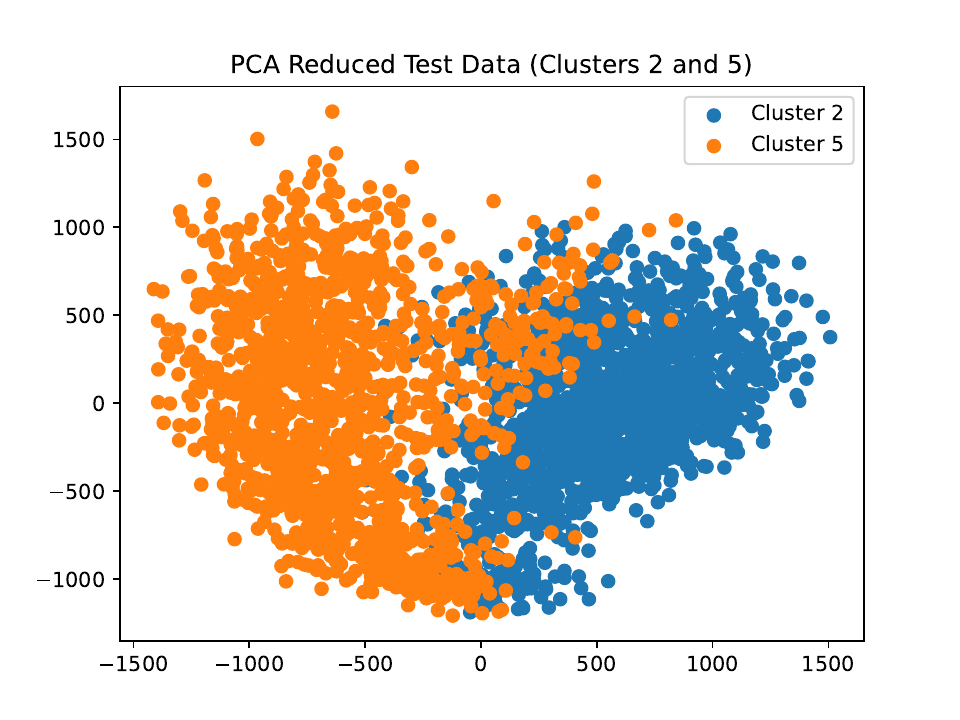}}
\subfigure[Update GMM]{\includegraphics[width = 2.30in]{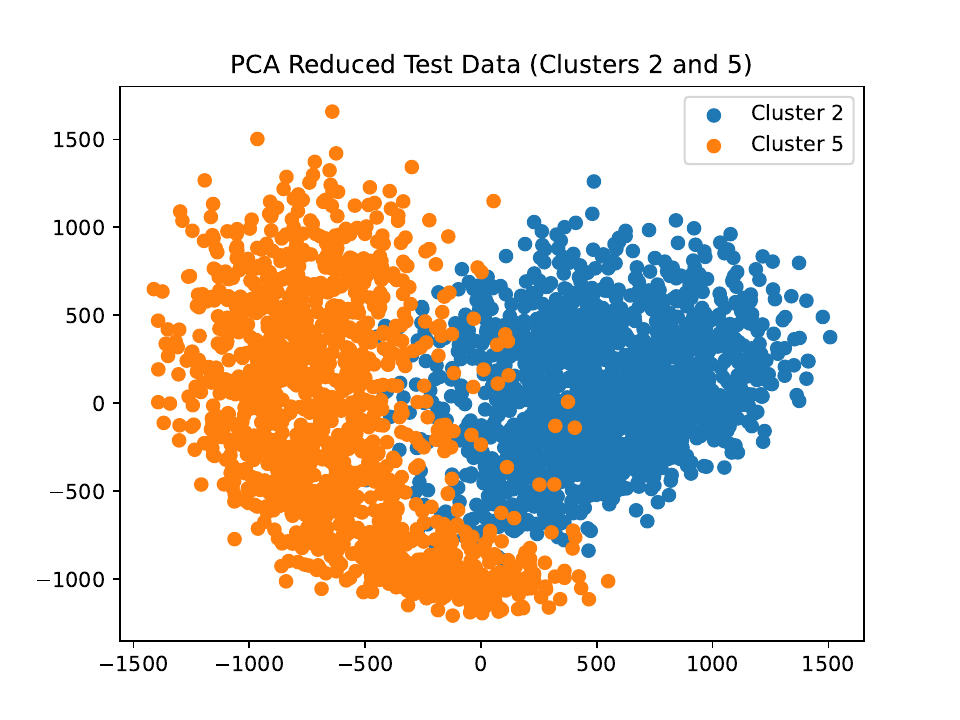}}
\caption{Visualization of clusters for various source and target combinations for digit pair (2, 5) for initial and updated GMM with 100 labelled data and 500 unlabelled data} \label{fig:anti-causal-result-visual}
\end{figure*}

Figure~\ref{fig:mnist1} illustrates the test set accuracy for different sizes of unlabelled data for the digit pair (2, 5). Our observations indicate that integrating unlabelled data significantly improves the model performance. Correspondingly, as the size of unlabelled data increases, the model accuracy also sees an increase, achieving approximately 99\% accuracy when the data size exceeds 500. This improvement indicates that unlabelled data indeed helps estimate the distribution parameters in the context of anti-causal learning, and this also empirically validates the results we presented in Table~\ref{tab:result}. To visualize the model performance on these two clusters, we further illustrate the clusters by plotting the two most significant principle components in Figure~\ref{fig:anti-causal-result-visual}. It demonstrates that updated GMM learning can indeed make two clusters more distinct and separable than the initial GMM, which leads to higher accuracy. To provide a more comprehensive demonstration of the usefulness of unlabelled data, we have randomly selected several additional digit pairs and conducted experiments with varying amounts of unlabelled data. We summarize the result in Table~\ref{tab:RES_ssl}. From the table, we can see that, in all cases, the unlabelled data help improve the accuracy in predictions, and as the sample size of unlabelled data increases, the accuracy also improves correspondingly. However, due to the variability between different digit pairs, and randomness from train and test sampling and estimation errors, the extent to which unlabelled data improves accuracy varies across different experiments. Through the experimental validation conducted on the MNIST dataset, our results confirm the substantial impact of unlabelled data on enhancing the performance of the anti-causal learning setting, particularly under conditions where labelled samples are limited. This establishes the crucial role of anti-causal learning settings in practical applications when it comes to semi-supervised learning problems. 

\begin{table}[hbpt!]
    \centering
    \begin{tabular}{|c|c|c|c|c|c|c|}
    \hline 
    Unlabelled Size  &  (2,5)  &  (5, 9)  & (3, 8)  & (4, 7) & (0, 6)  & (2, 3) \\
    \hline 
    0              & 0.896 & 0.531 &  0.575 & 0.854 & 0.893 & 0.855 \\
    50             & 0.941 & 0.623 &  0.817 & 0.858 & 0.942 &  0.865\\
    200            & 0.990 & 0.636 &  0.852 & 0.905 & 0.983 & 0.884 \\
    500            & 0.991 & 0.774 &  0.891 & 0.917 & 0.985 &  0.937 \\
    \hline 
    \end{tabular}
    \caption{Performance comparison of different sizes of datasets on various digit pairs}\label{tab:RES_ssl}
\end{table}

\subsubsection*{Unlabelled Domain adaptation}
We further assess the effectiveness of anti-causal learning in the realm of unlabelled domain adaptation. Here, we include three different source data domains for comparisons: the United States Postal Service (USPS) dataset \citep{uspsdataset}, an adapted MNIST dataset with added Gaussian noise, and a colour-infused MNIST dataset with colored backgrounds added to the digits. The USPS dataset, frequently used for digit recognition and domain adaptation tasks, consists of 9,298 grayscale images of handwritten digits (0-9) with a pixel resolution of $16 \times 16$. For the target domain, we randomly select two digits from the MNIST dataset to create a dataset containing 100 labelled samples. Subsequently, we will introduce the aforementioned three source data, each with 500 labelled samples, to help update the distribution parameters learned from the initial GMM. We aim to examine whether introducing an additional labelled dataset can significantly improve model performance, particularly when the causal mechanisms and generating distributions are closely similar. We apply a similar algorithm used in semi-supervised learning where we first apply PCA to both source and target data, and then we construct an initial GMM with the target data and then update the GMM using the EM algorithm on the source data. Here we pick various digit pairs to evaluate the effectiveness of the source data, and the results are summarized in Table~\ref{tab:RES}.

\begin{table}[hbpt!]
    \centering
    \begin{tabular}{|c|c|c|c|c|c|c|}
    \hline 
    Source  &        (2,5)   &  (5,9)  & (3, 8)  & (4,7) & (0,6) & (2,3) \\
    \hline 
    -               & 0.896 & 0.531 &  0.575 & 0.854 & 0.900 & 0.850 \\
    Colored MNIST   & 0.989 & 0.636 &  0.860 & 0.857 & 0.985 & 0.933 \\
    Noisy MNIST     & 0.993 & 0.926 &  0.882 & 0.889 & 0.979 & 0.946 \\
    USPS            & 0.971 & 0.835 &  0.840 & 0.525 & 0.550 & 0.510 \\
    \hline 
    \end{tabular}
    \caption{Performance comparison of different source datasets on various digit pairs. Here the sign `-' represents the accuracy with only 100 labelled MNIST data without any source data, while the remaining three rows are the performance with 500 additional colored MNIST, noisy MNIST and USPS data, respectively.}\label{tab:RES}
\end{table}

As can be observed from the above table, we compared the model performance by accuracy between not using source data and using three different types of source data. In most cases, the introduction of source data showed an improvement over not using source data, validating the beneficial impact of source data on target performance enhancement. Moreover, when comparing different source data, the colored MNIST and noisy MNIST are closer to the original MNIST in terms of the conditional generating distribution $P(X|Y)$, and they do perform better than the USPS in almost all cases.  

We also plot the two main components for clusters 3 and 8 in Figure~\ref{fig:domain-adaptation} to visualize the constructed GMM model. We can infer from the figure that the GMM model trained without using source data yields the poorest performance, as it fails to distinguish between digits 3 and 8 accurately, and moreover, the prediction of digit 8 is noticeably biased, contradicting the testing label distributions. Upon the introduction of source data, the GMM model trained with the additional USPS dataset still exhibits a substantial overlap between 3 and 8 in the test set, implying a less optimal performance. On the other hand, with the colored MNIST dataset, the two clusters are more separated, representing the best prediction performance. 
\begin{figure*}[h]
\centering
\subfigure[MNIST]{\includegraphics[width = 2.20in]{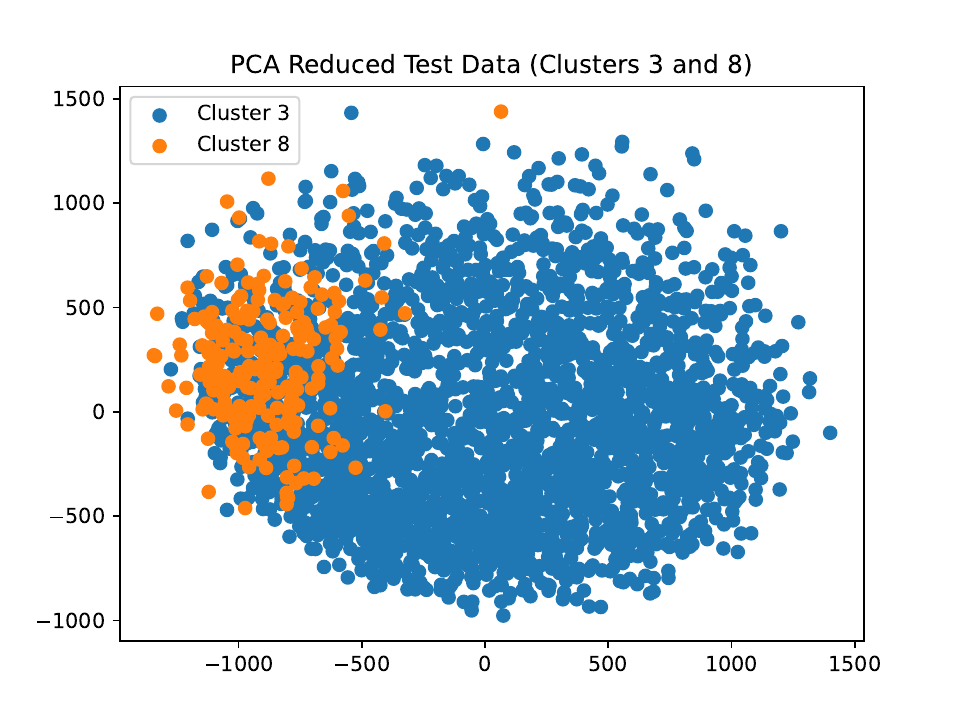}}
\subfigure[MNIST + USPS]{\includegraphics[width = 2.20in]{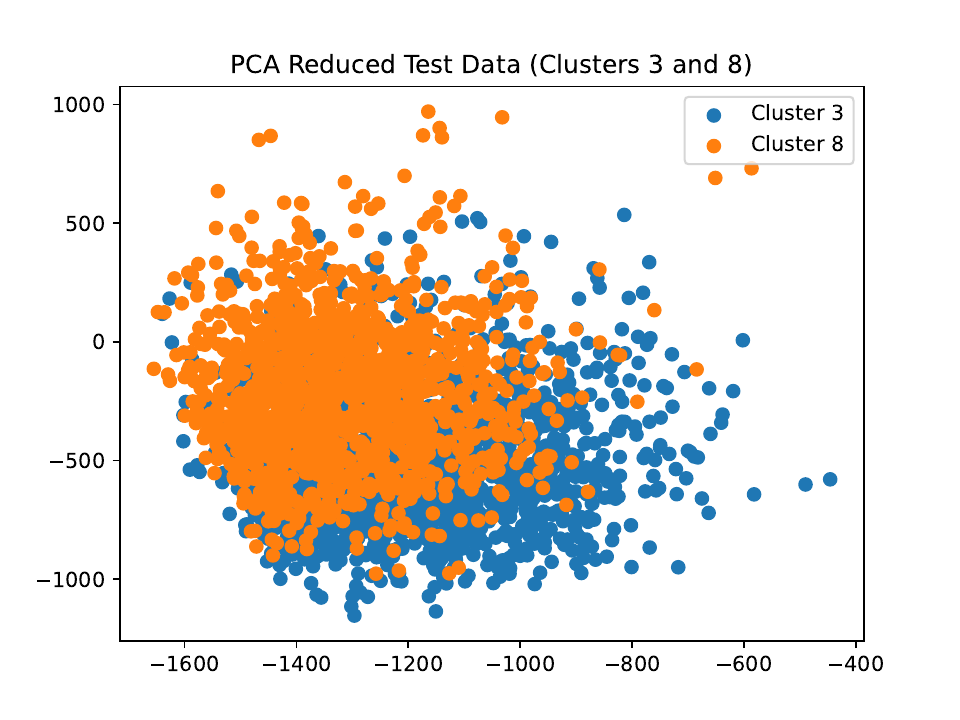}}
\subfigure[MNIST + Colored MNIST]{\includegraphics[width = 2.20in]{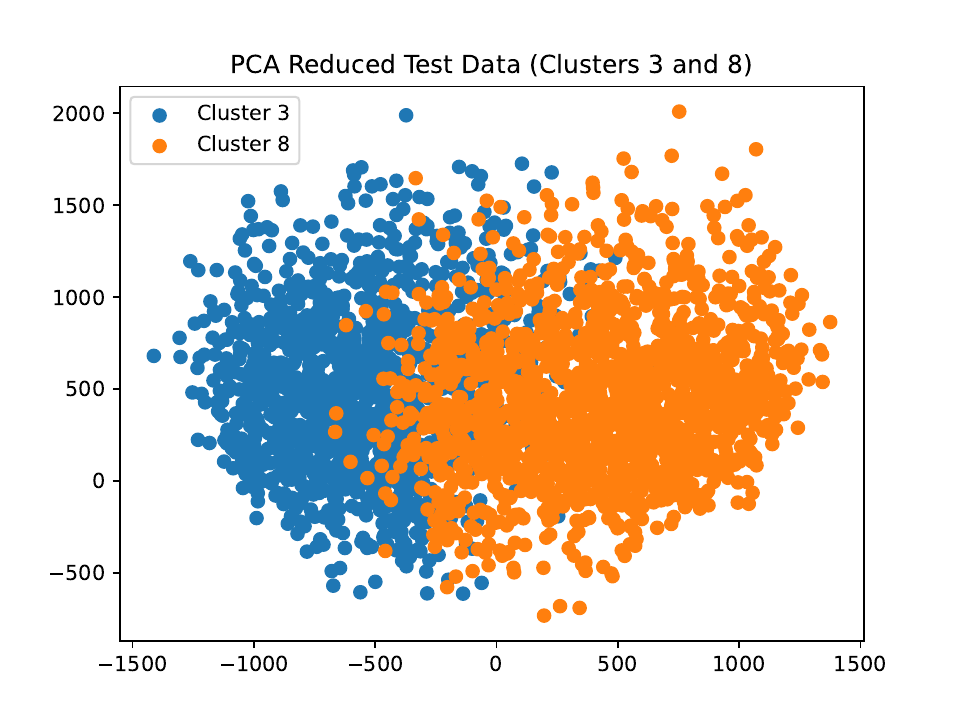}}
\subfigure[MNIST + Noisy MNIST]{\includegraphics[width = 2.20in]{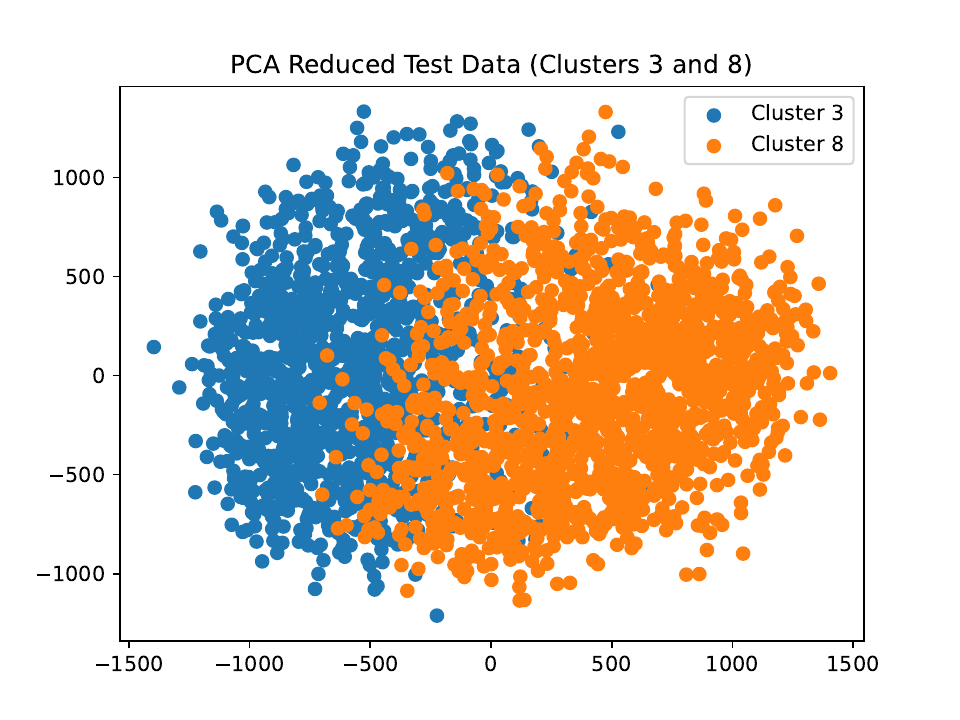}}
\caption{Visualization of clusters for various source and target combinations for digit pair (3,8)}\label{fig:domain-adaptation}
\end{figure*}

Referring to the table, we also noticed some cases where the use of USPS actually undermined the model accuracy for the digit pair $(4,7)$, $(0,6)$ and $(2,3)$. We point out that this does not contradict our earlier analytical results (source data should never degrade the performance). The reason is that the GMM models used to train the classifier are only approximations of the ``true" model, and importantly, the testing data is \textit{not} from these approximating parametric models but from the real dataset, whereas our analytical results hold under the assumption that both training and testing data are from parametric models. Nevertheless, we see that satisfactory results can still be achieved in many scenarios with this empirical setup, even when approximations are used, showing effective guidance of our theoretical results.

\subsubsection*{Multi-classification with SSL and UDA}

In the previous section, we provided a simplified comparison of SSL and UDA by focusing on results involving just two numerical categories. These experiments helped clearly demonstrate the data's practical value through 2D visual representations. In this section, we aim to assess the comprehensive performance across the dataset by applying our algorithm to data that includes all labels, e.g., the multi-classification of handwritten digits ranging from 0 to 9. For experiments, we randomly select 200 samples from the MNIST dataset for our initial labelled target dataset. Then, to explore the impact of additional training data, we gradually increase the number of these extra training samples from 400 to 5000. These additional samples are sourced from various datasets, including unlabelled MNIST samples or labelled samples from variants of the MNIST dataset (such as coloured MNIST and noisy MNIST) and the USPS dataset. Furthermore, we investigate how the number of PCA dimensions and the number of clusters in our model affect its performance. The results are organized across three tables. Table~\ref{tab:multi_sample_size_effect} details how varying the size of additional data samples impacts the model performance. Table~\ref{tab:multi_pca_effect} explores the influence of changing the dimensions within PCA. Lastly, Table~\ref{tab:multi_cluster_effect} examines the effects of altering the number of clusters in GMM. From the results, we identify some key insights as follows.
\begin{table}[h]
    \centering
    \begin{tabular}{|c|c|c|c|c|c|}
    \hline
    Sample sizes  & 400 & 800 & 1600 & 3200 & 5000\\
    \hline 
    -  &  \multicolumn{5}{c|}{0.397} \\
    \hline 
    Unlabelled MNIST &  0.364 & 0.545 & 0.606 & 0.623 & 0.636\\
    Colored MNIST & 0.399 & 0.400 & 0.481 & 0.455 & 0.531\\
    Noisy MNIST & 0.483 & 0.420 & 0.567 & 0.468 & 0.562\\
    USPS & 0.354 & 0.367 & 0.271 & 0.259 & 0.335\\
    \hline 
    \end{tabular}
    \caption{Effect of the sample size for additional training instances, where we set $N = 200$, $K = 10$ and PCA dimension to be 15. Here the sign `-' represents the accuracy with only 200 labelled MNIST data without any source data, while the rest four rows are the results with additional unlabelled MNIST, colored MNIST, noisy MNIST and USPS data, respectively (the same applies to tables below).}\label{tab:multi_sample_size_effect}
\end{table}

\begin{table}[h]
    \centering
    \begin{tabular}{|c|c|c|c|c|c|}
    \hline
    PCA dimension  & 5 & 15 & 25 & 35 & 45\\
    \hline 
    - &  0.470 & 0.397 & 0.195 & 0.485 & 0.372\\
    Unlabelled MNIST &  0.531 & 0.606 & 0.287 & 0.506 & 0.445\\
    Colored MNIST & 0.512 & 0.481 & 0.113 & 0.456 & 0.353\\
    Noisy MNIST & 0.538 & 0.567 & 0.137 & 0.441 & 0.461\\
    USPS & 0.264 & 0.271 & 0.139 & 0.132 & 0.094\\
    \hline 
    \end{tabular}
    \caption{Effect of the cluster number where we set $N = 200$, $M = L = 1600$ and the cluster number to be 10}\label{tab:multi_pca_effect}
\end{table}

\begin{table}[h]
    \centering
    \begin{tabular}{|c|c|c|c|c|c|}
    \hline
    Cluster number  & 10 & 15 & 20 & 25 & 30\\
    \hline 
    - &  0.590 & 0.616 & 0.609 & 0.563 & 0.491\\
    Unlabelled MNIST &  0.680 & 0.669 & 0.696 & 0.614 & 0.468\\
    Colored MNIST & 0.570 & 0.578 & 0.458 & 0.423 & 0.335\\
    Noisy MNIST & 0.693 & 0.695 & 0.727 & 0.649 & 0.592\\
    USPS & 0.458 & 0.487 & 0.479 & 0.338 & 0.223\\
    \hline 
    \end{tabular}
    \caption{Effect of the cluster number where we set $N = 400$, $M = L = 1600$ and PCA dimension to be 15}\label{tab:multi_cluster_effect}
\end{table}
\begin{itemize}
    \item \textbf{Additional Training Samples:} Including extra unlabelled MNIST samples steadily improves the model's performance, showing the value of unlabelled data in SSL. Nonetheless, the effect of augmenting the dataset with colored or noisy MNIST samples varies, indicating that while adding more training data from similar distributions can be advantageous, the presence of distribution shifts or noise might occasionally degrade the performance. The decrease in performance with USPS samples highlights the difficulty in adapting the model to different data distributions, also previously observed in Table~\ref{tab:RES} where the testing data distribution deviates from these approximating parametric models, emphasizing that in practice, the data might be instead useless if the generating distribution varies too much in the anti-causal direction. 
    \item \textbf{PCA Dimensions:} The link between the number of dimensions in PCA and how well a model performs is complex, showing that there is not a clear connection between adding more dimensions and achieving better performance. The best number of PCA dimensions changes depending on the dataset, suggesting the importance of a customized strategy for reducing dimensions that focuses on preserving key features while eliminating the effect of other factors, such as noise. This concept is especially clear when looking at the decline in performance across all dimension levels with USPS data, demonstrating the difficulties in applying a one-size-fits-all approach to different datasets.
    \item \textbf{Cluster Number:} The effectiveness of the model changes as the number of clusters changes. There is performance improvement up to a certain cluster number for particular datasets, and then it starts to decrease as the cluster increases. This indicates that there is an ideal number of clusters that can enhance the model's performance, a trend that is particularly noticeable with unlabelled and noisy MNIST datasets. On the other hand, for colored MNIST and USPS datasets, the performance tends to worsen as the number of clusters increases. This could be caused by over-segmentation or the loss of important features due to too many clusters.
\end{itemize}
These experiments examine the impact of different factors, such as additional data sample size, the number of PCA dimensions, and the number of clusters on the performance of models across various datasets for anti-causal learning. In the anti-causal learning setup, more unlabelled data without the distribution shift generally boosts the model performance, but adding labelled source data (such as the refactored MNIST datasets and USPS in our example) does not always lead to better results, pointing to the importance of causal direction and data generating mechanisms. The optimal number of PCA dimensions and clusters is not one-size-fits-all but needs customization for each dataset to ensure key parameters are retained while minimizing noises from the redundant features. For some datasets like unlabelled and noisy MNIST, a specific cluster number can improve performance, whereas for others, like colored MNIST and USPS, it may cause problems due that the testing data may not be drawn from these approximating GMM distributions and possibly over-segmentation with large cluster numbers or the loss of important features with small PCA dimensions.
\section{Extensions to Unknown Causal Settings} \label{Sec:modelselection}
Even though in this work we primarily focus on the setup where the setting is known to be either causal learning or anti-causal learning, it is also interesting to consider the scenario where the underlying relationship between $X$ and $Y$ is acyclic but \textbf{unknown}. We ask the question, which causal direction should we use for prediction? Our strategy is that given the statistics from the observed data $(X, Y)$, we try to fit the data with both causal-learning and anti-causal learning settings and decide which setting will enable us to make predictions more efficiently. Notice that it could be the case that the chosen setting is not the true underlying mechanism (and perhaps not physically possible). However, this is irrelevant as far as the prediction is concerned, as we only work with observed data and will not intervene in the system.  By the same argument, we could choose either setting for the prediction \textit{even if the true causal setting is known}. So it is tempting to carry out this comparison even if we know the true direction. However, it does not seem to be fruitful in general. Indeed, as pointed out by \citet{kocaoglu2017entropic} and \citet{compton2021entropic}, if we want to use an anti-causal learning setting to fit the data generated from a causal learning setting (or vice versa), this ``artificial" fitting is in general much more complicated than fitting from the true underlying setting, which would make the prediction more difficult. 

If the causal relationship between $X$ and $Y$ for a certain learning problem is unknown and we can model the data from both directions, our results imply that we should use whichever model achieves a better learning performance. This can be viewed as a causal model selection problem. Referring to Table~\ref{tab:result}, for semi-supervised learning, the rate from the causal direction will be $O(\frac{k}{m})$ while $O(\frac{k'+1}{m+n})$ for anti-causal learning if we have abundant source data ($n \ll m$) and $k< k'+1$, fitting from the causal direction will be easier. In contrast, if we have abundant target data ($m \ll n$), then fitting from the anti-causal direction will be more favourable. Using similar arguments in the domain adaptation scenarios, if the covariate shift assumption does not hold, the source data will be unhelpful from the causal direction, and we should always fit from the anti-causal direction. Otherwise, the model selection is, again, determined by the sample sizes $m$ and $n$.

In an attempt to investigate the model selection issue, we examine the excess risk from numerical analysis for the aforementioned parametric models under the semi-supervised learning condition for the sake of simplicity. We will consider the distribution $P_S(X,Y) = P_T(X,Y)$ from the anti-causal direction as:
\begin{align}
   Y &\sim \textup{Ber}(0.5), \nonumber  \\
   X_{0} &\sim \textup{Cat}(0.05, 0.6, 0.25, 0.1), \nonumber \\
   X_{1} &\sim \textup{Cat}(0.05, 0.3, 0.25, 0.4). \nonumber 
\end{align}
by setting $\theta_Y = 0.5$, $\theta_{{0}} = 0.05$ and $\theta_{{1}} = 0.05$. We can also model the same joint distribution from the causal directions by choosing the parameters as follows:
\begin{align}
   X  &\sim \textup{Cat}(0.05, 0.45, 0.25, 0.25),  \nonumber \\
   Y_{x_1}  &\sim \textup{Ber}(0.5), \quad Y_{x_2}  \sim \textup{Ber}(\frac{1}{3}),  \nonumber \\
   Y_{x_3}  &\sim \textup{Ber}(0.5), \quad Y_{x_3}  \sim \textup{Ber}(0.8). \nonumber
\end{align}

\begin{figure}[htb]
   \centering
   \subfigure[$k = 4, k' = 2$\label{subfigure:a}]{\includegraphics[width = 2.7in]{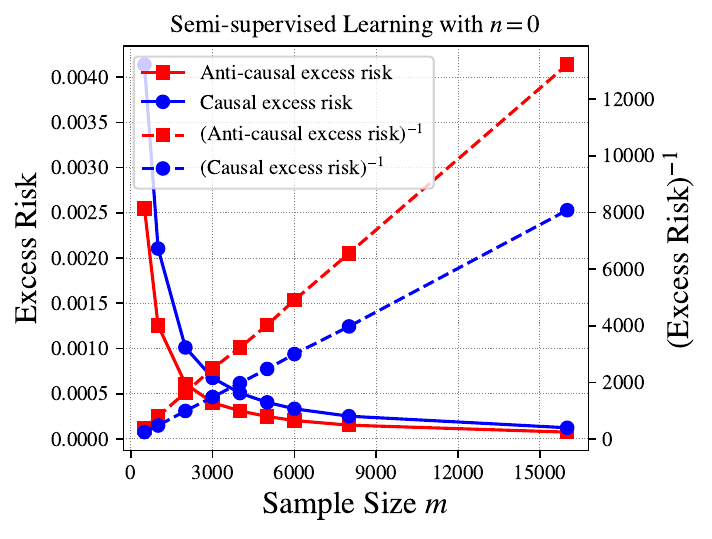}}
   \subfigure[$k = 1, k' = 2$\label{subfigure:b}]{\includegraphics[width = 2.7in]{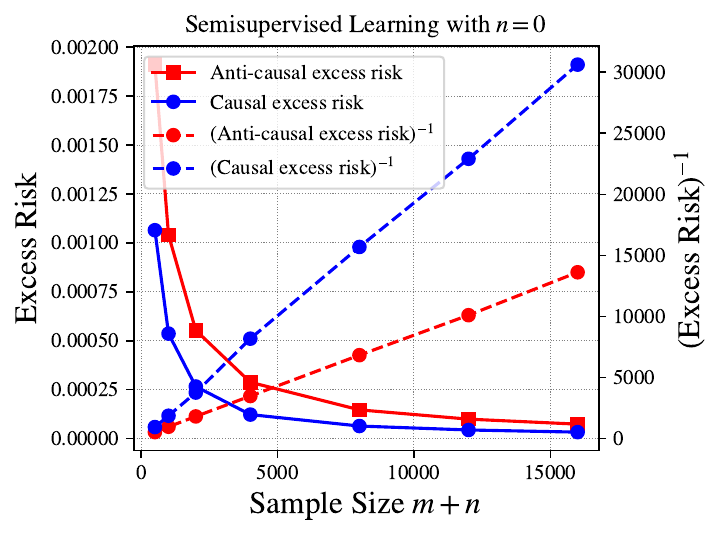}}
   \caption{Excess risk comparisons fitting from causal (red) and anti-causal (blue) under semi-supervised learning with labelled data only. Results are derived by the same parameterization from anti-causal learning with $k' = 2$ but different parameterization from causal learning with $k = 4$ in (a) and $k = 1$ in (b).} \label{fig:anti-comparison}
\end{figure}

By varying the sample size $m$ from 500 to 16000, we plot the excess risk $\mathcal{R}(b)$ under causal and anti-causal learning settings in Figure~\ref{subfigure:a}. It is observed that both directions produce the same rate of $O(\frac{1}{m})$. Compared to the causal case, fitting from the anti-causal direction enjoys a lower regret, and the slope of its reciprocal is higher, which implies that it is ``easier" to learn the distribution $P(Y|X)$ from the anti-causal direction. Roughly speaking, the reason is that learning $\theta^*_{Y_X}$ requires $k = 4$ parameters, but the inference from the anti-causal direction only requires $k'+1 = 3$ parameters, which decreases the model uncertainty and hence the better performance. Rigorously speaking, the slope (or scaling factor in the rate) depends on the information dimension (see \citet{haussler1995general} for reference). For example, under causal learning, the convergence rate is proved to be $\frac{k}{2m}$ and the slope will be $\frac{2}{k}$ where $k = 4$ is the number of parameters for $\theta^*_{Y_X}$ in this case. It is also confirmed from the figure that the slope is roughly $\frac{1}{2}$. The same argument applies in the anti-causal learning, and the convergence rate is $\frac{k'+1}{2m}$ when $m$ is sufficiently large, leading to a lower regret since $k > k'+1$. With such parameterization, it is always better to fit from the anti-causal direction.

However, if we model the distribution from the causal directions by  setting $\theta_Y = 0.5$ in the following way:
\begin{align}
   X  &\sim \textup{Cat}(0.05 , 0.45  , 0.25, 0.25), \nonumber \\
   Y_{x_1}  &\sim \textup{Ber}(\theta_{Y}), \quad Y_{x_2}  \sim \textup{Ber}(\theta_{Y}-\frac{1}{6}), \label{eq:causal-para2} \\
   Y_{x_3}  &\sim \textup{Ber}(\theta_{Y}), \quad Y_{x_4}  \sim \textup{Ber}(\theta_{Y}+0.3), \nonumber 
\end{align}
With such a restriction, the number of parameters $k$ is reduced to 1. We successively repeat the experiment and plot the result in Figure~\ref{subfigure:b}. The excess risk, in this case, becomes lower than fitting from the anti-causal direction and the rate is improved to approximately $\frac{1}{2m}$. The results indicate that the model selection depends on how we parameterize the data distributions, particularly the number of parameters from each causal direction.

In the above example, we only consider the labelled data. The unlabelled samples, however, are not useful for the causal direction but will take effect from the anti-causal direction from Figure~\ref{fig:causal_unlabel} and \ref{fig:anti_causal_unlabel}. Both causal learning and anti-causal learning can be more favorable than the other option, depending on the sample sizes. For instance, if we have abundant unlabelled data and limited labelled data, referring to Table~\ref{tab:result}, fitting from anti-causal direction yields the rate $O(\frac{k'+1}{m+n})$, which is better than the rate $O(\frac{k}{m})$ under causal direction if $n \gg m$. 

To numerically illustrate, we conduct the experiments with the parameterization in (\ref{eq:para-anti-s}) from anti-causal direction and (\ref{eq:causal-para2}) from causal direction under semi-supervised learning with both labelled and unlabelled data. We then plot the results in Figure~\ref{fig:model_selection}. 

\begin{figure*}[htb]
\centering
\subfigure[Fix $n$, vary $m$\label{fig:sub_a}]{\includegraphics[width = 2in]{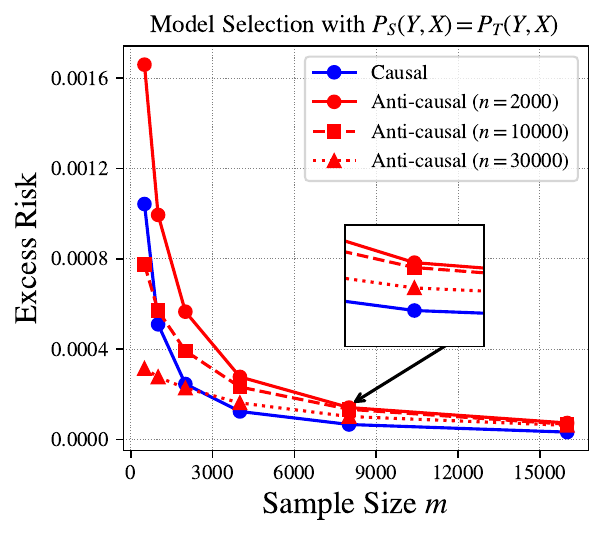}}
\subfigure[Fix $n$, vary $m$\label{fig:sub_b}]{\includegraphics[width = 2in]{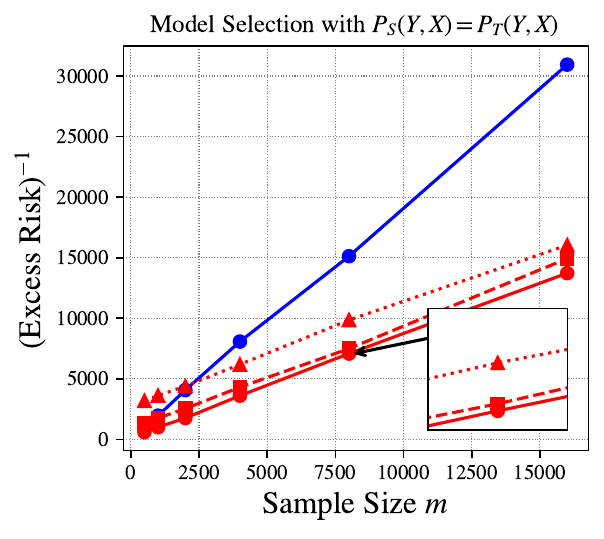}}
\subfigure[Fix $m$, vary $n$\label{fig:sub_c}]{\includegraphics[width = 2in]{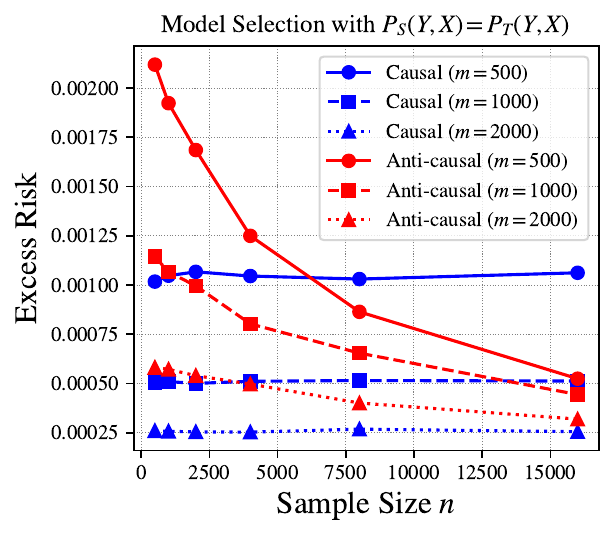}}
\subfigure[Fix $m$, vary $n$\label{fig:sub_d}]{\includegraphics[width = 2in]{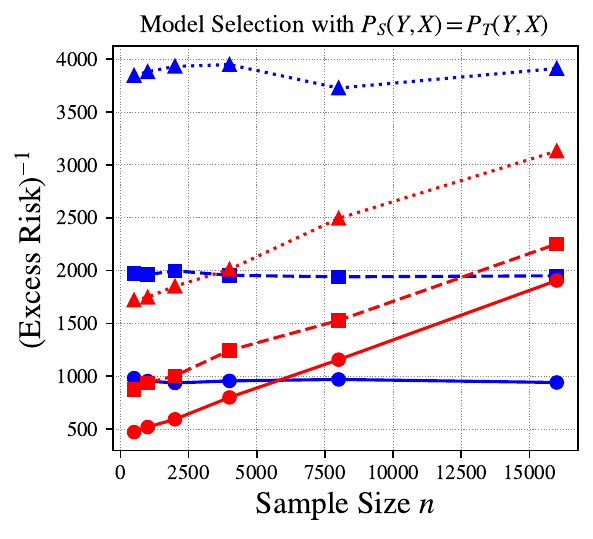}}
\caption{The figure shows the excess risk comparisons by varying $m$ and $n$ under causal and anti-causal learning for semi-supervised learning. The subfigure \ref{fig:sub_a} shows the results of $\mathcal{R}(b)$ and \ref{fig:sub_b} of $\mathcal{R}(b)^{-1}$ (sharing the same legend) by varying $m$ from 500 to 16000 and fixing $n = 2000, 10000$ and $30000$, respectively. In \ref{fig:sub_a}, the blue curve shows the result by fitting from the causal direction and from top to bottom, the other three red curves are derived by fitting from the anti-causal direction with $n = 500,1000$ and $2000$. The subfigure \ref{fig:sub_c} shows the results of $\mathcal{R}(b)$ and \ref{fig:sub_d} of $\mathcal{R}(b)^{-1}$ (sharing the same legend) by varying $n$ from 500 to 16000 and fixing $m = 1000, 2000$ and $3000$, respectively. In \ref{fig:sub_c}, from top to bottom, three blue curves correspond to $m = 500, 1000$ and $2000$ by the causal direction and the three red curves by the anti-causal direction. All results are derived by 3000 experimental repeats. }\label{fig:model_selection}
\end{figure*}

We firstly vary $m$ from 500 to 16000 by fixing $n = 2000, 10000$ and $30000$ to show the effectiveness of labelled data. We plot the corresponding results of $\mathcal{R}(b)$ in subfigure \ref{fig:sub_a} and $\mathcal{R}(b)^{-1}$ in \ref{fig:sub_b}. In \ref{fig:sub_a}, we only plot one curve in blue since $n$ does not affect the excess risk from the causal direction. The remaining three red curves are derived by fitting from the anti-causal direction with an increasing $n$, from top to bottom. One can observe that a larger $n$ will incur a smaller initial excess risk when $m = 500$. However, the convergence rates are identical for all three cases. Since the slope of $\mathcal{R}(b)^{-1}$ is higher from the causal direction, when $m$ is large enough ($m > 2000$), even with large unlabelled data ($n = 30000$), the excess risk is still higher fitting from the anti-causal direction.    

The subfigure~\ref{fig:sub_c} shows the results of $\mathcal{R}(b)$ and \ref{fig:sub_d} of $\mathcal{R}(b)^{-1}$ by varying $n$ from 500 to 16000 and fixing $m = 500, 1000$ and $2000$. In \ref{fig:sub_c}, from top to bottom, three blue curves correspond to $m = 500, 1000$ and $2000$ by the causal direction and the three red curves by the anti-causal direction. In this case, the excess risk from the causal direction is almost a constant depending on $m$, regardless of the unlabelled sample size $n$. Furthermore, a higher $m$ incurs a lower regret. On the contrary, from the anti-causal learning direction, the excess risk will converge as $n$ goes sufficiently large. Selecting an appropriate model strongly hinges on the unlabelled target sample size $n$.  For example, in our formulation, when $m =500$, we may need more than $6500$ extra unlabelled samples to achieve a lower regret, and if $m$ doubles, we will need to double the required unlabelled samples to achieve a comparable expected risk.

Overall, for a general domain adaptation task without knowing the underlying causal mechanism, if we can model the data with parameterised distributions for both causal and anti-causal directions without some physical constraints, both models can be more favourable than the other option depending on how we do the parameterization, how many data samples we have and how different the source and target domains are.

\section{Conclusions}

This paper proposes a probabilistic framework articulating the connection between SSL/UDA and causal mechanisms. We explicitly characterize the rate of learning performance under different causal mechanisms and domain shift conditions, from which the usefulness of the source and target data is manifested. However, in our analysis, the parametric characterization of both source and target data is crucial. A possible future direction is to relax the assumptions on parametric conditions to general probability distributions and find the excess risk in terms of the sample sizes.  Our analysis also heavily relies on the generating processes we skectch in Figure~\ref{fig:causal-anti-dag} (e.g., $X$ and $Y$ are unconfounded), and the possible future work could be performing a similar analysis for the case with more than two variables (e.g., causal setting with con-founders), which improves the generality and applicability in real-world problems. We have also observed that incorporating unlabelled data and labelled source data could significantly enhance the model performance for the target domain on both synthetic data and real benchmarks. Due to the discrepancy between the approximated parametric distribution and the underlying data distribution for real-world scenarios, our theoretical analysis cannot directly carry over. In addition, developing a method that can effectively handle non-parametric distributions is also a potential direction worth exploring.
\newpage

\appendix
\section{Appendix: Proofs}

\subsection{Mixture Asymptotics Lemma} \label{apd:mix_lemma}

\begin{lemma}[Mixture Asymptotics]\label{lemma:asymptotics}
Under Assumption~\ref{asp:para-dist},\ref{asp:para-trans},\ref{asp:proper-prior} and assume $m \asymp n^{p}$ for some $p > 0$ and let $n \rightarrow \infty$, then the mixture strategy yields
\begin{align}
    D(P_{\theta^*}(D^{m}_s, D^{U,n}_t)\|Q(D^{m}_s, D^{U,n}_t)) = \frac{d}{2}\log\frac{1}{2\pi e} + \log\frac{1}{\omega(\theta^*)} + \frac{1}{2}\log\operatorname{det}(I_{st}) + o(\frac{1}{n\vee m}),
\end{align}
where $\theta^* \in \mathbb{R}^{d}$ denotes the total parameters that characterize the source and target distributions and $d$ denotes the total dimension, depending on the causal directions and distribution shifting conditions. The Fisher information matrix associated with $D^{m}_s$ and $D^{U,n}_t$ is defined as $I_{st} = -\mathbb{E}_{\theta^*}[\nabla^2 \log P(D^{m}_s, D^{U,n}_t|\theta^*)]$.
\end{lemma}

\begin{proof} The proof and result is a generalization of \citet{clarke1990information,zhu2020semi} with some modifications to fit our purpose. Without the loss of generality, we first assume that the source parameter and target parameter will have $\tilde{k}+1-c$ domain-specific parameters and $c$ domain-sharing parameters, where $\tilde{k} = k$ for causal learning and $\tilde{k} = k'$ for anti-causal learning. $c$ will vary under different shift conditions. For example, under the target shift condition in anti-causal learning, $c$ will be $k'$ for identical parameters $\theta_{X_{y_i}}$ in both domains; In conditional shift condition, $c = 1$ since $\theta^{s*}_{Y} = \theta^{t*}_{Y}$. With a little abuse of notation in this section, we denote the true source-specific parameters by $\theta^*_{s} \in \mathbb{R}^{\tilde{k}+1-c}$, the target-specific parameters by $\theta^*_t \in \mathbb{R}^{\tilde{k}+1-c}$ and the domain-sharing parameters as $\theta^*_c \in \mathbb{R}^{c}$. Then the source data $(X,Y)$ is drawn from the distribution $P_{\theta^*_c, \theta^*_s}$ and the target data $X$ is drawn from the distribution $P_{\theta^*_c, \theta^*_t}$ under such parameterization. For simplicity, we can write the joint domain parameters $\mathbf{\theta^*} = (\theta^*_c, \theta^*_s, \theta^*_t)$ and the joint distribution for the source domain data and target domain data is expressed by
\begin{align}
    P_{\mathbf{\theta^*}}(D^m_s, D^{U,n}_t) = P_{\theta^*_c, \theta^*_s}(D^m_s) P_{\theta^*_c, \theta^*_t}(D^{U,n}_t) = \prod_{i=1}^{m}P_{\theta^*_c, \theta^*_s}(D^m_s) \prod_{j=1}^{n}P_{\theta^*_c, \theta^*_t}(D^{U,n}_t).
\end{align}
Based on the notations above, we define the score functions by
\begin{align}
    l_s(\theta_s, \theta_c) &= \nabla \log P(D^m_s|\theta_s, \theta_c), \\
    l_t(\theta_t, \theta_c) &= \nabla  \log P(D^{U,n}_t|\theta_t, \theta_c),\\
    l_{st}(\theta) &= \nabla \log P(D^{U,n}_t, D^m_s|\theta).
\end{align}
Note that
\begin{equation}
    l_{st}(\theta^*) = \begin{bmatrix}
l_s(\theta^*_s, \theta^*_c) \\
\mathbf{0}_{\tilde{k}+1 - c} 
\end{bmatrix} + \begin{bmatrix}
\mathbf{0}_{\tilde{k}+1 - c} \\
l_t(\theta^*_t, \theta^*_c)
\end{bmatrix},
\end{equation}
where $\mathbf{0}_{\tilde{k}+1 - c}$ denotes the zero vector with length $\tilde{k}+1-c$. We next restate the corresponding Fisher information matrix,
\begin{align}
    I_s &= -\mathbb{E}_{\theta^*_s, \theta^*_c}[ \nabla^2 \log P(X_s,Y_s|\theta^*_s, \theta^*_c)] \in \mathbb{R}^{(\tilde{k}+1)\times (\tilde{k}+1)}, \\
    I_t &= -\mathbb{E}_{\theta^*_t, \theta^*_c}[ \nabla^2 \log P(X_t|\theta_t, \theta_c)] \in \mathbb{R}^{(\tilde{k}+1)\times (\tilde{k}+1)} \\
    I_{0} &= -\mathbb{E}_{\theta^*}[\nabla^2 \log P(X_s,Y_s, X_t|\theta^*)] \in \mathbb{R}^{(2\tilde{k}+2-c)\times (2\tilde{k}+2-c)},\\
    I_{st} &= -\mathbb{E}_{\theta^*}[\nabla^2 \log P(D^{U,n}_t, D^m_s|\theta^*)] \in \mathbb{R}^{(2\tilde{k}+2-c)\times (2\tilde{k}+2-c)}.
\end{align}
Their corresponding \emph{empirical} versions are denoted by,
\begin{align}
    \tilde I_s(\theta_s, \theta_c) &= -[ \nabla^2 \log P(X_s,Y_s|\theta_s, \theta_c)] \in \mathbb{R}^{(\tilde{k}+1)\times (\tilde{k}+1)}, \\
    \tilde I_t(\theta_t, \theta_c) &= -[ \nabla^2 \log P(X_t|\theta_t, \theta_c)] \in \mathbb{R}^{(\tilde{k}+1)\times (\tilde{k}+1)}, \\
    \tilde I_{0}(\theta) &= -[\nabla^2 \log P(X_s,Y_s, X_t|\theta^*)] \in \mathbb{R}^{(2\tilde{k}+2-c)\times (2\tilde{k}+2-c)},\\
    \tilde I_{st}(\mathbf{\theta}) &= -[\nabla^2 \log P(D^{U,n}_t, D^m_s|\theta)] \in \mathbb{R}^{(2\tilde{k}+2-c)\times (2\tilde{k}+2-c)}.
\end{align}
For convenience, if not otherwise stated we will simply omit brackets for $\theta^*$ in the sequel, e.g., we write $\tilde I_{st}(\mathbf{\theta^*})$ as $\tilde I_{st}$. Define the neighbourhood of $\theta^*$ by $N_{\delta} = \{\theta: \|\theta-\theta^*\| \leq \delta \}$  where the norm in $\mathbb{R}^{2\tilde{k}+2-c}$ is defined as
\begin{align}
    \|\xi \|^2 =  \xi^TI_{0}\xi.
\end{align}
Define 
\begin{equation}
    L(\theta^*) = l^T_{st}(\theta^*) I^{-1}_{st} l_{st}(\theta^*).
\end{equation}
Note that,
\begin{align}
    \mathbb{E}[L(\theta^*)] &= \mathbb{E}[\operatorname{Tr}(I^{-1}_{st}l_{st}(\theta^*)^T l_{st}(\theta^*))] \nonumber \\
    &= \operatorname{Tr}(I^{-1}_{st}\mathbb{E}[l_{st}(\theta^*)^T l_{st}(\theta^*)]) \nonumber \\
    &= \operatorname{Tr}(I^{-1}_{st}I_{st})  \nonumber \\
    &= 2\tilde{k}+2-c.
\end{align}
For $ 0 < \epsilon < 1$ and $\delta > 0$, we define three events $A(\delta, \epsilon)$, $B(\delta, \epsilon)$ and $C(\delta)$ as
\begin{align}
    A(\delta, \epsilon) &= \left\{ \int_{N_{\delta}^{c}} P\left(D^s_m, D^{U,n}_t \mid \theta \right) \omega (\theta) d \theta  \leq \epsilon \int_{N_{\delta}} P\left(D^s_m, D^{U,n}_t \mid \theta \right) \omega(\theta) d \theta \right\}, \\
    B(\delta, \epsilon)&:=\{(1-\epsilon)\left(\theta-\theta^*\right)^{T} I_{st}\left(\theta-\theta^*\right)  \leq\left(\theta-\theta^*\right)^{T}\left(\tilde I_{st}(\theta')\right)\left(\theta-\theta^*\right), \\
    &\quad \quad  \leq(1+\epsilon)\left(\theta-\theta^*\right)^{T} I_{st}\left(\theta-\theta^*\right) \quad \text { for all } \left.\theta, \theta^{\prime} \in N_{\delta}\right\} \\  
    C(\delta) &:=\left\{L\left(\theta^*\right) \leq \min \{n, m\} \delta^{2}\right\},
\end{align}
and 
\begin{equation}
    \rho(\delta, \theta^*) = \sup_{\theta\in N_{\delta}} \left|\frac{\omega(\theta)}{\omega(\theta^*)}\right|.
\end{equation}
Following the similar procedures in \citet{clarke1990information}, we have the following upper and lower bounds on the density ratio.
\begin{lemma}
We assume condition 3 in Assumption~\ref{asp:para-trans} holds that $P_\theta^*$ is twice differentiable around $\theta^*$ and $I_{st}$ is positive definite. With proper prior $\omega(\theta)$, then on the set of $A\cap B$, we have,
\begin{align}
    \frac{Q(D^s_m,D^{t,U}_n)}{P_{\theta^*}(D^s_m,D^{t,U}_n)} \leq (1+\epsilon)\omega(\theta^*)e^{\rho(\delta,\theta^*)}(2\pi)^{\frac{2\tilde{k}+2-c}{2}}e^{\frac{1}{2(1-\epsilon)}L(\theta^*)}\operatorname{det}((1-\epsilon)I_{st})^{-\frac{1}{2}}.
\end{align}
Further, on the set of $B \cap C$, we have the lower bound,
\begin{align}
    \frac{Q(D^s_m,D^{t,U}_n)}{P_{\theta^*}(D^s_m,D^{t,U}_n)} \geq \omega(\theta^*)e^{-\rho(\delta,\theta^*)}(2\pi)^{\frac{2\tilde{k}+2-c}{2}}e^{\frac{1}{2(1+\epsilon)}L(\theta^*)}(1-2^{\frac{2\tilde{k}+2-c}{2}}e^{-\epsilon^2(n\wedge m)\delta^2/8 })\operatorname{det}((1+\epsilon)I_{st})^{-\frac{1}{2}}.
\end{align}
\end{lemma}
\begin{proof}
In both cases, we will use the Laplace method to give an upper and lower bound on the density ratio, for the upper bound, if we restrict on $A$ and $B$, then,
\begin{align*}
    \frac{Q(D^s_m,D^{t,U}_n)}{P_{\theta^*}(D^s_m,D^{t,U}_n)} &\leq (1+\epsilon) \int_{N_{\delta}} \frac{P_{\theta}(D^s_m,D^{t,U}_n)}{P_{\theta^*}(D^s_m,D^{t,U}_n)} \omega(\theta) d\theta \\
     &= (1+\epsilon) \int_{N_{\delta}} e^{\log \frac{P_{\theta}(D^s_m,D^{t,U}_n)}{P_{\theta^*}(D^s_m,D^{t,U}_n)}} \omega(\theta) d\theta \\
     \text{ (Taylor Expansion) } & =  (1+\epsilon) \int_{N_{\delta}} e^{(\theta-\theta^*)^Tl_{st}(\theta^*) - \frac{1}{2}(\theta-\theta^*)^T\tilde{I}_{st}(\theta')(\theta-\theta^*)} \omega(\theta) d\theta  \\
   \text{ (Definition of $\rho$) }  &\leq (1+\epsilon) \omega(\theta^*) e^{\rho(\delta,\theta^*)}\int_{N_{\delta}} e^{(\theta-\theta^*)^Tl_{st}(\theta^*) - \frac{1}{2}(\theta-\theta^*)^T\tilde{I}_{st}(\theta')(\theta-\theta^*)} d\theta \\
   \text{ (Event $B$) }   &\leq (1+\epsilon) \omega(\theta^*) e^{\rho(\delta,\theta^*)}\int_{N_{\delta}} e^{(\theta-\theta^*)^Tl_{st}(\theta^*) - \frac{1}{2}(1-\epsilon)(\theta-\theta^*)^TI_{st}(\theta-\theta^*)} d\theta \\
   &\overset{(*)}{=}  (1+\epsilon) \omega(\theta^*) e^{\rho(\delta,\theta^*)} e^{\frac{1}{2(1-\epsilon)}l^T_{st}(\theta^*)I^{-1}_{st}l_{st}(\theta^*)}\int_{N_{\delta}} e^{-\frac{1}{2}(1-\epsilon)(\theta-u)^TI_{st}(\theta - u)} d\theta\\ 
   &\leq (1+\epsilon) \omega(\theta^*) e^{\rho(\delta,\theta^*)} e^{\frac{1}{2(1-\epsilon)}l^T_{st}(\theta^*)I^{-1}_{st}l_{st}(\theta^*)}\int_{N_{\delta} \cup N^c_{\delta} } e^{-\frac{1}{2}(1-\epsilon)(\theta-u)^TI_{st}(\theta - u)} d\theta \\
   \text{ (Gaussian integral) } &= (1+\epsilon)\omega(\theta^*)e^{\rho(\delta,\theta^*)}e^{\frac{1}{2(1-\epsilon)}L(\theta^*)} (2\pi)^{\frac{2\tilde{k}+2-c}{2}} \operatorname{det}((1-\epsilon)I_{st})^{-\frac{1}{2}},
\end{align*}
where we define $u = \theta^*+\frac{1}{1-\epsilon}(\hat{\theta} - \theta^*)$ and $\hat{\theta} = \theta^* + I^{-1}_{st}l_{st}(\theta^*)$ provided that $I_{st}$ positive definite. We also use the identity in (*) by completing the square,
\begin{align}
    (\theta - \theta^*)^Tl_{st}(\theta^*) - \frac{1}{2}(1-\epsilon)  (\theta - \theta^*)^T I_{st} (\theta - \theta^*) = -\frac{1-\epsilon}{2}(\theta - u)^TI_{st}(\theta - u) + \frac{1}{2(1-\epsilon)} L(\theta^*).
\end{align}
For the lower bound, we have,
\begin{align*}
    \frac{Q(D^s_m,D^{t,U}_n)}{P_{\theta^*}(D^s_m,D^{t,U}_n)} &\geq  \int_{N_{\delta}} \frac{P_{\theta}(D^s_m,D^{t,U}_n)}{P_{\theta^*}(D^s_m,D^{t,U}_n)} \omega(\theta) d\theta \\
    &= \int_{N_{\delta}} e^{\log \frac{P_{\theta}(D^s_m,D^{t,U}_n)}{P_{\theta^*}(D^s_m,D^{t,U}_n)}} \omega(\theta) d\theta  \\
  \text{ (Taylor Expansion) } &= \int_{N_{\delta}} e^{(\theta - \theta^*)^Tl_{st}(\theta^*) - \frac{1}{2}(\theta - \theta^*)^T\tilde{I}_{st}(\theta')(\theta - \theta^*)}\omega(\theta) d\theta \\
  \text{(Event $B$)} & \geq \omega(\theta^*) e^{-\rho(\delta, \theta^*)}\int_{N_{\delta}} e^{(\theta - \theta^*)^Tl_{st}(\theta^*) - \frac{1}{2}(1+\epsilon)(\theta - \theta^*)^TI_{st}(\theta - \theta^*)} d\theta \\
  &= \omega(\theta^*)^{-\rho(\delta, \theta^*)} e^{\frac{1}{2(1+\epsilon)}L(\theta^*)} \int_{N_{\delta}} e^{-\frac{(1+\epsilon)}{2}(\theta - u)^TI_{st}(\theta - u)} d\theta \\ 
  &= \omega(\theta^*)^{-\rho(\delta, \theta^*)}  e^{\frac{1}{2(1+\epsilon)}L(\theta^*)} \Bigg[ \int_{\mathbb{R}^{2\tilde{k}+2-c}} e^{-\frac{(1+\epsilon)}{2}(\theta - u)^TI_{st}(\theta - u)} d\theta \\
  &\quad - \int_{N^c_{\delta}} e^{-\frac{(1+\epsilon)}{2}(\theta - u)^TI_{st}(\theta - u)} d\theta \Bigg]. 
 \end{align*}
Here we define $u = \theta^* + \frac{1}{1+\epsilon}(\hat{\theta} - \theta^*)$ and $\hat{\theta} = \theta^* + I^{-1}_{st}l_{st}(\theta^*) $. Since we restrict to the event $C$ and the norm is w.r.t. $I_{0}$, given Condition 2 such that $I_{st} \succcurlyeq (n \wedge m) I_0$, we have that for any $\theta \in N^c_{\delta}$, 
\begin{align}
   (\theta - u)^TI_{st}(\theta - u)  & \geq  (n \wedge m) (\theta - u)^TI_{0}(\theta - u)   \\
   \text{ (Definition of $\|\cdot\|$) }&=  (n \wedge m) \| \theta - u\|^2 \\
    &=  (n \wedge m) \| \theta - \theta^* - \frac{1}{1+\epsilon} (\hat{\theta} - \theta^*) \|^2\\
    &=  (n \wedge m) \| \theta - \theta^* - \frac{1}{1+\epsilon} ( I^{-1}_{st}l_{st}(\theta^*) ) \|^2 \\
    &\geq  (n \wedge m)  (\| \theta - \theta^* \| - \frac{1}{1+\epsilon} \|I^{-1}_{st}l_{st}(\theta^*)  \|)^2 \\
    & \geq  (n \wedge m)  (\| \theta - \theta^* \| - \frac{1}{1+\epsilon} \sqrt{ l^T_{st}(\theta^*) I^{-1}_{st}I_0I^{-1}_{st}l_{st}(\theta^*)} )^2  \\
    & \geq  (n \wedge m)  (\| \theta - \theta^* \| - \frac{1}{1+\epsilon} \sqrt{ \frac{1}{n \wedge m}l^T_{st}(\theta^*) I^{-1}_{st}l_{st}(\theta^*)} )^2 \\
  \text{ (Event C) }  & \geq  (n \wedge m) (\delta - \frac{1}{1+\epsilon} \sqrt{ \delta^2 } )^2 \\
    &\geq   \frac{\epsilon^2}{(1+\epsilon)^2}  (n \wedge m)\delta^2 .
\end{align}
Hence in the second integral in the lower bound, for any $\theta \in N^c_{\delta}$, the integrand is not greater than
\begin{align}
    e^{-\frac{(1+\epsilon)}{2}(\theta - u)^TI_{st}(\theta - u)} \leq e^{-\frac{(n\wedge m)\epsilon^2\delta^2}{4(1+\epsilon)}}e^{-\frac{(1+\epsilon)(n\wedge m)\|\theta - u\|^2}{4}}.
\end{align}
By expanding the terms, using the Gaussian integration and rearranging the integration, we have the lower bound and this completes the proof of this lemma.
\end{proof}
With substantially small $\delta$ and $\epsilon$, the integrand of the KL divergence term will approach $\frac{2\tilde{k}+2-c}{2}\log\frac{1}{2\pi} + \log\frac{1}{\omega(\theta^*)} + \frac{1}{2}\log\operatorname{det}(I_{st}) - \frac{1}{2}L(\theta^*)$, hence we define the remaining term $R_{st}$ by
\begin{align}
    R_{st} =  \frac{P_{\theta^*}(D^s_m,D^{t,U}_n)}{Q(D^s_m,D^{t,U}_n)} - \frac{2\tilde{k}+2-c}{2}\log\frac{1}{2\pi } - \log\frac{1}{\omega(\theta^*)} - \frac{1}{2}\log\operatorname{det}(I_{st}) + \frac{1}{2}L(\theta^*).
\end{align}
Using the similar argument in \citet{zhu2020semi} and \citet{clarke1990information}, we can show that the expected remaining term is upper-bounded and lower-bounded by 
\begin{align}
    \mathbb{E}[R_{st}] \geq& -\log(1+\epsilon) - \rho(\delta,\theta^*) - \frac{\epsilon}{2(1-\epsilon)}(2\tilde{k}+2-c) +\frac{2\tilde{k}+2-c}{2}\log\frac{1}{1-\epsilon}\\
    & + \mathbb{P}((A\cap B)^c)\left(\log\mathbb{P}((A\cap B)^c) + \frac{2\tilde{k}+2-c}{2}\log\frac{1}{2\pi} \right) - \mathbb{P}((A\cap B)^c)\log\frac{\operatorname{det}(I_{st})^{\frac{1}{2}}}{\omega(\theta^*)},
\end{align}
and 
\begin{align*}
    & \mathbb{E}[R_{st}]  \leq \rho(\delta,\theta^*) + \frac{\epsilon}{2(1+\epsilon)}(2\tilde{k}+2-c) + \frac{2\tilde{k}+2-c}{2}\log\frac{1}{1+\epsilon} - \log \left(1-2^{\frac{2\tilde{k}+2-c}{2}} e^{-\epsilon^{2} (m \wedge n) \delta^{2} / 8}\right) \\
    & + \mathbb{E}[L(\theta^*)\mathbf{1}_{(B \cap C)^c}] + \mathbb{P}((B \cap C)^c)\left(\frac{2\tilde{k}+2-c}{2}\log\frac{1}{2\pi} + |\log \int_{N_{\delta}}\omega(\theta) d\theta| + \log\frac{\operatorname{det}(I_{st})^\frac{1}{2}}{\omega(\theta^*)} \right) \\
    &+ \mathbb{P}((B \cap C)^c)\mathbb{E}\left[\sup_{\theta,\theta'} (\theta -\theta^*) \nabla\log P_{\theta'}(D^m_s, D^{U,n}_t) \right] \\
    &+ \mathbb{P}((B \cap C)^c)^{\frac{1}{2}}\mathbb{E}\left[ \sup_{\theta,\theta'} (\theta -\theta^*) \nabla^2\log P_{\theta'}(D^m_s, D^{U,n}_t) \right]^{\frac{1}{2}}.
\end{align*}
By application of Condition~2 in Assumption~\ref{asp:para-trans}, with sufficiently small $\delta$, the upper bound will go to zero if the probability of the data pair $D^{U,n}_t$ and $D^m_s$ belong to the set $P(A^c)$, $P(B^c)$ and $P(C^c)$ is $o(\frac{1}{n \vee m})$. In the following, we will show that the probability of $A^c$, $B^c$ and $C^c$ will decay exponentially fast with $m \wedge n$ so that the expected remaining term will converge as $o(\frac{1}{n \vee m})$ under the regime that $m = cn^{p}$ for some $c>0$ and finite $p > 0$. 

\begin{lemma} Assume condition 4 holds so that for all $\theta \in N_{\delta}$, let $v = n \vee m$, then for sufficiently small $\delta$, there is an $ r> 0$ and $\rho > 0$ so that,
\begin{align}
    \mathbb{P}((D^{U,n}_t, D^m_s) \in A^c(\delta, e^{-vr})) = O(e^{-(m\wedge n )\rho}).
\end{align}
\end{lemma}

\begin{proof}
For any given $r'>0$, we define the event
\begin{equation}
    U= \left\{\int_{N_\delta} \omega(\theta)P_{\theta}(D^{U,n}_t, D^m_s)d\theta > e^{-vr'}P_{\theta^*}(D^{U,n}_t, D^m_s) \right\}.
\end{equation}
We can bound the probability of $A^c$ by
\begin{align*}
\mathbb{P}\left(A^{c}\left(\delta, e^{-v r}\right)\right) &=\mathbb{P}\left(\int_{N_{\delta}} P(D^{U,n}_t, D^m_s \mid \theta) \omega(\theta) d \theta<e^{v r} \int_{N_{\delta}^{c}} P(D^{U,n}_t, D^m_s \mid \theta) \omega(\theta) d \theta\right) \\
&\leq \mathbb{P}\left(U \cap \left(\int_{N_{\delta}} P(D^{U,n}_t, D^m_s \mid \theta) \omega(\theta) d \theta\right.\right. \left.\left.<e^{v r} \int_{N_{\delta}^{c}} P(D^{U,n}_t, D^m_s \mid \theta) \omega(\theta) d \theta\right)\right)+\mathbb{P}\left(U^{c}\right) \\
&\leq \mathbb{P}\left(P\left(D^{U,n}_t, D^m_s \mid \theta^* \right)<e^{v\left(r+r^{\prime}\right)} \int_{N^{c}} \omega(\theta) P(D^{U,n}_t, D^m_s \mid \theta) d \theta\right) \\
&\quad +\mathbb{P}\left(e^{vr^{\prime}} \int_{N_{\delta}} P( D^{U,n}_t, D^m_s \mid \theta) \omega(\theta) d \theta < P\left(D^{U,n}_t, D^m_s \mid \theta^*\right)\right).
\end{align*}
For the first term, we use the argument in \citet{clarke1990information} (Eq. (6.6)) and \citet{zhu2020semi} (Lemma 7) and it can be concluded that it is of the order of $O(e^{-(n \wedge m)r''})$ for some $r''$ under the Condition 3 for soundness of the parametric families. For the second term, define $Q(D^{U,n}_t, D^m_s \mid N_{\delta}) = \int_{N_{\delta}}P(X|\theta)\omega(\theta|N_{\delta})d\theta$ and $\omega(\theta|N_{\delta}) = \frac{\omega(\theta)}{\int_{N_{\delta}} \omega(\theta) d\theta}$ and $\tilde{r} = r' - \frac{1}{v}\log \int_{N_{\delta}} \omega(\theta) d\theta$, we can write the probability as,
\begin{align}
\mathbb{P} & \left(e^{vr^{\prime}} \int_{N_{\delta}} P( D^{U,n}_t, D^m_s \mid \theta) \omega(\theta) d \theta < P\left(D^{U,n}_t, D^m_s \mid \theta^*\right)\right) =  \mathbb{P}\left(\log\frac{P(D^{U,n}_t, D^m_s \mid \theta^*)}{Q(D^{U,n}_t, D^m_s \mid N_{\delta})} > v\tilde{r} \right) \\
& \leq \mathbb{P}\left(\log P\left(D^{U,n}_t, D^m_s \mid \theta^*\right)-\int_{N_{\delta}} \log P(D^{U,n}_t, D^m_s \mid \theta) \omega\left(\theta \mid N_{\delta}\right) d \theta> v \tilde{r}\right) \\
&\leq \mathbb{P}\left(\int_{N_{s,\delta}} \log \frac{P\left(D^m_s \mid \theta^*_c,  \theta^*_s \right)}{P\left(D^m_s \mid \theta_c, \theta_s \right)} \omega\left(\theta \mid N_{s,\delta}\right) d \theta +\int_{N_{t,\delta}} \log \frac{P\left(D^{U,n}_t \mid \theta^*_c, \theta^*_t \right)}{P\left(D^{U,n}_t \mid \theta_c, \theta_t \right)} \omega\left(\theta \mid N_{t,\delta}\right) d \theta>v \tilde{r}\right) \\
&=\mathbb{P}\left(\sum_{i=1}^{m} g_s\left(Z^{(i)}_s\right)+\sum_{j=1}^{n} g_t\left(X^{(j)}_t\right)>v \tilde{r}\right) \\
&\leq \mathbb{P}\left(\frac{1}{v} \sum_{i=1}^{m} g_s\left(Z^{(i)}_s\right) >\tilde{r} / 2\right)+\mathbb{P}\left(\frac{1}{v} \sum_{j=1}^{n} g_t\left(X^{(j)}_t\right) >\tilde{r} / 2\right) \\
&\leq \mathbb{P}\left(\frac{1}{m} \sum_{i=1}^{m} g_s\left(Z^{(i)}_s\right) >\tilde{r} / 2\right)+\mathbb{P}\left(\frac{1}{n} \sum_{j=1}^{n} g_t\left(X^{(j)}_t\right) >\tilde{r} / 2\right),
\end{align}
where we define,
\begin{align}
    g_s(Z^{(i)}_s) &:= \int_{N_{s,\delta}} \log \frac{P\left(Z^{(i)}_s \mid \theta^*_s, \theta^*_c \right)}{P(Z^{(i)}_s \mid \theta_s,\theta_c)} \omega\left(\theta_s,\theta_c \mid N_{s,\delta}\right) d \theta_s d\theta_c, \\
    g_t(X^{(j)}_t) &:= \int_{N_{t,\delta}} \log \frac{P\left(X^{(j)}_t \mid \theta^*_t, \theta^*_c \right)}{P(X^{(j)}_t \mid \theta_t,\theta_c)} \omega\left(\theta_t,\theta_c \mid N_{t,\delta}\right) d \theta_s d\theta_c, \\
    \omega(\theta_c, \theta_s\mid N_{s,\delta}) &:= \frac{\omega(\theta_c,\theta_s)}{\int_{N_{s,\delta}} \omega(\theta_c,\theta_s) d\theta_c d\theta_s}, \\
    \omega(\theta_c, \theta_t\mid N_{t,\delta}) &:= \frac{\omega(\theta_c,\theta_t)}{\int_{N_{t,\delta}} \omega(\theta_c,\theta_t) d\theta_c d\theta_t}.
\end{align}
In this case, we use a slightly different notation that $N_{s,\delta} = \{\theta_{sc}: \|\theta_{sc} - \theta^*_{sc}\| \leq \delta \}$, where $\theta_{sc} = (\theta_s,\theta_c)$ denotes the source parameters and the norm is w.r.t. the Fisher information matrix $I_{s}$, e.g., $\|\theta_{sc}\|^2 = \theta^T_{sc} I_s \theta_{sc}$. Similarly, $N_{t,\delta} = \{\theta_{tc}: \|\theta_{tc} - \theta^*_{tc}\| \leq \delta \}$ where $\theta_{tc} = (\theta_t,\theta_c)$ denotes the target parameters and norm is w.r.t. the Fisher information matrix $I_t$ as defined previously. The second inequality holds due to that $I_t \prec I_t+I_{sc}$ and $I_s \prec I_s + I_{tc}$ for $I_{sc} = -\mathbb{E}_{\theta^*_s, \theta^*_c}[ \nabla^2 \log P(X_s,Y_s|\theta^*_t, \theta^*_c)]$ and $I_{tc} = -\mathbb{E}_{\theta^*_t, \theta^*_c}[ \nabla^2 \log P(X_t|\theta^*_t, \theta^*_c)]$ the fisher information matrix w.r.t. $\theta^*_c$ in both source and target domains, with the fact that $N_{\delta} \subset N_{s,\delta}$ and $N_{\delta} \subset N_{t,\delta}$. If the source and target domain share the same parameters (e.g., $c = \tilde{k}+1$), then our case generalizes to Lemma 7 in \citet{zhu2020semi}. 

\end{proof}
\begin{lemma} Assume condition 5 holds so that for sufficiently small $\delta$,  there is some $\rho > 0$ such that,
\begin{align}
    \mathbb{P}((D^{U,n}_t, D^m_s) \in B^c(\delta, \epsilon)) = O(e^{-\rho(m\wedge n )}).
\end{align}
\end{lemma}

\begin{proof}
The proof exactly follow \citet{zhu2020semi} with similar assumptions, which is omitted here.
\end{proof}

\begin{lemma} 
Assume condition 6 holds, then for sufficiently small $\delta$, there is a $\rho > 0$ so that,
\begin{align}
    \mathbb{P}((D^{U,n}_t, D^m_s) \in C^c(\delta)) = O(e^{-(m\wedge n )\rho}).
\end{align}
\end{lemma}

\begin{proof}
We firstly expand the term $L(\theta^*)$ by: 
\begin{align}
   L(\theta^*) &= l^T_{st} I^{-1}_{st} l^T_{st} \\
               &= \sum_{i=1}^m l^T_{s,i} I^{-1}_{st} l_{s,i} + \sum^m_{i\neq k} l^T_{s,i}I^{-1}_{st}l_{s,k} + \sum_{i=1}^{n}l^T_{t,i} I^{-1}_{st} l_{t,i} + \sum^n_{i\neq k} l^T_{t,i}I^{-1}_{st}l_{t,k} \\
               &\quad + 2\sum^{n}_{i=1}\sum^m_{k=1} l^T_{t,i} I^{-1}_{st} l_{s,k}
\end{align}
Then we have that,
\begin{align*}
     \mathbb{P}&((D^{U,n}_t, D^m_s) \in C^c(\delta)) = \mathbb{P}(L(\theta^*) > (n\wedge m) \delta^2) \\
     &\leq \mathbb{P}\left(\frac{1}{m}\sum_{i=1}^{m}l^T_{s,i} I^{-1}_{st} l_{s,i} \geq \frac{(n\wedge m)\delta^2}{6m} \right) + \mathbb{P}\left(\frac{1}{m(m-1)}\sum^m_{i\neq k} l^T_{s,i}I^{-1}_{st}l_{s,k} \geq \frac{(n\wedge m)\delta^2}{6m(m-1)} \right) \\
     &\quad + \mathbb{P}\left(\frac{1}{n}\sum_{i=1}^{n}l^T_{t,i} I^{-1}_{st} l_{t,i} \geq \frac{(n\wedge m)\delta^2}{6n} \right) + \mathbb{P}\left( \frac{1}{n(n-1)}\sum^n_{i\neq k} l^T_{t,i}I^{-1}_{st}l_{t,k} \geq \frac{(n\wedge m)\delta^2}{6n(n-1)} \right) \\
     &\quad + \mathbb{P}\left( \frac{2}{nm}\sum^{n}_{i=1}\sum^m_{k=1} l^T_{t,i} I^{-1}_{st} l_{s,k}  \geq \frac{(n\wedge m)\delta^2}{3nm} \right)
\end{align*}
We first consider the case where $m = cn^{p}$ for $p \geq 1$, then we can show that these five terms will decay exponentially fast. We first bound the expected value by 
\begin{align}
    \mathbb{E}[l^T_{s,i}I_{st}l_{s,i}] &= \operatorname{Tr}(I^{-1}_{st}\mathbb{E}[l^T_{s,i}l_{s,i}]) \\
    &\leq \frac{1}{n\wedge m} \operatorname{Tr}(I^{-1}_0)I_s \\
    &\leq  \frac{1}{n\wedge m} \operatorname{Tr}(I^{-1}_0)I_0) \\
    &=\frac{2\tilde{k}+2-c}{n\wedge m}
\end{align}
since $I_s \prec I_0$ due to the Condition 2. Also we have for large $m$,
\begin{align}
   \mathbb{E}[l^T_{t,i}I_{st}l_{t,i}] =  \frac{2\tilde{k}+2-c}{m},
\end{align}
and 
\begin{align}
    \mathbb{E}[l^T_{t,i}I_{st}l_{t,k}] = 0, \\
    \mathbb{E}[l^T_{s,i}I_{st}l_{s,k}] = 0, \\
    \mathbb{E}[l^T_{t,i}I_{st}l_{s,k}] = 0 
\end{align}
due to that $l_{s,i}$ and $l_{s,k}$ are mutually independent. Since the Condition 6 holds, we will use the Chernoff bound again so that the inequality is bounded by $O(e^{-\rho (n \wedge m)})$ for some $\rho >0$ under the case that $m = cn^{p}$ for $p \geq 1$ as \citet{zhu2020semi} (Lemma 9) suggested, where the details are omitted here. For the case where $m = cn^p$ for some $0 < p <1$, since for large $n \gg m$,
\begin{align}
   \mathbb{E}[l^T_{t,i}I_{st}l_{t,i}] = \frac{2\tilde{k}+2-c}{n}.
\end{align}
We can upper bound the term on the source score function by,
\begin{align}
    \mathbb{E}[l^T_{s,i}I_{st}l_{s,i}] &= \operatorname{Tr}(I^{-1}_{st}\mathbb{E}[l^T_{s,i}l_{s,i}])  \\
    &\leq \frac{2\tilde{k}+2-c}{n \wedge m}.
\end{align}
Then similar argument can be made that the probability is bounded by $O(e^{-\rho' (n \wedge m)})$ for some $\rho' >0$, and this completes the proof for all $p > 0$.

\end{proof}
Overall, putting everything together we complete the proof.

\end{proof}

\subsection{Proof of Theorem~\ref{thm:excessrisk-log}}

\begin{proof}
We firstly show that given any prior over $\Theta_s$ and $\Theta_t$,
\begin{align*}
   &I(Y'_t;\Theta_t, \Theta_s|D^{U,n}_t, D^{m}_s, X'_t)  \\
   &= I(\Theta_t, \Theta_s; Y'_t, X'_t, D^{U,n}_t, D^{m}_s) - I(\Theta_t, \Theta_s ; X'_t, D^{U,n}_t, D^{m}_s)  \\
   &= D(P_{\Theta_t,\Theta_s}(D^{U,n}_t,D^{m}_s, Y'_t, X'_t)\|Q(D^{U,n}_t,D^m_s, Y'_t, X'_t))  \\
    &\quad - D(P_{\Theta_s,\Theta_t}(D^m_s, D^{U,n}_t, X'_t)\|Q(D^m_s, D^{U,n}_t, X'_t))  \\ 
   &=  \int \Bigg( \mathbb{E}_{\theta_s,\theta_t} \left[ \log \frac{P_{\theta_t,\theta_s}(D^{U,n}_t,D^{m}_s, Y'_t, X'_t)}{Q(D^{U,n}_t,D^{m}_s, Y'_t, X'_t)} \right]  - \mathbb{E}_{\theta_s,\theta_t} \left[ \log \frac{P_{\theta_t,\theta_s}(D^m_s, D^{U,n}_t, X'_t)}{Q(D^m_s, D^{U,n}_t, X'_t)} \right]  \Bigg) \omega(\theta_s,\theta_t) d\theta_s d\theta_t  \\
   &= \int \left( \mathbb{E}_{\theta_s,\theta_t} \left[ \log \frac{P_{\theta_t}(Y'_t|X'_t)}{Q(Y'_t|D^{U,n}_t, D^{m}_s, X'_t)} \right] \right) \omega(\theta_s,\theta_t) d\theta_s d\theta_t, 
\end{align*}
where in the last equality we use the chain rule and the assumption that both source and target data are drawn in an i.i.d. way under Assumption~\ref{asp:para-dist}. The mutual information density at $\Theta_s = \theta^*_s$ and $\Theta_t = \theta^*_t$ is then given by
\begin{align*}
     \mathcal{R}(b) &= I(Y'_t;\theta^*_t, \theta^*_s|D^{U,n}_t, D^{m}_s, X'_t) \\
     &= \mathbb{E}_{\theta^*_s, \theta^*_t} \left[ \log \frac{P_{\theta^*_t}(Y'_t|X'_t)}{Q(Y'_t|D^{U,n}_t, D^{m}_s, X'_t)} \right] ,
\end{align*}
which completes the proof.
\end{proof}

\subsection{Proof of Theorem~\ref{thm:excessrisk-general}}

\begin{proof}
We can show that the expected excess risk can be bounded by
\begin{align*}
   \mathcal{R}(b) &= \mathbb{E}_{\theta^*_t,\theta^*_s} \left[  \ell(b, Y'_t) -  \ell(b^*, Y'_t) \right] \\
   &=   \mathbb{E}_{D^{m}_s,D^{U,n}_t,X'_t, Y'_t}\mathbb{E}_{Y'_t}\left[\ell(b, Y'_t)  -  \ell(b^*,Y'_t)|D^m_s, D^{U,n}_t, X'_t\right] \\
   &=   \mathbb{E}_{D^{m}_s,D^{U,n}_t,X'_t}\sum_{y'_t} \left( \ell(b, y'_t)  -  \ell(b^*, y'_t) \right) P_{\theta^*_s,\theta^*_t}(y'_t|D^m_s, D^{U,n}_t, X'_t) \\
  &=   \mathbb{E}_{D^{m}_s,D^{U,n}_t,X'_t}\sum_{y'_t} \left( \ell(b, y'_t)  -  \ell(b^*, y'_t) \right) (P_{\theta^*_s,\theta^*_t}(y'_t|D^m_s, D^{U,n}_t, X'_t)  \\ 
  & \quad - Q(y'_t|D^m_s, D^{U,n}_t, X'_t) +  Q(y'_t|D^m_s, D^{U,n}_t, X'_t))  \\
  &\overset{(a)}{\leq} \mathbb{E}_{D^{m}_s, D^{U,n}_t, X'_t}\sum_{y'_t} \left( \ell(b, y'_t)  -  \ell(b^*, y'_t) \right) (P_{\theta^*_s,\theta^*_t}(y'_t|D^m_s, D^{n}_t,X'_t)  - Q(y'_t|D^m_s, D^{U,n}_t, X'_t))  \\
  &\overset{(b)}{\leq}  M \mathbb{E}_{D^{m}_s,D^{U,n}_t, X'_t} \sum_{y'_t}  (P_{\theta^*_s,\theta^*_t}(y'_t|D^m_s, D^{U,n}_t,X'_t) - Q(y'_t|D^m_s, D^{U,n}_t, X'_t))  \\
  &\overset{(c)}{\leq}  M \mathbb{E}_{D^{m}_s,D^{U,n}_t, X'_t} \sqrt{2D\left(P_{\theta^*_s,\theta^*_t}(Y'_t|D^m_s, D^{U,n}_t,X'_t) \| Q(y'_t|D^m_s, D^{U,n}_t, X'_t)\right)} \\
  &\overset{(d)}{\leq} M\sqrt{2 \mathbb{E}_{D^{m}_s,D^{U,n}_t,X'_t} D\left(P_{\theta^*_s,\theta^*_t}(Y'_t|D^m_s, D^{U,n}_t,X'_t) \| Q(Y'_t|D^m_s, D^{U,n}_t, X'_t)\right)} \\
  &= M\sqrt{2D\left(P_{\theta^*_t} \| Q | D^m_s, D^{U,n}_t, X'_t\right)} \\
  &= M\sqrt{2D(P_{\theta^*_t}(Y'_t|X'_t) \| Q(Y'_t|D^{U,n}_t,D^m_s,X'_t))} \\
  &= M\sqrt{2I(Y'_t;\Theta_t = \theta^*_t, \Theta_s = \theta^*_s|D^{m}_s,D^{U,n}_t,X'_t) },
\end{align*} 
where in $(a)$ we use the definition of $Q$, then $(b)$ holds since we assume the loss function is bounded, $(c)$ follows from the Pinsker's inequality, $(d)$ holds from the Jensen's inequality.
\end{proof}

\subsection{Proof of Theorem~\ref{thm:causal}}

We firstly consider the scenario for covariate shift condition where $P_S(X) \neq P_T(X)$ and $P_S(Y|X) = P_T(Y|X)$.
\begin{proof}
 Knowing the conditions $\theta^{s*}_{Y_{x_i}} = \theta^{t*}_{Y_{x_i}}$ for every $i = 1,2,\cdots,k$, we choose the prior distribution $\omega(\Theta_s,\Theta_t)$ as
\begin{align}
    \omega(\Theta_s,\Theta_t) = \omega(\Theta^{t}_{X})\omega(\Theta^{s}_{X})\omega(\Theta^{st}_{Y_X}) .
\end{align}
In the causal setting, $\Theta^{s}_{X}$ is usually considered as independent of $\Theta^{t}_{X}$ and $\Theta^{s}_{Y_X}$. We also set the parameter $\Theta^{t}_{Y_X} = \Theta^{s}_{Y_X}$  from the assumption $P_S(Y|X) = P_T(Y|X)$ and denote it by $\Theta^{st}_{Y_X}$. With a proper prior distribution, we will arrive at the asymptotic estimation of the expected excess risk as
\begin{align}
    & D(P_{\theta^*_t,\theta^*_s}(D^{\textup{U},n}_t,X'_t, Y'_t, D^{m}_s)\|Q(D^{\textup{U},n}_t,X'_t, Y'_t, D^{m}_s))  - D(P_{\theta^*_t,\theta^*_s}(D^{\textup{U},n}_t,X'_t, D^{m}_s)\|Q(D^{\textup{U},n}_t,X'_t, D^{m}_s)) \nonumber \\
    &=  D(P_{\theta^{t*}_{X}}(D^{\textup{U},n}_t,X'_t)\| Q(D^{\textup{U},n}_t,X'_t)) + D(P_{\theta^{s*}_{X}}(X^m_s)\| Q(X^m_s)) + D(P_{\theta^*_{Y_X}}(Y'_{X',t}, Y^m_{X,s})\|Q(Y'_{X',t}, Y^m_{X,s})) \nonumber \\
    &\quad - D(P_{\theta^{t*}_{X}}(D^{\textup{U},n}_t,X'_t)\| Q(D^{\textup{U},n}_t,X'_t)) + D(P_{\theta^{s*}_{X}}(X^m_s)\| Q(X^m_s)) - D(P_{\theta^*_{Y_X}}(Y^m_{X,s})\|Q(Y^m_{X,s})) \nonumber \\
    & =  D(P_{\theta^*_{Y_X}}(Y'_{X',t}, Y^m_{X,s})\|Q(Y'_{X',t}, Y^m_{X,s})) - D(P_{\theta^*_{Y_X}}(Y^m_{X,s})\|Q(Y^m_{X,s}))  \\
    & = \frac{1}{2}\log\operatorname{det} \mathbf{I}_{1} + \frac{k}{2}\log \frac{1}{2\pi e} + \log\frac{1}{\omega(\theta^*_{Y_X})} - \frac{1}{2}\log\operatorname{det} \mathbf{I}_{0} - \frac{k}{2}\log \frac{1}{2\pi e} - \log\frac{1}{\omega(\theta^*_{Y_X})}  + o(\frac{1}{m})\\
    &= \frac{1}{2}\frac{\log\operatorname{det} (\mathbf{I}_{1}) }{\log\operatorname{det} (\mathbf{I}_{0}) } + o(\frac{1}{m}),
\end{align}
where we use the i.i.d. property of the data distribution and the independence property of the prior distribution among $\Theta^t_X$, $\Theta^s_X$ and $\Theta^{st}_{Y_X}$. Since $Y'_{X',t}$ and  $Y^m_{X,s}$ are parameterized by the same set of parameters $\theta^*_{Y_X}$, we denote the Fisher information matrix of $P_{\theta^{*}_{Y_X}}(Y)$ for source and target domains by
\begin{align}
  I(\theta^{*}_{Y_{x_i}}) &= \mathbb{E}_{Y_{x_i}}[\partial^2 \log P_{\theta^{*}_{Y_{x_i}}}(Y) / (\partial \theta_{Y_{x_i}})^2], \textup{ for } i = 1,2,\cdots, k, \\
  I_s(\theta^{*}_{Y_X}) &= -\mathbb{E}_{Y_{X_s}}\left[ \partial^2 \log P_{\theta^{*}_{Y_X}}(Y_{X}) / \partial \theta_j \partial \theta_k \right]_{j,k = 1,2,\cdots,k} = \operatorname{diag}[P_{\theta^{s*}_X}(X = x_i) * I(\theta^*_{Y_{x_i}}) ]_{i=1,\cdots,k}, \\
  I_t(\theta^{*}_{Y_X}) &= -\mathbb{E}_{Y_{X_t}}\left[ \partial^2 \log P_{\theta^{*}_{Y_X}}(Y_{X}) / \partial \theta_j \partial \theta_k \right]_{j,k = 1,2,\cdots,k} = \operatorname{diag}[P_{\theta^{t*}_X}(X = x_i) * I(\theta^*_{Y_{x_i}}) ]_{i=1,\cdots,k},
\end{align}
due to the mutually independence property of $Y_{X_i}$. Then $\mathbf{I}_1$ and $\mathbf{I}_0$ are expressed as follows.
\begin{align}
   \mathbf{I}_{0} &= mI_{s}(\theta^{*}_{Y_X}), \\
    \mathbf{I}_{1} &= mI_{s}(\theta^{*}_{Y_X}) + I_{t}(\theta^{*}_{Y_X}).
\end{align}
With the assumptions that the Fisher information matrix around true $\theta^{*}_{Y_X}$ are bounded and positive definite, we can calculate the excess risk by
\begin{align}
    \mathcal{R}(b) &= \frac{1}{2}\frac{\log\operatorname{det} (\mathbf{I}_{1}) }{\log\operatorname{det} (\mathbf{I}_{0}) } + o(\frac{1}{m})\\
    & =  \frac{1}{2}\log\operatorname{det}\left( \mathbf{I}_{k} + \frac{1}{m} I_{t}(\theta^{t*}_{Y_X}) I^{-1}_{s}(\theta^{s*}_{Y_X})\right) + o(\frac{1}{m}).
\end{align}
We then use the expansion of determinant:
\begin{align}
    \operatorname{det}(\mathbf{I}+\frac{1}{m} A) = 1 + \frac{1}{m} \operatorname{Tr}(A)+o(1 / m).
\end{align}
As a consequence,
\begin{align}
    \mathcal{R}(b) &= \frac{1}{2}\log \left( 1 + \frac{1}{m} \operatorname{Tr}(I_{t}(\theta^{*}_{Y_X}) I^{-1}_{s}(\theta^{*}_{Y_X}))+o(1 / m) \right) + o(\frac{1}{m}) \\
    &=\frac{1}{2}\log \left( 1 + \frac{1}{m} \sum_{i=1}^{k} \frac{P_{\theta^{t*}_X}(X = x_i)}{P_{\theta^{s*}_X}(X = x_i)}+o(1 / m) \right) + o(\frac{1}{m}) \\
    &\asymp \left(\frac{\sum_{i=1}^{k} \frac{P_{\theta^{t*}_X}(X = x_i)}{P_{\theta^{s*}_X}(X = x_i)}}{m} \right) \\
    &\asymp \frac{k}{m}. 
\end{align}
given that $P_{\theta^{t*}_X}(X = x_i)$ and $P_{\theta^{t*}_X}(X = x_i)$ are positive and bounded for any $i$. In other word, the convergence is guaranteed only when the source and target domains share the same support of the input $X$. For the case $\theta^{s*}_{X} = \theta^{t*}_{X}$, using the same procedure, by choosing
\begin{align}
    \omega(\Theta_s,\Theta_t) = \omega(\Theta^{st}_{X})\omega(\Theta^{st}_{Y_X}).
\end{align}
we will also arrive at
\begin{align}
    \mathcal{R}(b) &= \frac{1}{2}\log \left( 1 + \frac{1}{m} \operatorname{Tr}(I_{t}(\theta^{t*}_{Y_X}) I^{-1}_{s}(\theta^{s*}_{Y_X}))+o(1 / m) \right) + o(\frac{1}{m})\\
    &\asymp  \frac{k}{m}.
\end{align}
which leads to the same rate and completes the proof.
\end{proof}
Next we will look at the concept drift scenario where $P_S(Y|X) \neq P_T(Y|X)$ and $P_S(X) = P_T(X)$.
\begin{proof}
 Knowing the conditions $\theta^{s*}_{Y_{x_i}} \neq \theta^{t*}_{Y_{x_i}}$ for every $i = 1,2,\cdots,k$, if $P_S(X) = P_T(X)$, we choose the prior distribution $\omega(\Theta_s,\Theta_t)$ as
\begin{align}
    \omega(\Theta_s,\Theta_t) = \omega(\Theta^{st}_{X})\omega(\Theta^{s}_{Y_X})\omega(\Theta^{t}_{Y_X}).
\end{align}
following the similar machinery in the covariate shift conditions. Then the mixture distribution $Q$ becomes
\begin{align}
    & Q(Y'_t|D^{\textup{U},n}_t, D^m_s,X'_t)  \\
    &= \frac{\int P_{\theta_t}(D^{\textup{U},n}_t, X'_t, Y'_t) P_{\theta_s}(D^m_s) \omega(\theta_t,\theta_s) d\theta_t d\theta_s }{\int P_{\theta_t}(X'_t)P_{\theta_t}(D^{\textup{U},n}_t) P_{\theta^*_s}(D^m_s) \omega(\theta_t,\theta_s) d\theta_t d\theta_s  }  \\
    &= \int  P_{\theta_t}(Y'_t|X'_t) P(\theta_t,\theta_s|X'_t, D^m_s, D^{\textup{U},n}_s) d\theta_sd\theta_t 
    \\
    &= \int  P_{\theta^{t}_{Y_X}}(Y'_t|X'_t) P(\theta^{st}_{X},\theta^{s}_{Y_X}, \theta^{t}_{Y_X} |X'_t, D^m_s, D^{\textup{U},n}_s) d\theta^{s}_{Y_X} d\theta^{t}_{Y_X} \theta^{st}_{X}  \\
    &\overset{(a)}{=}\int  P(Y'_t|X'_t, \theta_{Y_{X}}) \omega(\theta_{Y_X}) d\theta_{Y_X}   \\
    &= \int  P_{\theta_{Y_{X'_t}}}(Y'_t) \omega(\theta_{Y_{X'_t}}) d\theta_{Y_{X'_t}},  
\end{align}
where $(a)$ holds because $X'_t, D^m_s$ and $D^{\textup{U},n}_s$ are all independent of $\Theta^t_{Y_X}$. Therefore, the excess risk becomes,
\begin{align}
     \mathcal{R}(b) =& \mathbb{E}_{\theta^*_s, \theta^*_t, X'_t, Y'_t}\left[ \log \frac{P(Y'_t|\theta^{t*}_{Y|X}, X'_t)}{Q(Y'_t|D^{\textup{U},n}_t, D^m_s,X'_t)} \right] \nonumber \\
     =&  \mathbb{E}_{\theta^{t*}_X}[\textup{KL}(P_{\theta^{t*}_{Y_{X'_t}}}(Y'_t)\|Q(Y'_t|X'_t)].  \label{eq:concept_drift}
\end{align}
If $P_S(X) \neq P_T(X)$, we choose the prior distribution $\omega(\Theta_s,\Theta_t)$ as,
\begin{align}
    \omega(\Theta_s,\Theta_t) = \omega(\Theta^{t}_{X})\omega(\Theta^{s}_{X})\omega(\Theta^{s}_{Y_X})\omega(\Theta^{t}_{Y_X}),
\end{align}
where we will end up with the same results as (\ref{eq:concept_drift}).
\end{proof}

\subsection{Proof of Theorem~\ref{thm:anti-case}}

Before proving Theorem~\ref{thm:anti-case}, we first restate the definition for Fisher information matrix and define extra quantities for proving purposes. 
\begin{align}
    I_s &= -\mathbb{E}_{\theta^*_s}[ \nabla^2 \log P(X_s,Y_s|\theta^*_s)], \\
    I_t &= -\mathbb{E}_{\theta^*_t}[ \nabla^2 \log P(X_t|\theta^*_{t})], \\
    I_{t,X,Y} &= -\mathbb{E}_{\theta^*_t}[ \nabla^2 \log P(X_t,Y_t|\theta^*_{t})], \\
    I_{0} &= -\mathbb{E}_{\theta^*}[ \nabla^2 \log P(X_t, X_s, Y_s|\theta^*)], \\
    I_{t,Y,U} &= -\mathbb{E}_{\theta^*_{t}}[ \nabla^2_{\theta_Y} \log P(X_t|\theta^{t*}_{Y},\theta^{t*}_{X_Y})], \\
    I_{t,Y} &= -\mathbb{E}_{Y_t}[ \nabla^2_{\theta_Y} \log P(Y_t|\theta^{t*}_{Y})], \\
    I_{s,Y} &= -\mathbb{E}_{Y_s}[ \nabla^2_{\theta_Y} \log P(Y_s|\theta^{s*}_{Y})], \\
    I_{t,X_Y,U} &= -\mathbb{E}_{\theta^*_{t}}[ \nabla^2_{\theta_{X_Y}} \log P(X_t|\theta^{t*}_{Y},\theta^{t*}_{X_Y})], \\
    I_{t,X_Y} &= -\mathbb{E}_{\theta^*_{t}}[ \nabla^2_{\theta_{X_Y}} \log P(X_t|\theta^{t*}_{X_{Y_t}})], \\
    I_{s,X_Y} &= -\mathbb{E}_{\theta^{*}_{s}}[ \nabla^2_{\theta_{X_Y}} \log P(X_s|\theta^{s*}_{X_{Y_s}})].
\end{align}

Now we will firstly consider the case $P_{S}(Y) \neq P_S(Y)$ and $P_S(X|Y) \neq P_T(X|Y)$. 
\begin{proof}
Knowing the conditions $\theta^{s*}_{Y} \neq \theta^{t*}_{Y}$ and $\theta^{s*}_{X_{y_i}} \neq \theta^{t*}_{X_{y_i}}$ for every $i = 1,2,\cdots,k'$, we then choose the prior distribution $\omega(\Theta_s,\Theta_t)$ as
\begin{align}
    \omega(\Theta_s,\Theta_t) = \omega(\Theta^{t}_{Y})\omega(\Theta^{s}_{Y})\omega(\Theta^{s}_{X_Y})\omega(\Theta^{t}_{X_Y}).
\end{align}
 With such a prior distribution, we will arrive at the asymptotic estimation of the KL divergence as,
\begin{align}
    D(P_{\theta^*_t,\theta^*_s}(D^{\textup{U},n}_t,D^{m}_s,X'_t, Y'_t ) & \|Q(D^{\textup{U},n}_t, D^{m}_s, X'_t, Y'_t)) \nonumber \\
    & = \frac{1}{2}\log\operatorname{det} \mathbf{I}_{\theta} + \log \frac{1}{2\pi e} + \log\frac{1}{\omega(\theta^*_s,\theta^*_t)} + o(\frac{1}{m \vee n}) ,
\end{align}
where 
\begin{equation}
   \mathbf{I}_{\theta} = \begin{bmatrix}
    nI_{t} + I_{t,X,Y}   &  \mathbf{0}  \\
    \mathbf{0}   &    m I_{s}
    \end{bmatrix}.
\end{equation}
We also have,
\begin{align}
   D(P_{\theta^*_t,\theta^*_s}(D^{\textup{U},n}_t,X'_t, D^{m}_s) &\|Q(D^{\textup{U},n}_t,X'_t, D^{m}_s)) \nonumber \\
    & = \frac{1}{2}\log\operatorname{det} \widetilde{\mathbf{I}}_{\theta} + \log \frac{1}{2\pi e} + \log\frac{1}{\omega(\theta^*_s,\theta^*_t)} + o(\frac{1}{m \vee n}),
\end{align}
where 
\begin{equation}
 \widetilde{\mathbf{I}}_{\theta} = \begin{bmatrix}
  (n+1)I_{t} &  \mathbf{0}  \\
   \mathbf{0} & m I_{s}  
    \end{bmatrix}.
\end{equation}
Then the regret can be calculated by
\begin{align}
    \mathcal{R}(b) &= \frac{1}{2}\frac{\log\operatorname{det} (\mathbf{I}_{\theta}) }{\log\operatorname{det} (\widetilde{\mathbf{I}}_{\theta}) } + o(\frac{1}{m \vee n}) \\
    &=  \frac{1}{2}\log\operatorname{det}\left( \mathbf{I}_{k'+1} + \frac{1}{n+1} (I_{t,X,Y} - I_{t} )I^{-1}_{t}\right) + o(\frac{1}{m \vee n}) \\
    &\asymp \log\left(1 + \frac{\operatorname{Tr}((I_{t,X,Y} - I_{t} )I^{-1}_{t})}{n+1}) \right) \\
    &\asymp \frac{k'+1}{n+1},
\end{align}
which completes the proof.
\end{proof}
Now we turn to conditional shifting case $P_{S}(Y) = P_S(Y)$ and $P_S(X|Y) \neq P_T(X|Y)$. 
\begin{proof}
In this section, we define,
\begin{equation}
    I_{t,U} = -\mathbb{E}_{\theta^{*}_{t}}\left[\frac{\partial^2 \log P(X_t|\theta^{*}_{t})}{\partial\theta_{Y}\partial \theta_{X_{y_i}}}  \right] \text{ for } i = 1,2,\cdots,k'.  \label{eq:itu}
\end{equation}
Knowing the conditions $\theta^{s*}_{Y} = \theta^{t*}_{Y}$ and $\theta^{s*}_{X_{y_i}} \neq \theta^{t*}_{X_{y_i}}$ for every $i = 1,2,\cdots,k'$, we then choose the prior distribution $\omega(\Theta_s,\Theta_t)$ as
\begin{align}
    \omega(\Theta_s,\Theta_t) = \omega(\Theta^{st}_{Y})\omega(\Theta^{s}_{X_Y})\omega(\Theta^{t}_{X_Y}).
\end{align}
where we denote the random variable for estimating $\theta^{st*}_Y$ by $\Theta^{st}_Y$. With such a prior distribution, we will arrive at the asymptotic estimation of the KL divergence as,
\begin{align}
    D(P_{\theta^*_t,\theta^*_s}(D^{\textup{U},n}_t,D^{m}_s,X'_t, Y'_t )& \|Q(D^{\textup{U},n}_t, D^{m}_s, X'_t, Y'_t)) \nonumber \\
       & = \frac{1}{2}\log\operatorname{det} \mathbf{I}_{\theta} +  \log \frac{1}{2\pi e} + \log\frac{1}{\omega(\theta^*_s,\theta^*_t)} + o(\frac{1}{m \vee n}),
\end{align}
where the joint Fisher information matrix $\mathbf{I}_{\theta}$ is defined as,
\begin{equation}
   \mathbf{I}_{\theta} = \begin{bmatrix}
    nI_{t,Y,U} + mI_{s,Y} + I_{t,Y}   &  nI_{t,U}   & \mathbf{0}  \\
    nI^T_{t,U} &  nI_{t,X_Y, U} + I_{t,X_{Y}} &    \mathbf{0} \\
   \mathbf{0} &  \mathbf{0} &  m I_{s,X_Y}
    \end{bmatrix}.
\end{equation}
Here zero vectors are due to the mutually independence assumption between the distribution parameters and i.i.d. assumption on the source and target samples. We also have,
\begin{align}
    &D(P_{\theta^*_t,\theta^*_s}(D^{\textup{U},n}_t,X'_t, D^{m}_s)\|Q(D^{\textup{U},n}_t,X'_t, D^{m}_s))  = \frac{1}{2}\log\operatorname{det} \widetilde{\mathbf{I}}_{\theta} +  \log \frac{1}{2\pi e} + \log\frac{1}{\omega(\theta^*_s,\theta^*_t)} + o(\frac{1}{m \vee n}),
\end{align}
where 
\begin{equation}
   \widetilde{\mathbf{I}}_{\theta} = \begin{bmatrix}
    (n+1)I_{t,Y,U} + mI_{s,Y}   &  (n+1)I_{t,U} & \mathbf{0}  \\
    (n+1)I^T_{t,U}  &  (n+1)I_{t,X_Y,U}  &    \mathbf{0} \\
   \mathbf{0} &  \mathbf{0} &  m I_{s,X_Y}
    \end{bmatrix}.
\end{equation}
Assume $m = cn^p$ for some $p>0$, as $n$ goes to infinity, we define the scalars $\Delta_U = I_{t,Y,U} - I_{t,U}I^{-1}_{t,X_Y,U}I^T_{t,U}$ and $\Delta_s = I_{s,Y}$, then the regret can be calculated by
\begin{align}
    \mathcal{R}(b) =& \frac{1}{2}\frac{\log\operatorname{det} (\mathbf{I}_{\theta}) }{\log\operatorname{det} (\widetilde{\mathbf{I}}_{\theta}) } + o(\frac{1}{m \vee n}) \\
    =& \frac{1}{2} \left( \log\operatorname{det}(\mathbf{I}_{k} + \frac{1}{n+1} (I_{t,X_Y} - I_{t,X_Y,U}) I^{-1}_{t,X_Y, U}  ) + \log\operatorname{det}(1 + \frac{I_{t,Y} - I_{t,Y,U}}{(n+1)\Delta_U + cn^p\Delta_s}) \right) \\
    &+ o(\frac{1}{m \vee n}) \\
    \asymp &  \frac{k'}{n+1} +  \frac{1}{(n+1) \vee n^p} \\
    \asymp &  \frac{k'}{n} +  \frac{1}{n \vee n^p}, 
\end{align}
where $\mathbf{I}_{k}$ denotes the identity matrix with dimension of $k\times k$. Since we assume $I_t \succ 0$ and $I_s \succ 0$, we have that $\Delta_U > 0$ and $\Delta_s > 0$. From the information processing perspective, the labelled target data always contains more information than unlabelled target data, hence we have both $I_{t,X_Y} - I_{t,X_Y,U} \succ 0$ and $I_{t,Y} - I_{t,Y,U} \succ 0$, which completes the proof.
\end{proof}
Regarding the target shift scenario $P_{S}(Y) \neq P_S(Y)$ and $P_S(X|Y) = P_T(X|Y)$, we could follow the similar procedures as the label drifting case.
\begin{proof}
Knowing the conditions $\theta^{s*}_{Y} = \theta^{t*}_{Y}$ and $\theta^{s*}_{X_{y_i}} \neq \theta^{t*}_{X_{y_i}}$ for every $i = 1,2,\cdots,k'$, we choose the prior distribution as,
\begin{align}
    \omega(\Theta_s,\Theta_t) = \omega(\Theta^{s}_{Y})\omega(\Theta^{t}_{Y})\omega(\Theta^{st}_{X_Y}).
\end{align}
where we denote the random variables for estimating $\theta^{t*}_{X_Y}$ by $\Theta^{st}_{X_Y}$.  Following the similar procedure as shown in the proof of conditional shift case, we can write,
\begin{equation}
   \mathbf{I}_{\theta} = \begin{bmatrix}
    nI_{t,X_Y,U} + mI_{s,X_Y} + I_{t,X_Y}   &  nI^T_{t,U}   & \mathbf{0}  \\
    nI_{t,U} &  nI_{t,Y,U} + I_{t,Y} &    \mathbf{0} \\
   \mathbf{0} &  \mathbf{0} &  m I_{s,Y}
    \end{bmatrix}
\end{equation}
and 
\begin{equation}
   \widetilde{\mathbf{I}}_{\theta} = \begin{bmatrix}
    (n+1)I_{t,X_Y,U} + mI_{s,X_Y}   &  (n+1)I^T_{t,U} & \mathbf{0}  \\
    (n+1)I_{t,U}  &  (n+1)I_{t,Y,U}  &    \mathbf{0} \\
   \mathbf{0} &  \mathbf{0} &  m I_{s,Y}
    \end{bmatrix}.
\end{equation}
where $I_{t,U}$ is defined in~(\ref{eq:itu}). We first consider the case where $m = cn^p$ for some $p \geq 1$, as $n$ goes to infinity, we define the matrices $\Delta_U = I_{t,X_Y,U} - I^T_{t,U}I^{-1}_{t,Y,U}I_{t,U}$ and $\Delta_s = I_{s,X_Y}$, then the expected regret can be calculated by using the following argument
\begin{align}
    \operatorname{det}(\mathbf{I}+\frac{1}{n} A) = 1 + \frac{1}{n} \operatorname{Tr}(A)+o(1 / n).
\end{align}
Then,
\begin{align}
    \mathcal{R}(b) =& \frac{1}{2}\log\operatorname{det}(1 + \frac{1}{n+1} (I_{t,Y} - I_{t,Y,U}) I^{-1}_{t,Y,U}  )\\
    &+ \frac{1}{2}\log\operatorname{det}( I_{t,X_Y} - I_{t,X_Y,U} + (n+1)\Delta_U + cn^p\Delta_s) \\
    &- \frac{1}{2}\log\operatorname{det}((n+1)\Delta_U + cn^p\Delta_s)  + o(\frac{1}{m \vee n})\\
    =& \frac{1}{2}\log\operatorname{det}(1 + \frac{1}{n+1} (I_{t,Y} - I_{t,Y,U}) I^{-1}_{t,Y,U}   ) \\
    & + \frac{1}{2}\log(\mathbf{I}_{k'}+\frac{1}{cn^p}(I_{t,X_Y} - I_{t,X_Y,U} + (n+1)\Delta_U)\Delta^{-1}_s) \\
    & - \frac{1}{2}\log(\mathbf{I}_{k'}+\frac{1}{cn^p}((n+1)\Delta_U\Delta^{-1}_s))  + o(\frac{1}{m \vee n}) \\
    \asymp& \frac{(I_{t,Y} - I_{t,Y,U}) I^{-1}_{t,Y,U}}{n+1} + \frac{\operatorname{Tr}((I_{t,X_Y} - I_{t,X_Y,U} + (n+1)\Delta_U)\Delta^{-1}_s))}{cn^p}  - \frac{\operatorname{Tr}( (n+1)\Delta_U\Delta^{-1}_s)}{cn^p} \\
    \asymp & \frac{1}{n} + \frac{k'}{cn^p}
\end{align}
the last asymptotic relationship is due to that $I_{t,Y} \succ I_{t,Y,U}$ and $ I_{t,X_Y} \succ I_{t,X_Y,U}$ as mentioned in the conditional shift case. For the case $0 < p <1$, similarly we arrive at,
\begin{align}
    \mathcal{R}(b) \asymp &  \frac{1}{n+1} +  \frac{\operatorname{Tr}((I_{t,X_Y} - I_{t,X_Y,U} +  cn^p\Delta_s)\Delta^{-1}_U))}{n+1} \\
    \asymp &  \frac{1}{n} +  \frac{k'}{n}. \\
\end{align}
which completes the proof.
\end{proof}

In the following, we consider the semi-supervised learning scenario as $P_{S}(Y) = P_S(Y)$ and $P_S(X|Y) = P_T(X|Y)$. 
\begin{proof}
Since the source and the target have the same distribution, we choose the prior distribution as,
\begin{align}
    \omega(\Theta_s,\Theta_t) = \omega(\Theta^{st}_{Y})\omega(\Theta^{st}_{X_Y}).
\end{align}
Combining the proofs of labelling drift and target shift cases, we arrive at,
\begin{equation}
   \mathbf{I}_{\theta} = \begin{bmatrix}
    nI_{t,X_Y,U} + mI_{s,X_Y} + I_{t,X_Y}   &  nI^T_{t,U}  \\
    nI_{t,U} &  nI_{t,Y,U} + I_{t,Y} + mI_{s,Y}
    \end{bmatrix} = nI_t + mI_s + I_{t,X,Y}
\end{equation}
and 
\begin{equation}
   \widetilde{\mathbf{I}}_{\theta} = \begin{bmatrix}
    (n+1)I_{t,X_Y,U} + mI_{s,X_Y}   &  (n+1)I^T_{t,U}  \\
    (n+1)I_{t,U}  &  (n+1)I_{t,Y, U}  + mI_{s,Y}
    \end{bmatrix} = (n+1)I_t + mI_s.
\end{equation}
We first consider $m = cn^p$ for some $p \geq 1$, as $n$ goes to infinity, we define the matrices $\Delta_U = I^T_{t,X_Y,U} - I_{t,U}I^{-1}_{t,Y,U}I_{t,U}$ and $\Delta_s = I_{s,X_Y}$, then the regret can be calculated by,
\begin{align}
    \mathcal{R}(b) =& \frac{1}{2}\frac{\log\operatorname{det} (\mathbf{I}_{\theta}) }{\log\operatorname{det} (\widetilde{\mathbf{I}}_{\theta}) } + o(\frac{1}{m\vee n})\\
    =& \frac{1}{2} \log\operatorname{det}\left( \mathbf{I}_{k'+1} + (I_{t,X,Y} - I_t)((n+1)I_t + cn^pI_s)^{-1} \right) +  o(\frac{1}{m\vee n}) \\
    \asymp& \frac{\operatorname{Tr}((I_{t,X,Y} - I_t)(\frac{n+1}{cn^p}I_t + I_s)^{-1})}{cn^p}\\
    \asymp & \frac{k'+1}{n^p} 
\end{align}
due to that $I_{t,X,Y} \succ I_t$. Similarly for the case where $0<p<1$, we have,
\begin{align}
    \mathcal{R}(b) &\asymp \frac{\operatorname{Tr}((I_{t,X,Y} - I_t)(\frac{cn^p}{n+1}I_s + I_t)^{-1})}{n+1} \\
    &\asymp \frac{k'+1}{n}.
\end{align}
As a consequence, 
\begin{align}
    \mathcal{R}(b) \asymp \frac{k'+1}{n \vee n^p}.
\end{align}
\end{proof}

\subsection{Proof of Lemma~\ref{lemma:wsc}}
\begin{proof}
We write the minimax expected regret as,
\begin{small}
\begin{align*}
&\min _{b} \max_{\theta^*_s,\theta^*_t} R(b) =\min _{Q}\left\{\max _{\theta_s, \theta_t}\left\{D\left(P_{\theta_s,\theta_t} \| Q(\theta_s,\theta_t)\right)\right\}\right\} \\
&=\min _{b}\left\{\max _{\theta_s, \theta_t} \left\{\int P_{\theta_s,\theta_t}\left(D^{U,m}_t, D^m_s,X'_t,Y'_t \right) \log \left(\frac{P_{\theta_t}\left(Y'_t|X'_t\right)}{Q\left(Y'_t|D^{U,m}_t, D^m_s,X'_t \right)}\right) d D^{U,m}_t d D^m_s dX'_t dY'_t \right\}\right\} \\
& \stackrel{(a)}{=} \min _{b}\left\{\max _{\omega(\theta_s,\theta_t)} \left\{\int P_{\theta_s,\theta_t}\left(D^{U,m}_t, D^m_s,X'_t,Y'_t \right) \log \left(\frac{P_{\theta_t}\left(Y'_t|X'_t\right)}{Q\left(Y'_t|D^{U,m}_t, D^m_s,X'_t \right)}\right) \omega(\theta_s,\theta_t) d\theta_s d\theta_t d D^{U,m}_t d D^m_s dX'_t dY'_t \right\}\right\} \\
& \stackrel{(b)}{=} \max _{\omega(\theta_s,\theta_t)}\left\{\min _{b}\left\{\int D\left(P_{\theta_t}\left(Y'_t|X'_t\right) \| Q\left(Y'_t|D^{U,m}_t, D^m_s,X'_t \right) \right) \omega(\theta_s,\theta_t) d\theta_s d\theta_t d D^{U,m}_t d D^m_s dX'_t dY'_t \right\}\right\}  \\
&=\max _{\omega(\theta_s,\theta_t)} I(Y'_t;\theta_s,\theta_t|D^m_s,D^{U,n}_t,X'_t), 
\end{align*}
where (a) follows as maximizing over $\theta_s$ and $\theta_t$ and is equivalent to maximizing over a distribution over them and (b) follows from the minimax theorem, e.g., see \citet{du2013minimax} for proof. 
\end{small}
\end{proof}

\vskip 0.2in
\bibliography{sample}

\end{document}